\documentclass{article}

\usepackage{microtype}
\usepackage{graphicx}
\usepackage{subcaption}
\usepackage{booktabs} 

\usepackage{hyperref}

\usepackage[accepted]{icml2026}

\usepackage{amsmath}
\usepackage{amssymb}
\usepackage{dsfont}
\usepackage{mathtools}
\usepackage{amsthm}

\usepackage{amsfonts}

\usepackage{xcolor}

\definecolor{algcomment}{RGB}{90,90,90}      
\definecolor{algaccent}{RGB}{30,90,160}      

\usepackage[capitalize,noabbrev]{cleveref}

\theoremstyle{plain}
\newtheorem{theorem}{Theorem}[section]

\newtheorem{lemma}[theorem]{Lemma}

\theoremstyle{definition}

\newtheorem{assumption}[theorem]{Assumption}
\theoremstyle{remark}
\newtheorem{remark}[theorem]{Remark}

\newcommand{\F}{\mathcal{F}}
\newcommand{\E}{\mathbb{E}}
\newcommand{\Pp}{\mathbb{P}}

\newcommand{\cO}{\mathcal{O}}

\newcommand{\ind}[1]{\mathds{1}\{#1\}}

\newcommand{\AVDTbet}{\ensuremath{\mathrm{AVT}_{\mathrm{B}}}}
\newcommand{\AVDTcs}{\ensuremath{\mathrm{AVT}_{\mathrm{CS}}}}

\newcommand{\AVARFbet}{\ensuremath{\mathrm{AVF}_{\mathrm{B}}}}

\usepackage[textsize=tiny]{todonotes}
\usepackage{dblfloatfix}
\usepackage{xcolor}
\usepackage{tikz}
\definecolor{redd}{HTML}{b30326}
\definecolor{bloo}{HTML}{3a4cc0}

\definecolor{ourfirst}{HTML}{3498DB}

\definecolor{oursecond}{HTML}{1ABC9C}

\definecolor{ours}{HTML}{27AE60}

\definecolor{baseline}{HTML}{2C3E50}

\definecolor{alggray}{RGB}{110,110,110}

\newcommand{\algc}[1]{\hfill{\footnotesize\color{alggray}$\triangleright$~#1}}

\icmltitlerunning{Correcting Split Selection in Online Decision Trees via Anytime-Valid Inference}

\begin{document}

\twocolumn[
  \icmltitle{Correcting Split Selection in Online Decision Trees via Anytime-Valid Inference}



  \icmlsetsymbol{equal}{*}

  \begin{icmlauthorlist}
    \icmlauthor{Salim I. Amoukou}{jjj}
    \icmlauthor{Saumitra Mishra}{jjj}
    \icmlauthor{Manuela Veloso}{jjj}

  \end{icmlauthorlist}

    \icmlaffiliation{jjj}{J.P. Morgan AI Research}
  \icmlcorrespondingauthor{Salim I. Amoukou}{salim.ibrahimamoukou@jpmorgan.com}

  \icmlkeywords{Machine Learning, ICML}

  \vskip 0.3in
]



\printAffiliationsAndNotice{}  

\begin{abstract}
Bagging-based ensembles, most notably Adaptive Random Forests,
are among the strongest performers for learning from data streams.
A common denominator across these methods is their reliance on
Hoeffding Trees as base learners, which grow decision trees incrementally by
testing whether a candidate split is significantly better than its
alternatives using concentration inequalities. Despite their empirical success, existing variants lack valid
statistical guarantees. Current analyses rely on fixed-sample
concentration bounds, while split decisions are made using
data-dependent stopping rules, which invalidates their guarantees and can drive the probabilty of incorrect splits to one. We introduce a principled alternative based on
\emph{anytime-valid inference}.
Our method provides:
(i) anytime-valid control of false splits under arbitrary data streams,
including non-stationary settings;
(ii) finite commitment time under a predictive advantage;
and (iii) under stationary i.i.d.\ data, risk is monotone decreasing and strictly improves at every split. Empirically, we evaluate both standalone trees and their use within Adaptive Random Forests on non-stationary streams. Our method improves performance while producing substantially smaller trees.
\end{abstract}


\section{Introduction}

In streaming environments, where data arrive sequentially and models must be updated online,
ensemble methods have emerged as the dominant paradigm.
In particular, bagging-based approaches such as
\emph{Adaptive Random Forests} (ARF)~\cite{gomes2018adaptive}, along with related variants including Streaming Random Patches~\cite{gomes2019streaming},
and K-Nearest Leaves~\cite{sun2022soknl},
consistently rank among the top performers
on recent benchmarks
\cite{montiel2020adaptive,wang2022elastic,aspis2025driftmoe}.
Among these, ARF has effectively become the \emph{de facto} standard
for ensemble learning in data streams,
across both classification \cite{ref_arf} and regression \cite{gomes2018adaptive}.

A striking commonality across these methods is their universal reliance on the Hoeffding Tree (HT), also known as the Very Fast Decision Tree (VFDT) \cite{domingos2000mining}, as the fundamental base learner. Despite its empirical success and wide adoption, the statistical foundation of the Hoeffding Tree is not valid. The HT's core idea is to grow tree incrementally: at each node, it tests if sufficient data has arrived to commit to a split. The original algorithm invalidly applied Hoeffding’s inequality to non-linear impurity measures (e.g., Gini index). Subsequent corrections either substituted linear measures (e.g., misclassification error) \cite{matuszyk2013correcting} or applied more general bounds like McDiarmid’s inequality \cite{rutkowski2012decision, de2017confidence, jaworski2017new}.

We argue that these fixes remain statistically invalid.
All existing approaches overlook a fundamental issue: \emph{standard concentration inequalities assume a fixed sample size, yet the HT algorithm employs a data-dependent stopping rule.} This phenomenon, known as \emph{optional continuation} \cite{grunwald2020safe, shafer_2024}, invalidates the guarantees of fixed-time concentration inequality. In particular, the probability of a false split can inflate to one (see Fig.~1 in~\cite{howard2021time}).

To resolve this, we abandon fixed-sample-size inequalities and instead leverage recent advances in sequential testing, specifically the \emph{safe, anytime-valid inference (SAVI)} framework \cite{shafer2001probability, shafer2011test, hendriks2018test, jun2019parameter, orabona2023tight, waudby2024estimating, ramdas2025hypothesis}. This paradigm is inherently robust to optional stopping. 

\textbf{Our main contributions are:}

\begin{enumerate}
\item We derive a split criterion that remains valid under arbitrary
data-dependent stopping rules.
Our method controls false split decisions
simultaneously over the entire lifetime of the tree,
and applies to non-stationary and dependent data streams. When a candidate split exhibits a persistent predictive advantage, our procedure commits to the split in finite time.

\item Under stationarity, the learned tree's expected loss is non-increasing over time, and improves monotonically as data accumulate.

\item  We validate our method across several non-stationary environments,
both as a standalone replacement for Hoeffding Trees
and as a drop-in substitute within Adaptive Random Forests.
The resulting ensemble of trees achieve higher predictive performance
while being consistently smaller
in depth or number of nodes.
\end{enumerate}

\vspace{-0.5cm}
\section{Problem Definition}

Let $(\F_t)_{t\ge0}$ be the filtration representing all information available up to time $t$, including the data and any internal randomization:
$\F_t := \sigma((X_1,Y_1),\dots,(X_t,Y_t),U)$ where $U$ denotes algorithmic
randomness (if any). A quantity chosen at time $t$ is \emph{predictable} if it is
$\F_{t-1}$-measurable.

For clarity, we present the problem for classification decision trees with binary splits on continuous features; extensions to multiway splits, categorical features, and regression follow by standard adaptations \cite{rutkowski2020stream}.  Each leaf node update mechanism has two components:
(i) a \emph{split measure} quantifying split quality, and (ii) a \emph{split condition} determining whether the node should be partitioned.

A decision tree recursively partitions the input space
$\mathcal X\subseteq\mathbb R^d$
into a collection of axis-aligned regions.
Each node $v$ corresponds to a measurable region
$R(v)\subseteq\mathcal X$. A split candidate is a pair $c=(j,s)$ with feature index $j\in[d]$ and threshold $s\in\mathbb{R}$, inducing child regions
\[
R(v_L^c)=R(v)\cap\{x_j\le s\},
\qquad
R(v_R^c)=R(v)\cap\{x_j> s\}.
\]
At each node $v$, the set of admissible splits is restricted to a finite
candidate class $C_v \subseteq \{1,\dots,d\}\times\mathbb R$. The finiteness of $C_v$ reflects practical implementations
(e.g., discretised thresholds, quantiles)
and ensures that optimality is always defined
with respect to a fixed, finite hypothesis class.

\paragraph{Batch decision tree learning.}

Given a fixed dataset
$\mathcal D_t=\{(X_i,Y_i)\}_{i=1}^t$,
an effective split maximally reduces node impurity.
Define the samples reaching node $v$ as $A_t(v)
:=
\{(X_i,Y_i)\in\mathcal D_t:X_i\in R(v)\}$, $n_t(v)
:= |A_t(v)|$. Let $ p_t(v)$ denote the empirical class proportions,
with components $p_{t,k}(v) =
\frac{1}{n_t(v)}
\sum_{(X_i,Y_i)\in A_t(v)}
\ind{Y_i=k}.$
Common impurity measures $\mathcal I(\cdot)$ include entropy and the Gini index: $\mathcal{ I}_{\mathrm{entropy}}( p) = -\sum_k  p_k\log \mathbf p_k,$ and $ \mathcal{I}_{\mathrm{Gini}}( p) = \sum_k  p_k \bigl(1- p_k\bigr).$
For a candidate split $c=(j,s)\in C_v$,
the empirical impurity decrease at node $v$ is
\[
\Delta_t^{v,c}
=
\mathcal I(\mathbf p_t(v))
-
P_{t,L}\,
\mathcal I(\mathbf p_t(v_{c}^L))
-
P_{t,R}\,
\mathcal I(\mathbf p_t(v_{c}^R)),
\]
where $P_{t,L}=n_t(v_{c}^L)/n_t(v),$ and $ P_{t,R}=n_t(v_{c}^R)|/n_t(v).$

In batch learning, the split at node $v$ is chosen by maximising the empirical impurity reduction over the candidate set $c^\star
=
\arg\max_{c\in C_v}
\Delta_t^{v,c}.$

\paragraph{Online decision tree learning.}

In a streaming setting,
observations $(X_t,Y_t)_{t\ge1}$ arrive sequentially
and the dataset is never fully observed.
The batch procedure is therefore inapplicable;
instead, the learner must decide \emph{online}
when sufficient evidence justifies a split.

For each node $v$ and split $c\in C_v$,
define the \emph{population impurity decrease} as
\begin{equation}
\Delta^{v,c}
=
\mathcal I( p(v))
-
P_L\,
\mathcal I( p(v_{c}^L))
-
P_R\,
\mathcal I( p(v_{c}^R)),
\label{eq:pop_impurity}
\end{equation}
where probabilities correspond to their population counterparts
under the data-generating distribution.
Ideally, the learner would select a split
$c^\star$ satisfying
\begin{equation}
\Delta^{v,c^\star}
>
\Delta^{v,c^\prime}
\qquad
\forall c^\prime \in C_v,
\label{eq:pop_split}
\end{equation}
that is,
the empirically chosen split coincides
with the population-optimal one
over the same finite candidate class.

Since $\Delta^{v,c}$ is unknown,
online algorithms rely on statistical inference
from empirical estimates
$\Delta_t^{v,c}$
computed from the data observed so far.
A common strategy introduces a threshold
$\varepsilon(n_t(v),\delta)$
such that, if
\[
\Delta_t^{v,c^\star}
-
\Delta_t^{v,c^\prime}
>
\varepsilon(n_t(v),\delta)
\quad
\forall c^\prime\in C_v,
\]
then, with probability at least $1-\delta$,
the selected split is population-optimal.

Hoeffding Trees~\cite{domingos2000mining} instantiate
$\varepsilon(\cdot)$
using Hoeffding’s inequality,
despite the non-linearity of impurity criteria.
Subsequent work replaces this with
McDiarmid-type bounds
\cite{rutkowski2012decision,de2017confidence,jaworski2017new}
or linearized impurity measures
\cite{matuszyk2013correcting}.

\paragraph{The overlooked issue.}

All existing guarantees are valid given
a \emph{fixed sample size}.
In practice, however,
impurity estimates are updated sequentially,
and the algorithm commits to a split
at the \emph{first} time
the stopping condition is satisfied.
This induces a data-dependent stopping time
that depends on the entire data trajectory.

Let $(C_t)_{t\ge1}$ be a sequence of confidence intervals for a parameter $\delta$, constructed using a classical concentration inequality (e.g., Hoeffding's inequality).
Such bounds provide only \emph{fixed-time} control: for any fixed $t$,
\[
\mathbb P(\delta \notin C_t)\le \alpha,
\]
which offers no guarantee under data-dependent stopping.
Instead, we require \emph{anytime-valid} coverage:
\[
\mathbb P\!\left(\exists t\ge1:\delta \notin C_t\right)\le\alpha,
\quad\text{or}\quad
\mathbb P\!\left(\forall t\ge1:\delta \in C_t\right)\ge 1-\alpha.
\]
This immediately implies validity at any (possibly infinite) stopping time $\tau$:
\[
\mathbb P(\delta \notin C_\tau)\le\alpha,
\qquad
\mathbb P(\delta \in C_\tau)\ge 1-\alpha.
\]

Standard bounds fail under optional stopping
(see Fig.~1 in~\cite{howard2021time})
and therefore do \emph{not} control
the probability of selecting an incorrect split
at the time the algorithm actually stops.

Moreover, a fundamental gap exists
between theory and practice.
Classical Hoeffding-based analyses assume
independent observations,
yet Hoeffding Trees are most often deployed
inside ensemble methods such as
Adaptive Random Forests,
which operate on non-stationary
and potentially dependent streams.
This directly violates the assumptions
underlying existing guarantees. These observations motivate
our central question:
\begin{center}
\emph{
How can we design a splitting criterion
that remains valid under
data-dependent stopping
and non-stationary, dependent data streams?
}
\end{center}

\section{Valid Sequential Tests for Split Selection}

Our approach departs from the classical objective of selecting the split that
maximizes population impurity reduction
(cf.~\eqref{eq:pop_split}).
In non-stationary environments, the population impurity
\eqref{eq:pop_impurity} may itself be ill-defined,
since the data-generating distribution may evolve over time
and the optimal split may change accordingly.
Rather than comparing candidate splits \emph{against each other}
under a static distribution,
we instead ask a different question:

\begin{center}
\emph{Is a proposed split demonstrably better than leaving the node unsplit,
based on the data observed so far?}
\end{center}

This reframes split selection as an
\emph{online model comparison} problem.
Specifically, we compare two predictors:
an \emph{incumbent} model corresponding to the unsplit leaf,
and a \emph{challenger} model obtained by applying a candidate split.
The goal is to determine, using accumulated information up to time $t$,
whether the challenger has achieved better predictive performance
than the incumbent under potentially non-stationary
and dependent data.

Predictive performance is evaluated using a loss function $\ell$, with log loss and
Brier score serving as loss-based analogues of entropy and Gini impurity.
Our theoretical analysis assumes bounded losses $\ell\in[0,1]$; when this condition
does not hold in practice, we enforce it via normalization.

\subsection{Testing a Candidate Split}

Fix a leaf node $v$ created at time $ s^v$.
Until a split is committed,
node $v$ acts as the \emph{incumbent} predictor.

At time $ s^v$,
the learner selects a candidate split $c\in C_v$.
This defines a hypothetical \emph{challenger} model
that partitions $R(v)$ into two child regions.
No structural changes are made at this stage:
the challenger is evaluated in \emph{shadow mode}
and does not affect future routing decisions.

\begin{itemize}
\item Incumbent $m^{v}$
predicts using the empirical label distribution
at node $v$,
estimated from data observed up to the current time.

\item Challenger $m^{v_c}$
identical to the incumbent,
except that node $v$ is split according to $c$,
with each child maintaining its own
empirical label distribution.
\end{itemize}

For each $t\ge s^v+1$, 
\begin{enumerate}
\item Both predictors output \emph{prequential}
predictions based on $\mathcal F_{t-1}$.
\begin{align}
m^{v}_{t-1}(x)
&=
p_{t-1}(v),
\qquad
x\in R(v), \nonumber\\
m^{v_{c}}_{t-1}(x)
&=
\begin{cases}
p_{t-1}(v_{c}^L),
& x\in R(v_{c}^L),\\[2mm]
 p_{t-1}(v_{c}^R),
& x\in R(v_{c}^R),
\end{cases}
\end{align}

\item After observing $Y_t$,
we compute the losses and define the bounded difference
\[
\Delta_t^{v,c}
:=
\ell\!\left(m^v_{t-1}(X_t),Y_t\right)
-
\ell\!\left(m^{v_{c}}_{t-1}(X_t),Y_t\right)
\in[-1,1].
\]

\item Both predictors are updated identically
using $(X_t,Y_t)$;
their only distinction remains the proposed split.
\end{enumerate}

Since both predictions are $\mathcal F_{t-1}$-measurable,
the conditional mean is well defined:
\[
\delta_t^{v,c}
:=
\mathbb E\!\left[
\Delta_t^{v,c}
\mid
\mathcal F_{t-1}
\right].
\]

\paragraph{Defining improvement.}
In non-stationary environments, there are two standard hypothesis-testing formulations for comparing predictive models over time
\cite{lehmann1975statistical,rosenbaum1995,ehm2018forecast,choe2024comparing}.

The first formulation tests whether a challenger ever attains an advantage over the incumbent. This corresponds to the \emph{strong null hypothesis}~\cite{choe2024comparing}:
\[
H_{0}^{v,c}:
\qquad
\forall t\ge s^v+1,
\quad
\delta_t^{v,c} \le 0.
\]
Under this hypothesis, the challenger is assumed to never outperform the incumbent at any point in time.

The second formulation evaluates whether the challenger improves upon the incumbent on average over time. This leads to the \emph{weak null hypothesis}:
\[
H^{v,c}_{\mathrm{w},0}:
\qquad
\forall t\ge s^v+1,
\quad
\frac{1}{t- s^v}
\sum_{u= s^v+1}^{t} \delta_u^{v,c}
\le 0.
\]

The choice between these formulations depends on the specific application. Recent work ~\cite{choe2024comparing} suggests using the strong null when comparing closely related algorithms operating on the same data stream, and the weak null when comparing substantially different models.

\paragraph{Choice in our setting.}

We adopt the \emph{strong} null ($H_0^{v,c}$) as our primary testing target.
This aligns with the recommendation of \cite{choe2024comparing} to use the strong null when comparing closely related algorithms trained on the same data.  Indeed, during the test period, both models are trained using the same data and are directly related by construction. Empirically, we also found that the \emph{weak} null leads to more conservative behaviour: splits are committed later, resulting in inferior predictive performance. This is expected, since the weak formulation relies on averaged effects that can be influenced by accumulated history.

For completeness, we derive sequential tests for both hypotheses. For notational brevity, we omit the superscripts $(v,c)$ where the context is unambiguous.

\subsection{Anytime-Valid Test via Testing by Betting}

We adopt the \emph{testing-by-betting} framework
\cite{shafer2001probability,shafer2011test,hendriks2018test,jun2019parameter,orabona2023tight,waudby2024estimating, ramdas2025hypothesis},
which interprets statistical hypothesis testing as a
\emph{betting game against the null hypothesis}.
At each time step, the bettor wagers on whether the challenger
outperforms the incumbent, as measured by the sign of
$\Delta_t$.
Under the null hypothesis $H_0$,
no betting strategy can achieve positive expected growth.
Under the alternative, the bettor’s wealth increases over time.

\paragraph{Testing by betting.}
The bettor starts with unit wealth $W_{s}=1$.
At each time $t\ge s + 1$,
a predictable betting fraction
$\beta_t\in[0,1]$ is chosen based on $\mathcal F_{t-1}$.
After observing $\Delta_t$,
wealth is updated according to
\begin{equation}
W_t
=
W_{t-1}\bigl(1+\beta_t \Delta_t\bigr).
\label{eq:wealth_update_main}
\end{equation}
Since under the \emph{strong} null $H_0$ we have
$\E[\Delta_t\mid\mathcal F_{t-1}]\le 0$,
the process $(W_t)_{t\ge s}$
is a nonnegative supermartingale.

By Ville’s inequality, for any $\alpha\in(0,1)$,
\[
\Pp_{H_0}\!\left(
\sup_{t\ge s } W_t
\ge
\frac{1}{\alpha}
\right)
\le
\alpha.
\]
Consequently, thresholding the wealth process
yields an anytime-valid test:
regardless of the betting strategy,
under $H_0$ the wealth exceeds $1/\alpha$
with probability at most $\alpha$.
Therefore, a wealth above this threshold is a strong evidence that the null is false.

\paragraph{Portfolio interpretation.}
Following~\cite{orabona2023tight},
we can equivalently represent the betting game as an
online portfolio selection with two stocks.
Define two assets with returns
\[
R_t^{(0)} := 1,
\qquad
R_t^{(1)} := 1+\Delta_t \in [0,2].
\]
Investing a fraction $\beta_t$ in asset~1
and $1-\beta_t$ in asset~0
yields gross return
\[
(1-\beta_t)R_t^{(0)} + \beta_t R_t^{(1)}
=
1+\beta_t\Delta_t,
\]
which coincides with~\eqref{eq:wealth_update_main} one step's update.
Designing a powerful anytime-valid test
therefore reduces to selecting an effective
online portfolio strategy.

Rather than tuning $\beta_t$, we adopt a parameter-free strategy:
the \emph{Universal Portfolio} (UP).
UP is minimax-optimal
with respect to the best constant rebalanced portfolio
and, in i.i.d.\ markets,
achieves the optimal growth rate
of any fixed strategy~\cite{cover1998elements}.
UP constructs a mixture over
constant rebalanced portfolios
using a Jeffreys prior.

Concretely, we define
\begin{equation}
\beta_t
:=
\frac{
\int_0^1
\beta
\prod_{u=s}^{t-1}
\bigl(1+\beta\,\Delta_u\bigr)
\,dF_+(\beta)
}{
\int_0^1
\prod_{u=s}^{t-1}
\bigl(1+\beta\,\Delta_u\bigr)
\,dF_+(\beta)
},
\;
F_+
=
\mathrm{Beta}\!\left(\tfrac12,\tfrac12\right),
\label{eq:UP_beta}
\end{equation}
and update wealth via~\eqref{eq:wealth_update_main}.
The resulting process remains
a nonnegative supermartingale under $H_0$
and empirically accumulates evidence rapidly
when the challenger is superior.

\paragraph{Replacement rule and global error control.}
We run one such test
for each candidate split
available at a leaf.
Given per-test significance levels
$\alpha^{v, c}$, we replace the incumbent
with the challenger
at the stopping time
\[
\tau^{v,c}
:=
\inf\!\left\{
t\ge s^v + 1:
W^{v,c}_t
\ge
\frac{1}{\alpha^{v, c}}
\right\}.
\]
We show in the next section that it yields \emph{global} control:
with probability at least $1-\alpha$,
no false replacement ever occurs
over the entire stream. When multiple candidates cross their thresholds
simultaneously, we select the one with the largest wealth.

\subsection{Anytime-Valid Test via Confidence Sequences}
\label{sec:cs}

The \emph{weak} null hypothesis can be tested in a closely related manner using
confidence sequences (CS).
Let $(L_t,U_t)_{t\ge s^v +1}$ be a confidence sequence for the running average $
\bar\delta^{v,c}_t  :=
\frac{1}{t-s^v}\sum_{u=s^v+1}^t \delta_u$, that is,
\[
\mathbb P\!\left(
\forall t\ge s^v + 1:\;
L_t \le \bar\delta^{v,c}_t \le U_t
\right)
\ge
1-\alpha.
\]
Then the stopping time
\[
\tau_{\mathrm{w}}^{v,c}
:=
\inf\!\left\{
t\ge s^v + 1:\;
L_t > 0
\right\}
\]
defines an anytime-valid test of the weak null hypothesis,
with guarantees analogous to those of the betting-based test for the strong null.

In practice, we use the empirical Bernstein CS
\cite{howard2021time},
which has recently been successfully applied to online model comparison \cite{schirmer2025monitoring} and sequential forecasters evaluation \cite{choe2024comparing}.

\begin{remark}
Although betting-based tests and confidence-sequence-based tests
appear different here, they are closely connected.
Confidence sequences are typically constructed from nonnegative
(super)martingales via Ville’s inequality, and conversely,
betting processes can be inverted to yield confidence sequences.
See \citet{orabona2023tight} for constructions of confidence sequences
from betting strategies similar to ours, and
\citep{howard2021time,choe2024comparing,ramdas2025hypothesis}
for a general treatment of confidence sequences derived from
supermartingales.
\end{remark}

\begin{remark}
To test for a strictly positive advantage $\varepsilon>0$,
both approaches admit straightforward modifications.
For the betting-based (strong) test, we use the
$\varepsilon$-shifted wealth process
$
W^{v,c}_{\varepsilon,t}
:=
\prod_{u=s^v+1}^{t}
\bigl(1+\beta_u(\Delta^{v,c}_u-\varepsilon)\bigr),
$
with betting fractions constrained to
$\beta_u\in[0,\,1/(1+\varepsilon)]$.
For the CS-based (weak) test, the stopping rule is modified to $L_t>\varepsilon$.
\end{remark}

Algorithmic details are given in Algorithm~\ref*{alg:avht}.
Details on UP approximation, the empirical Bernstein CS, the ${\alpha_{v,c}}$ schedule, and challenger generation are given in Appendix~\ref{app:parameters}.

\begin{algorithm}[!ht]
\caption{Anytime-Valid Online Decision Tree (AVT)}
\label{alg:avht}
\begin{algorithmic}[1]
\REQUIRE Global level $\alpha$, minimum samples $n_{\min}$, advantage $\varepsilon\ge 0$,
test type $\textsc{Test}\in\{\textsc{Betting},\textsc{CS}\}$

\STATE Initialize root $v_0$; active leaves $\mathcal{L}\gets\{v_0\}$
\STATE Set levels $\{\alpha^{v,c}\}$ with $\sum_{v,c}\alpha^{v,c}\le\alpha$
\STATE $(S^{v,c},\tau^{v,c})\gets
\begin{cases}
(W^{v,c}_{\varepsilon}\!\leftarrow\!1,\;1/\alpha^{v,c}), & \textsc{Betting},\\
(L^{v,c}\!\leftarrow\!-\infty,\;\varepsilon), & \textsc{CS}
\end{cases}$
\algc{Initialize e-process / CS and stopping threshold}

\FOR{$t = 1,2,\ldots$}
    \STATE Receive $(X_t,Y_t)$; route $X_t$ to leaf $v\in\mathcal{L}$
    \STATE $\hat{y}^v\gets m^v_{t-1}(X_t)$ \algc{Incumbent prediction}

    \FOR{each candidate split $c\in\mathcal{C}_v$}
        \STATE $\hat{y}^{v_c}\gets m^{v_c}_{t-1}(X_t)$ \algc{Challenger prediction}
        \STATE $\Delta_t^{v,c}\gets \ell(\hat{y}^v,Y_t)-\ell(\hat{y}^{v_c},Y_t)$
        \algc{Loss improvement}
        \STATE $S^{v,c}\gets$\textsc{TestUpdate}$(S^{v,c},\Delta_t^{v,c})$
        \algc{Update wealth (betting) via UP or CS bound}
        \STATE Update challenger statistics with $(X_t,Y_t)$
    \ENDFOR

    \STATE Update incumbent statistics with $(X_t,Y_t)$
    \IF{$n_t(v)\ge n_{\min}$ \AND $\exists c:\, S^{v,c}>\tau^{v,c}$}
        \STATE $c^\star\gets\arg\max_c S^{v,c}$ \algc{Most significant challenger}
        \STATE Split $v$ using $c^\star$; create children $v^{c^\star}_L, v^c_R$ \STATE $\mathcal{L} \gets (\mathcal{L} \setminus \{v\}) \cup \{v^{c^\star}_L, v^{c^\star}_R\}$ \algc{Commit split; replace leaf by its two children}
    \ENDIF
\ENDFOR
\end{algorithmic}
\end{algorithm}

\section{Theoretical Guarantees}
\label{sec:theory}

We establish guarantees for the proposed sequential split-selection
mechanism under two notions of improvement:
a \emph{strong} (anytime advantage) hypothesis $H_0^{v,c}$ and a \emph{weak} (average advantage) hypothesis $H_{\mathrm{w},0}^{v,c}$.
These correspond respectively to testing-by-betting and
confidence-sequence–based tests. Proofs and technical details are
deferred to the appendix.
\subsection{Anytime Validity and Global Error Control}

Both tests are anytime-valid: they control Type~I error under arbitrary,
data-dependent stopping times and without independence assumptions.

\begin{theorem}[Anytime validity and global control]
\label{thm:anytime_global}
Let $\{\alpha^{v,c}\}_{v,c}$ satisfy
$\sum_{v,c}\alpha^{v,c}\le \alpha$.
Apply to each candidate split $(v,c)$ either
(i) the betting-based test for $H_0^{v,c}$ or
(ii) the confidence-sequence test for $H_{\mathrm{w},0}^{v,c}$,
with level $\alpha^{v,c}$.

Then, with probability at least $1-\alpha$, no false split is ever committed
relative to the tested global null hypothesis.
This guarantee holds uniformly over time, and adaptive tree growth.
\end{theorem}


\subsection{Power: Finite Commitment Under Advantage}

When a candidate split enjoys a persistent predictive advantage,
both tests commit in finite time.
\begin{theorem}[Finite commitment under advantage]
\label{thm:power}
Fix a candidate split $(v,c)$.
Assume there exist $\Delta>0$ and $n_0\ge 0$ such that either the instantaneous
advantage satisfies $\delta_t^{v,c}\ge \Delta$ for all $t\ge s^v+n_0$, or the
average advantage satisfies $\bar\delta_t^{v,c}\ge \Delta$ for all
$t\ge s^v+n_0$.
Then, under the corresponding alternative, the stopping times
$\tau^{v,c}$ and $\tau_{\mathrm{w}}^{v,c}$ are finite almost surely, with rate $\tilde{\mathcal O}(\log(1/\alpha^{v,c})/\Delta^2)$ with high probability.
\end{theorem}

\subsection{Guarantees Under Stationary Data}

Under stationary data, the proposed procedure enjoys a stronger guarantee: the expected predictive performance of the deployed
tree improves monotonically over time.

\begin{theorem}[Strict monotonicity between and at commit times]
\label{thm:monotone_model}
Assume $(X,Y)\sim P$ are i.i.d.\ and $\ell(\hat y,y)$ is convex in its first argument.
Run Algorithm~\ref{alg:avht} and let $0=\tau_0<\tau_1<\tau_2<\cdots$ denote the (random)
commit times, and write $\hat m_t$ for the deployed predictor.

\emph{(Between commits.)}
For all $k\ge 0$ and $t\in(\tau_k,\tau_{k+1})$,
\[
\mathbb E\!\left[\ell(\hat m_t(X),Y)\right]
\;\leq\;
\mathbb E\!\left[\ell(\hat m_{t-1}(X),Y)\right].
\]

\emph{(At commits.)}
There exists $\varepsilon$ such that if splits are committed
only when the weak test certifies an advantage exceeding $\varepsilon$, then with
probability at least $1-\alpha$, for every commit time $\tau_k$,
\[
\mathbb E\!\left[\ell(\hat m_{\tau_k}(X),Y)\right]
\;\le\;
\mathbb E\!\left[\ell(\hat m_{\tau_k^-}(X),Y)\right],
\]
where $\hat m_{\tau_k^-}$ denotes the pre-commit model.
\end{theorem}

\begin{remark}
The monotone update property is independent of the testing procedure and follows solely from convexity and plug-in leaf updates; under strong convexity, these
updates yield strict improvement. Similar monotonicity results for ensembles are established by \citet{mattei2025ensembles}. The role of anytime-valid
testing in our method is to ensure that \emph{structural} updates (committed
splits) also preserve monotonicity under data-dependent stopping.
\end{remark}

\begin{remark}
In contrast to the weak hypothesis, rejecting the strong hypothesis only certifies
that a challenger outperforms the incumbent at some time, not necessarily at the
(commit) stopping time. Ensuring monotonicity under the strong hypothesis therefore
requires an additional \emph{stability} condition, such as persistence of the
advantage after it first appears. However, these conditions are required for the theoretical analysis; in practice, we observe monotonicity for both tests.
\end{remark}

\subsection{Warm-up experiment}
We evaluate the practical implications of our stationarity guarantees on a
controlled synthetic benchmark.
Data are generated i.i.d.\ from a fixed, tree-structured distribution using the
\texttt{RandomTree} generator introduced in the original Hoeffding Tree
work~\cite{domingos2000mining}.
This generator constructs a random decision tree by recursively splitting
features at random and assigning class labels at the leaves.
We use $10$ numerical and $10$ categorical features and generate trees of depth
$4$, corresponding to at most $8$ leaves.

We compare the HT with our Anytime-Valid Tree (AVT) based
on betting tests ($\mathrm{AVT}_{\mathrm{B}}$) and confidence sequences ($\mathrm{AVT}_{\mathrm{CS}}$). Unless stated otherwise, we use the same default hyperparameters across all
experiments, namely $n_{\min}=20$ and $\varepsilon=0$.
At each time step, generalisation performance is estimated using a fixed,
independent holdout set of size $200{,}000$ drawn from the same distribution.

\begin{figure}[!ht]
  \centering
  \includegraphics[width=0.7\columnwidth]{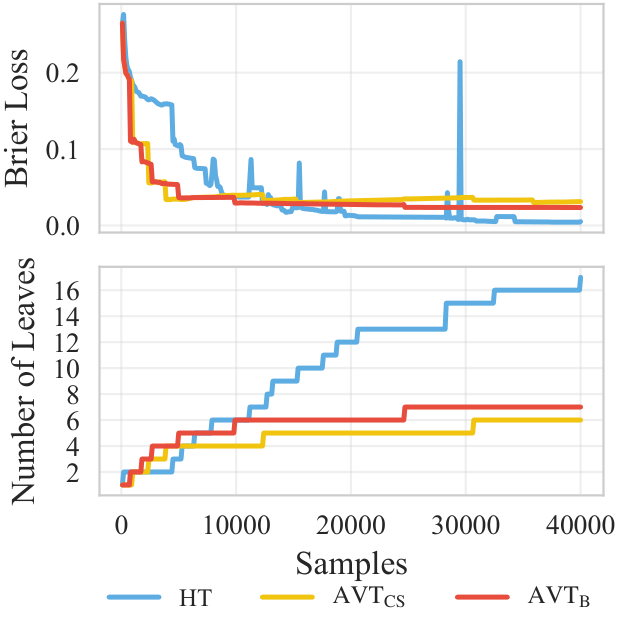}
  \caption{Generalisation error over time.
  The HT exhibits repeated performance drops, whereas
  anytime-valid methods remain stable.}
  \label{fig:convergence}
\end{figure}

\paragraph{Results.}
Figure~\ref{fig:convergence} highlights a clear contrast between methods.
The HT exhibits multiple abrupt drops in generalisation performance.
Notably, this behaviour occurs despite the data being fully i.i.d.\ and
stationary.

In contrast, both anytime-valid variants avoid such collapses and display
stable, monotone improvement in generalisation, consistent with
Theorem~\ref{thm:monotone_model}.
The HT also commits substantially more splits than our methods,
leading to larger and less stable trees. Our approach caps the number of leaves at $8$, matching the size of
the ground-truth tree.

Regarding the two anytime-valid methods, we observe that the variant
based on the weak (average-based) hypothesis performs slightly worse in practice.
This variant commits to splits later than the strong-hypothesis
test, which translates into delayed adaptation and inferior predictive
performance.

Finally, the HT’s slight performance advantage at very large sample sizes is expected: By committing to many splits, it reduces approximation bias irrespective of split quality, and as estimation variance vanishes with more data, the model can perform well asymptotically.

\begin{table*}[!ht]
  \centering
  \small
  \setlength{\tabcolsep}{2pt}
  \renewcommand{\arraystretch}{0.95}
  \resizebox{\textwidth}{!}{%
  \begin{tabular}{lcccccccccccc}
  \toprule
  & \multicolumn{12}{c}{\textbf{Datasets}} \\
  \cmidrule(lr){2-13}
  Method 
  & rbfm100k & airlines & creditcard & hyper100k & elec2 & http 
  & bike & chick & fried 
  & nzenergy & house & abalone \\
  \midrule
  
  \multicolumn{13}{c}{\textbf{Inference time (ms, mean $\pm$ std)}}\\
  \midrule
  HT    
  & 0.124$\pm$0.011 & 0.028$\pm$0.008 & 0.010$\pm$0.001 & 0.066$\pm$0.004 & 0.031$\pm$0.004 & 0.006$\pm$0.001
  & 0.024$\pm$0.004 & 0.018$\pm$0.004 & 0.026$\pm$0.004
  & 0.041$\pm$0.004 & 0.030$\pm$0.003 & 0.021$\pm$0.003 \\
  
  \AVDTbet{}
  & 0.124$\pm$0.013 & 0.010$\pm$0.001 & 0.008$\pm$0.001 & 0.060$\pm$0.006 & 0.066$\pm$0.020 & 0.006$\pm$0.001
  & 0.022$\pm$0.003 & 0.016$\pm$0.002 & 0.023$\pm$0.003
  & 0.037$\pm$0.004 & 0.026$\pm$0.003 & 0.018$\pm$0.002 \\
  
  ARF
  & 0.051$\pm$0.006 & 0.035$\pm$0.012 & 0.010$\pm$0.002 & 0.035$\pm$0.009 & 0.030$\pm$0.007 & 0.006$\pm$0.001
  & 0.016$\pm$0.002 & 0.014$\pm$0.002 & 0.018$\pm$0.003
  & 0.023$\pm$0.003 & 0.019$\pm$0.002 & 0.013$\pm$0.002 \\
  
  \AVARFbet{}
  & 0.130$\pm$0.012 & 0.016$\pm$0.001 & 0.012$\pm$0.003 & 0.043$\pm$0.006 & 0.055$\pm$0.015 & 0.006$\pm$0.001
  & 0.020$\pm$0.003 & 0.015$\pm$0.002 & 0.020$\pm$0.003
  & 0.031$\pm$0.004 & 0.024$\pm$0.003 & 0.016$\pm$0.002 \\
  
  \midrule
  \multicolumn{13}{c}{\textbf{Update time (ms, mean $\pm$ std)}}\\
  \midrule
  HT    
  & 0.074$\pm$0.009 & 0.078$\pm$0.026 & 0.110$\pm$0.023 & 0.147$\pm$4.795 & 0.094$\pm$0.334 & 0.017$\pm$0.001
  & 0.078$\pm$0.023 & 0.064$\pm$0.006 & 0.046$\pm$0.018
  & 0.608$\pm$0.670 & 0.139$\pm$0.075 & 0.056$\pm$0.010 \\
  
  \AVDTbet{}
  & 5.22$\pm$0.71 & 9.61$\pm$3.16 & 0.132$\pm$0.038 & 2.48$\pm$4.09 & 18.66$\pm$55.84 & 0.020$\pm$0.003
  & 3.16$\pm$0.29 & 1.20$\pm$0.23 & 4.78$\pm$0.26
  & 25.82$\pm$28.79 & 5.52$\pm$0.48 & 1.89$\pm$0.22 \\
  
  ARF
  & 0.228$\pm$2.449 & 2.18$\pm$14.21 & 0.079$\pm$0.012 & 0.469$\pm$1.107 & 0.379$\pm$0.870 & 0.028$\pm$0.006
  & 0.013$\pm$0.005 & 0.010$\pm$0.002 & 0.028$\pm$0.007
  & 0.074$\pm$0.079 & 0.021$\pm$0.004 & 0.011$\pm$0.001 \\
  
  \AVARFbet{}
  & 7.65$\pm$0.40 & 10.65$\pm$4.40 & 0.208$\pm$0.052 & 3.82$\pm$4.42 & 27.94$\pm$82.30 & 0.034$\pm$0.010
  & 0.401$\pm$0.034 & 0.190$\pm$0.033 & 0.486$\pm$0.028
  & 3.06$\pm$3.63 & 0.533$\pm$0.035 & 0.277$\pm$0.023 \\
  
  \bottomrule
  \end{tabular}
  }
  \caption{Per-example \textbf{inference and update time} (ms, mean $\pm$ std) across datasets.}
  \label{tab:runtime_all}
\end{table*}

\section{Experiments}

We evaluate our approach against
(i) the Hoeffding Tree (HT) and
(ii) its ensemble counterpart, Adaptive Random Forests (ARF).
In both cases, we replace the HT base learner with our proposed
Anytime-Valid Tree (AVT), yielding an Anytime-Valid Forest (AVF). In the main paper, we report results only for the betting-based variant
(\AVDTbet{}), as it performs better overall than the CS
version (\AVDTcs{}), whose results are deferred to the Appendix \ref{app:cs_comp}. Nevertheless, any
positive advantage observed for \AVDTbet{} over the baselines also carries over to \AVDTcs{}. While our primary focus is improving HT and ARF, additional streaming baselines are reported in Appendix~\ref{app:baselines_2}. All methods are implemented using the
\texttt{river} framework~\cite{montiel2021river}.

We evaluate all methods on 12 data streams, comprising
six regression and six classification tasks.

\textbf{Regression tasks:}
\texttt{bike}, predicting daily bike-sharing usage;
\texttt{chick}, modeling poultry growth trajectories;
\texttt{fried}, forecasting food consumption patterns;
\texttt{nzenergy}, forecasting regional energy consumption;
\texttt{house}, predicting house prices; and
\texttt{abalone}, estimating abalone age from physical measurements.

\textbf{Classification tasks:}
\texttt{elec2}, predicting electricity price movements (up/down);
\texttt{airlines}, predicting flight delays;
\texttt{http-KDD99}, detecting HTTP network intrusions;
\texttt{creditcard}, detecting fraudulent transactions;
\texttt{rbfm100k}, a synthetic stream with gradual drift; and
\texttt{hyper100k}, a synthetic stream with linear drift.

For regression tasks, we enforce bounded losses
via an adaptive scaling procedure.
At each time step, we maintain an online estimate
of the maximal observed loss and rescale the current loss, clipping it to $[0,1]$. For both regression and classification, we use the squared loss
$\ell(p,y)=\lVert p-y\rVert^2$.

\paragraph{Evaluation protocol.}
Evaluation is conducted using the standard prequential (test-then-train) protocol. We report classification accuracy and regression MAE, visualized in two ways: the left panels show the prequential performance curves averaged over $10$ independent runs (with shaded $95\%$ confidence intervals), together with the evolution of tree depth over time; for ensemble methods, we
report the \emph{average} tree depth across trees to enable direct comparison
with single-tree models; and the right panels display the distribution of
aggregate performance across runs.

\paragraph{Model Performance and Complexity.} Figure~\ref{fig:main_all_datasets_merged} summarizes these results.
Across the $12$ data streams, our anytime-valid tree (\AVDTbet{}) outperforms the
standard HT on the majority of datasets, with
\texttt{abalone} being the main exception where HT attains better
aggregate performance.
On several streams, including \texttt{bike},
\texttt{fried}, and \texttt{hyper100k} \AVDTbet{} is competitive with and occasionally surpasses the
ensemble baseline ARF. Importantly, gains persist throughout the stream, as reflected by the prequential curves.

Among all approaches, \AVARFbet{} achieves the strongest overall performance across nearly all datasets, consistently outperforming the standard ARF baseline. Overall, performance differences become more pronounced on datasets exhibiting consistent drops or increased volatility over time, where betting-based methods adapt more effectively.

In terms of model complexity, \AVARFbet{} in addition of providing the best model also produces the smallest trees on average, even compared to single-tree methods.
By contrast, \AVDTbet{} tends to grow deeper trees compared to its baseline HT.
This is expected: in the absence of ensembling, a single tree must capture greater heterogeneity in the data stream, requiring additional depth to match the accuracy achievable by an ensemble.

\paragraph{Computational Cost.}

Table~\ref{tab:runtime_all} reports per-instance inference and update times (mean $\pm$ std over $10$ runs).
Inference costs for \AVDTbet{} remain comparable to HT, since prediction involves only standard tree traversal.
The primary computational overhead arises during updates, where betting wealth is maintained for each candidate split.

As expected, betting-based methods incur higher update costs due to (i) maintaining shadow challenger statistics, and (ii) estimating the Universal Portfolio wealth.

Despite this additional cost, runtimes remain within practical bounds.
For most datasets, update times stay within a few milliseconds per instance.
Even on more demanding streams (e.g., \texttt{nzenergy}), costs remain compatible with real-time processing.

Two factors further mitigate this overhead.
First, \AVARFbet{} typically learns shallower trees, reducing the number of nodes requiring update; for example, it is even faster than its single-tree counterpart \AVDTbet{} on \texttt{nzenergy}. Second, computations across candidate splits and ensemble members are independent, making the method naturally amenable to parallelization.
Overall, the results demonstrate a favorable computational trade-off for the observed performance gains.

\section{Conclusion}

We introduced an anytime-valid framework for split selection in online decision trees.
Our approach corrects a fundamental limitation of existing methods based on fixed-sample concentration bounds by providing statistical guarantees that remain valid under data-dependent stopping rules, including non-stationary and dependent data streams.
The resulting procedure serves as a drop-in replacement for standard Hoeffding Trees and integrates seamlessly into ensemble methods such as Adaptive Random Forests, a state-of-the-art bagging-based approach for data stream learning.
Empirical evaluations on a diverse suite of real-world streaming datasets demonstrate that the resulting ensembles consistently outperform existing baselines while producing substantially smaller trees.

\begin{figure*}[!ht]
  \centering
  \setlength{\tabcolsep}{2pt}
  \renewcommand{\arraystretch}{0.9}

  \begin{tabular}{ccc}

  \begin{subfigure}[b]{0.31\textwidth}
    \centering
    \includegraphics[width=1\textwidth]{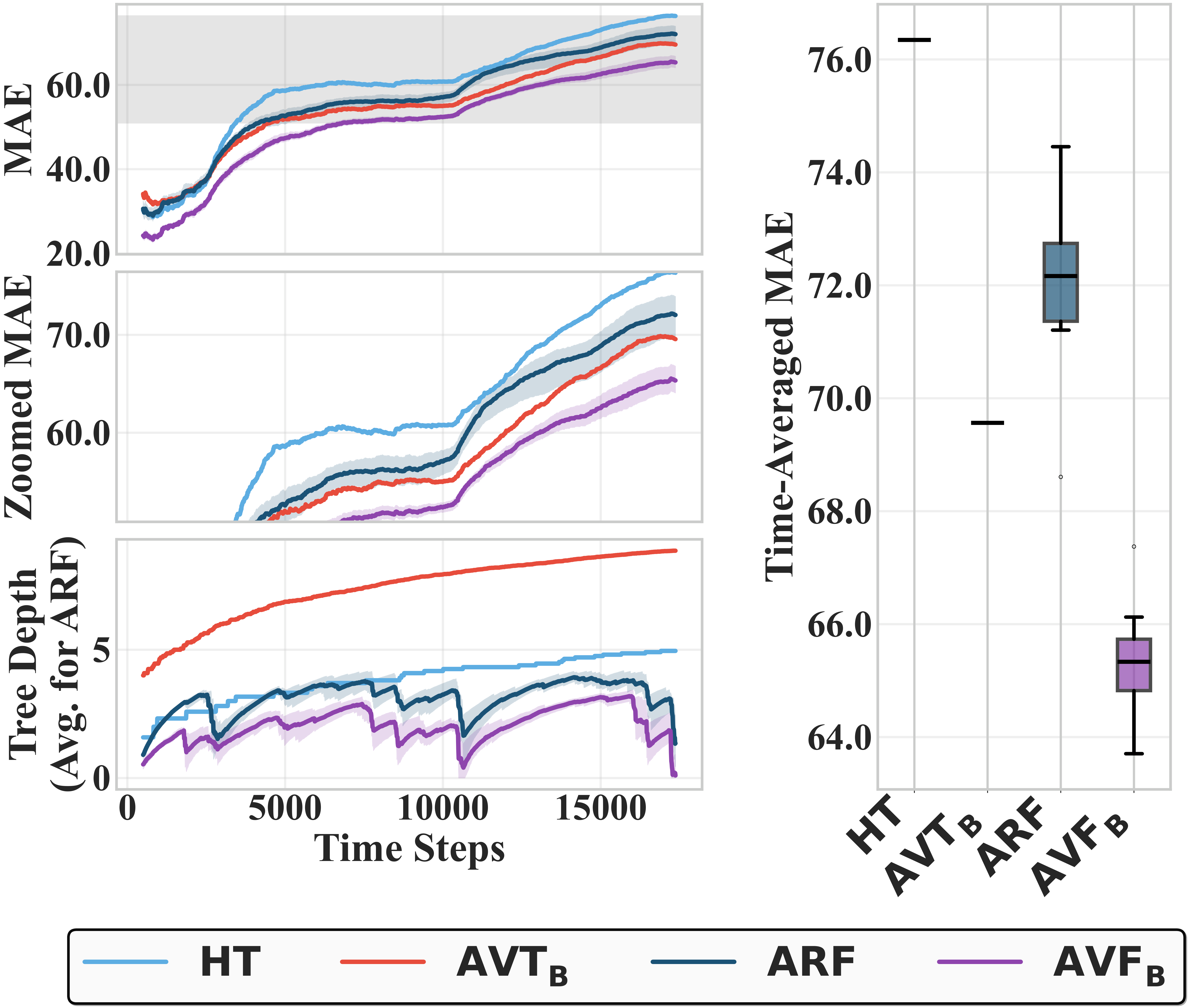}
    \caption{bike}
    \label{fig:reg_bike}
  \end{subfigure} &
  \begin{subfigure}[b]{0.31\textwidth}
    \centering
    \includegraphics[width=\textwidth]{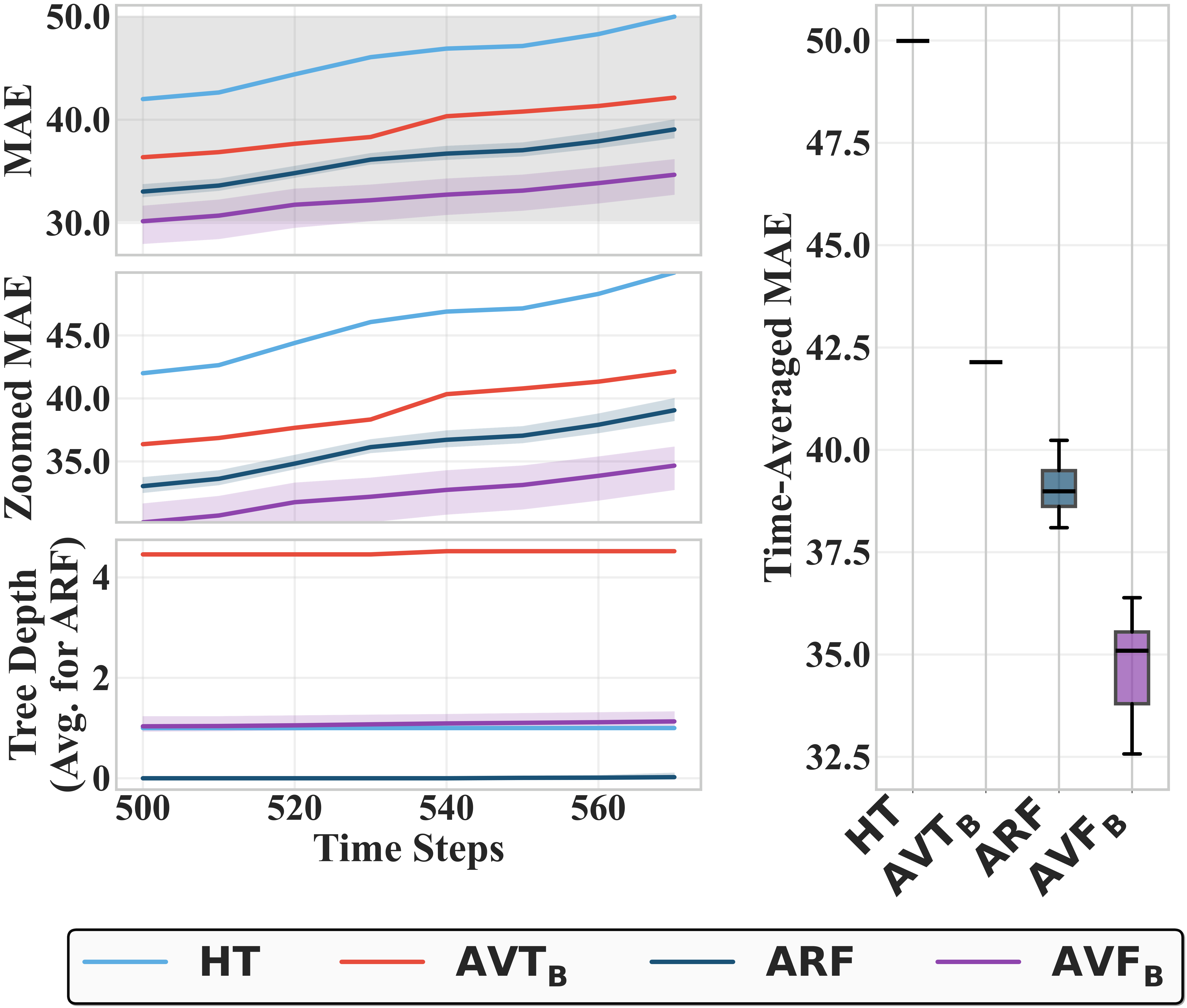}
    \caption{chick}
    \label{fig:reg_chick}
  \end{subfigure} &
  \begin{subfigure}[b]{0.31\textwidth}
    \centering
    \includegraphics[width=\textwidth]{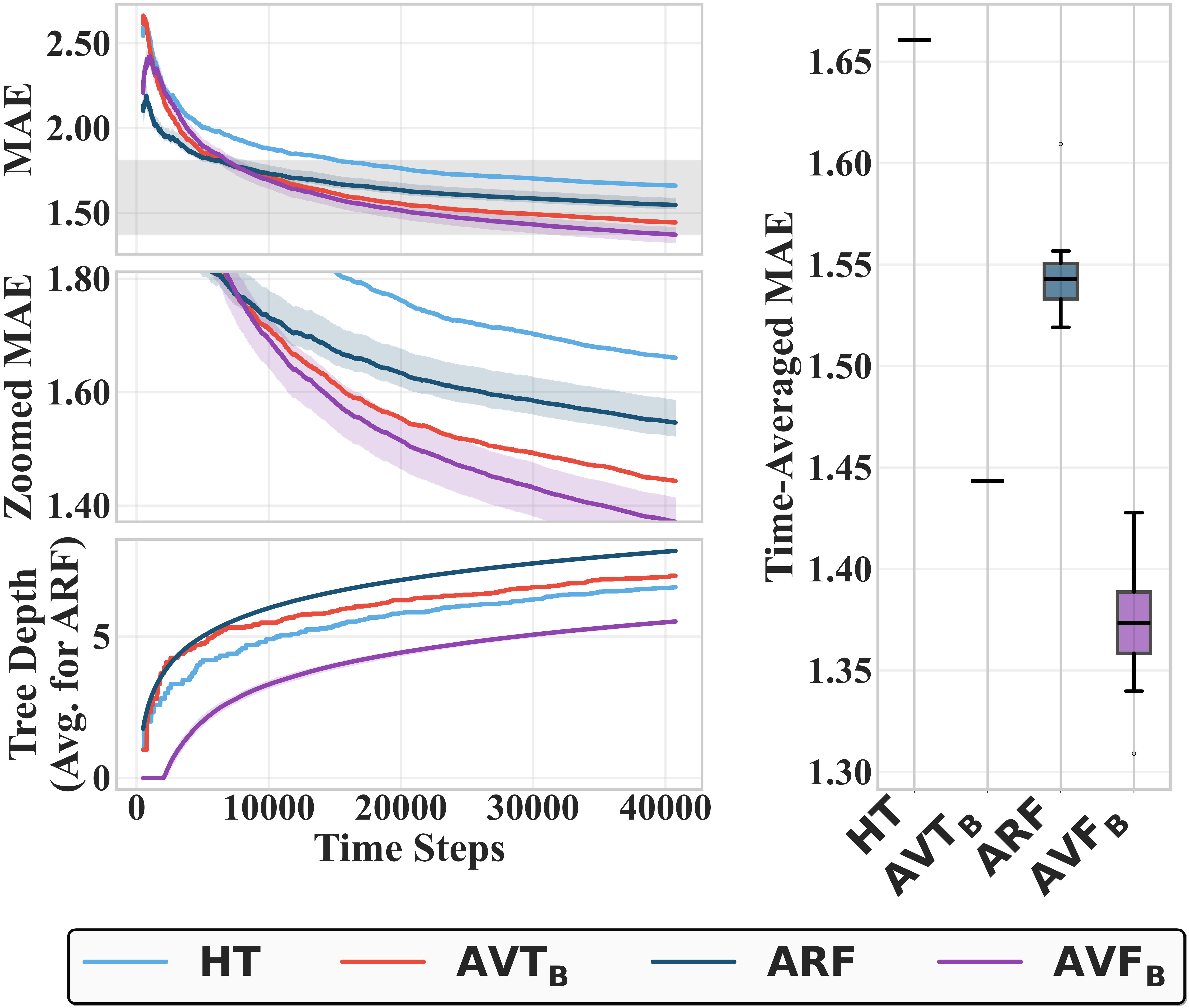}
    \caption{fried}
    \label{fig:reg_fried}
  \end{subfigure} \\[4pt]

  \begin{subfigure}[b]{0.31\textwidth}
    \centering
    \includegraphics[width=\textwidth]{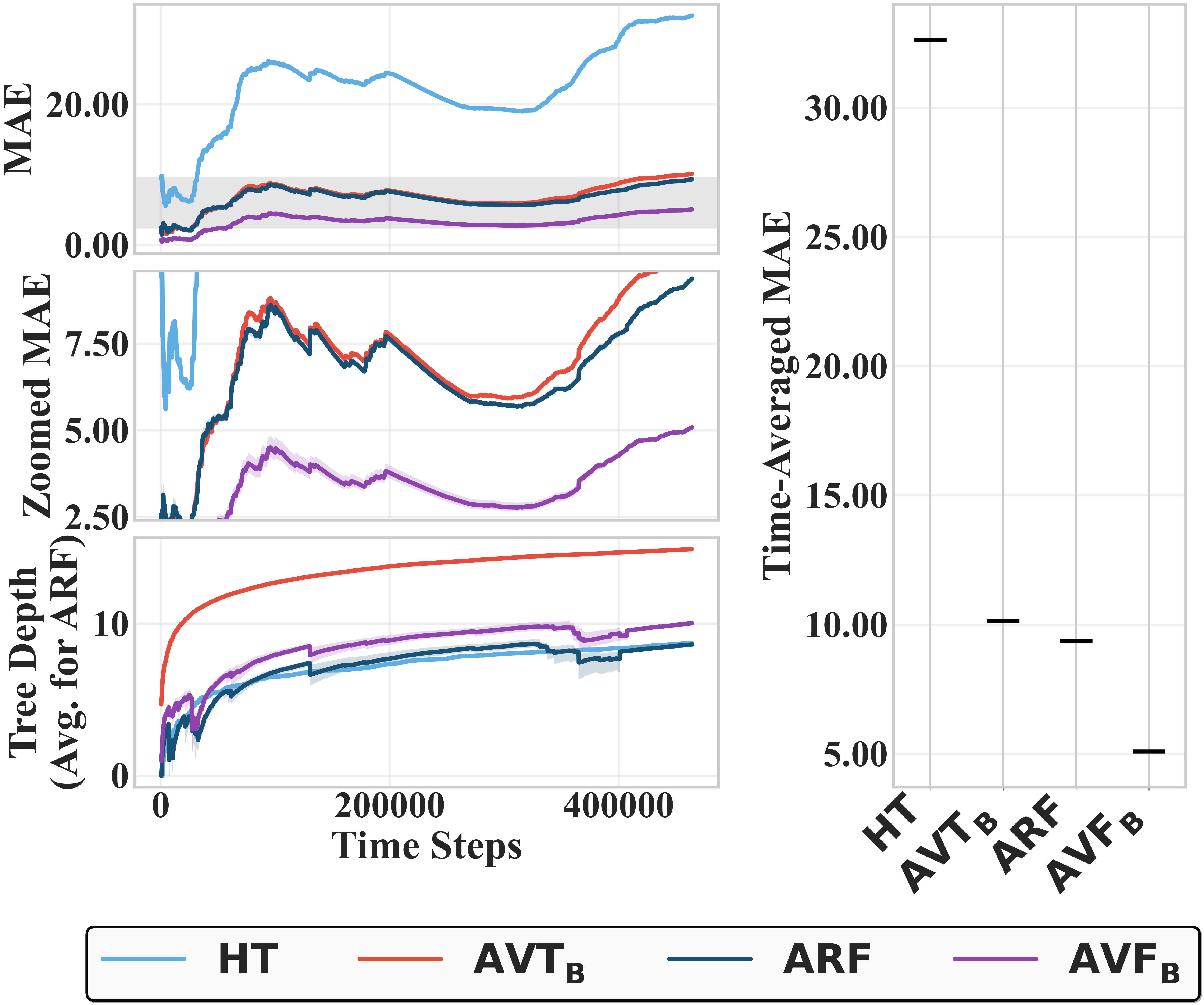}
    \caption{nzenergy}
    \label{fig:reg_nzenergy}
  \end{subfigure} &
  \begin{subfigure}[b]{0.31\textwidth}
    \centering
    \includegraphics[width=\textwidth]{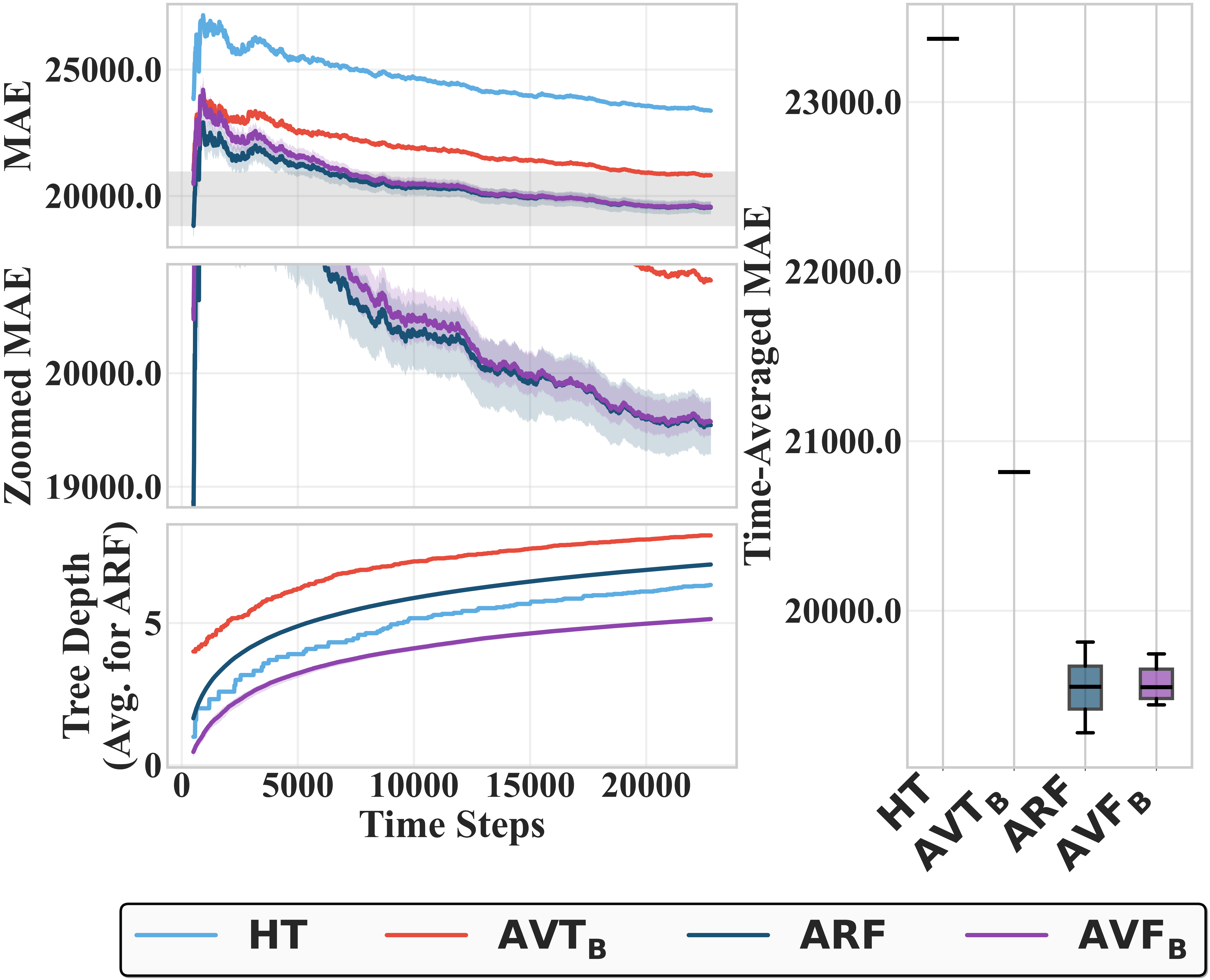}
    \caption{house}
    \label{fig:reg_house}
  \end{subfigure} &
  \begin{subfigure}[b]{0.31\textwidth}
    \centering
    \includegraphics[width=\textwidth]{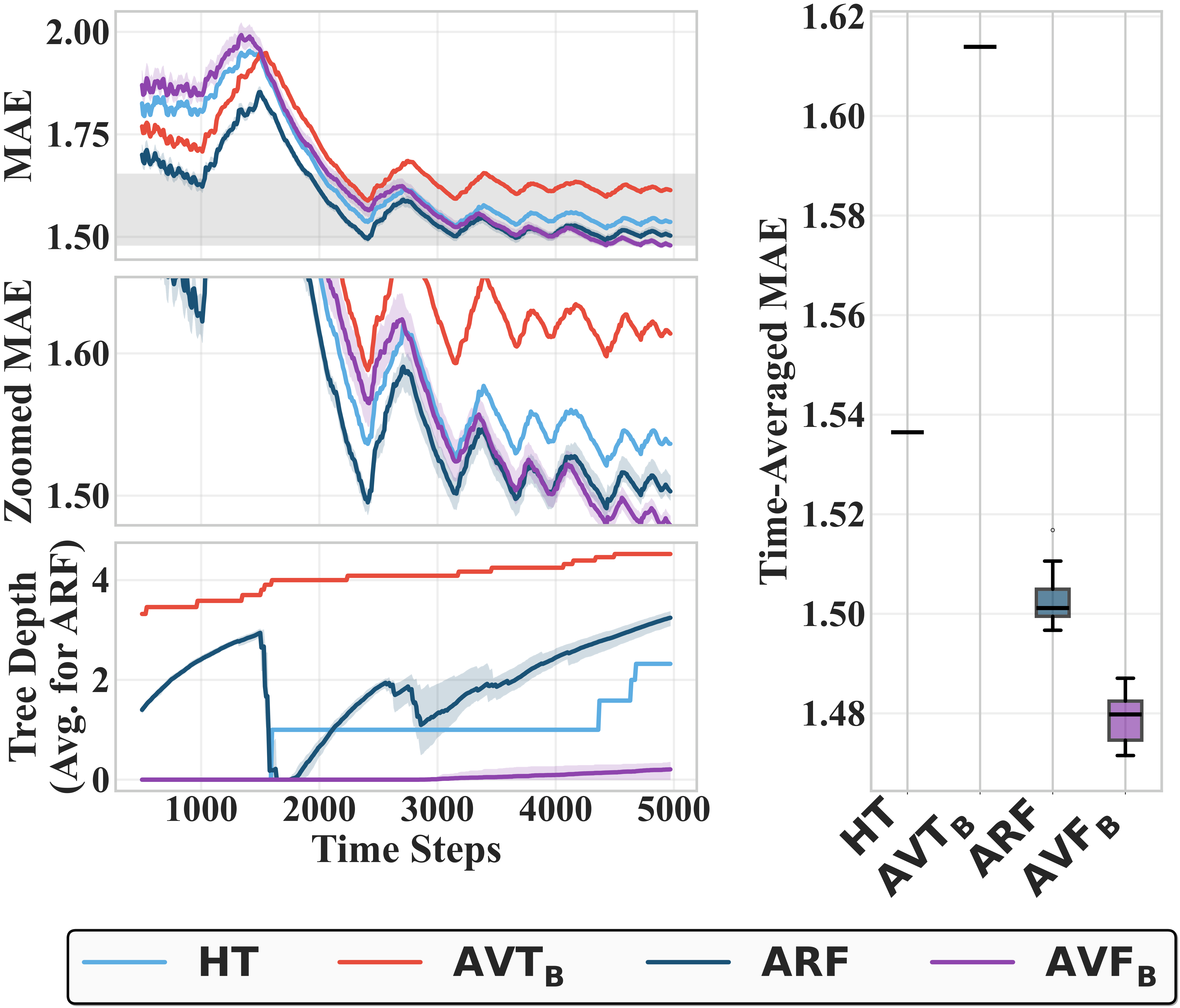}
    \caption{abalone}
    \label{fig:reg_abalone}
  \end{subfigure} \\[6pt]

  \begin{subfigure}[b]{0.31\textwidth}
    \centering
    \includegraphics[width=\textwidth]{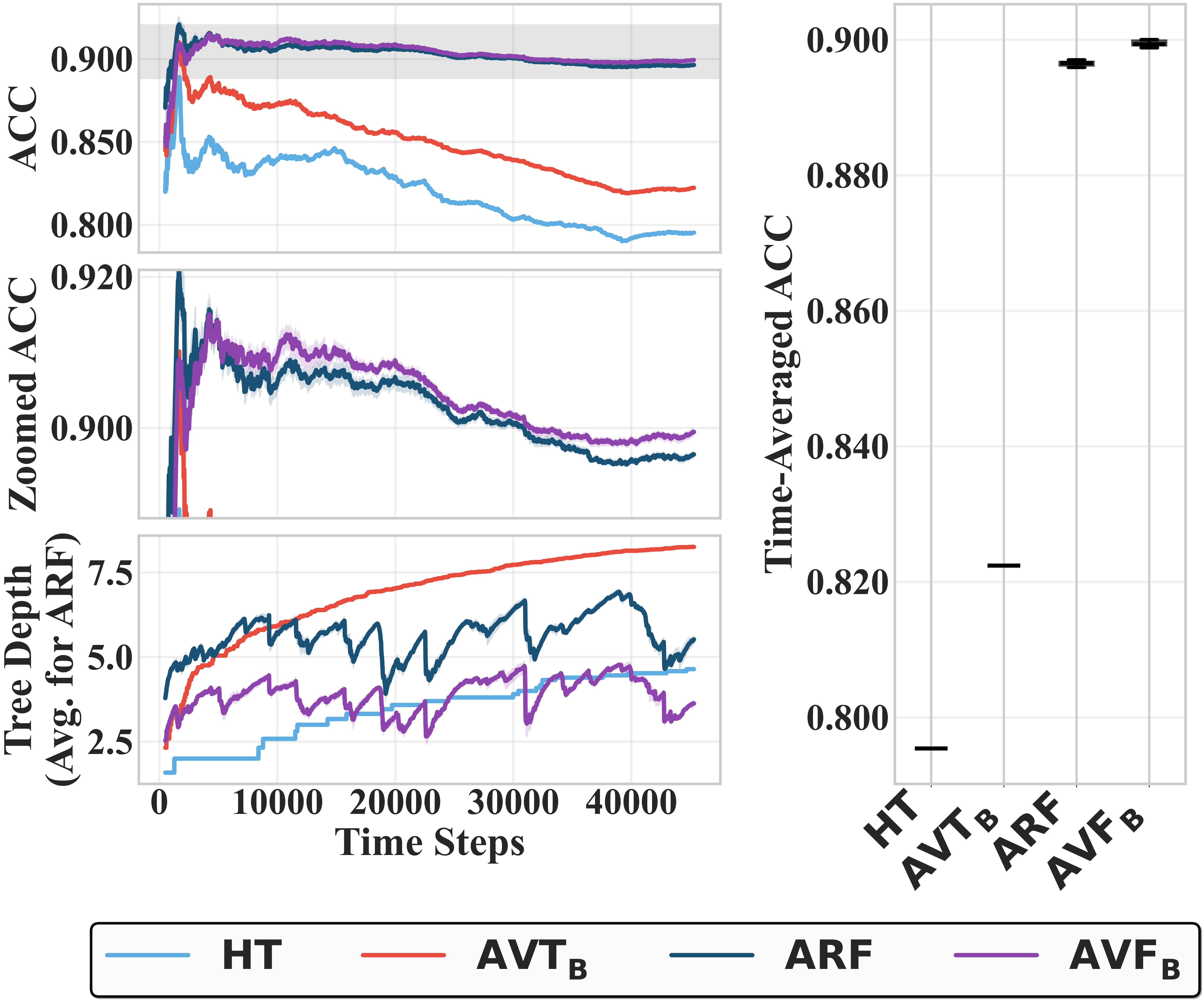}
    \caption{elec2}
    \label{fig:cls_elec2}
  \end{subfigure} &
  \begin{subfigure}[b]{0.31\textwidth}
    \centering
    \includegraphics[width=\textwidth]{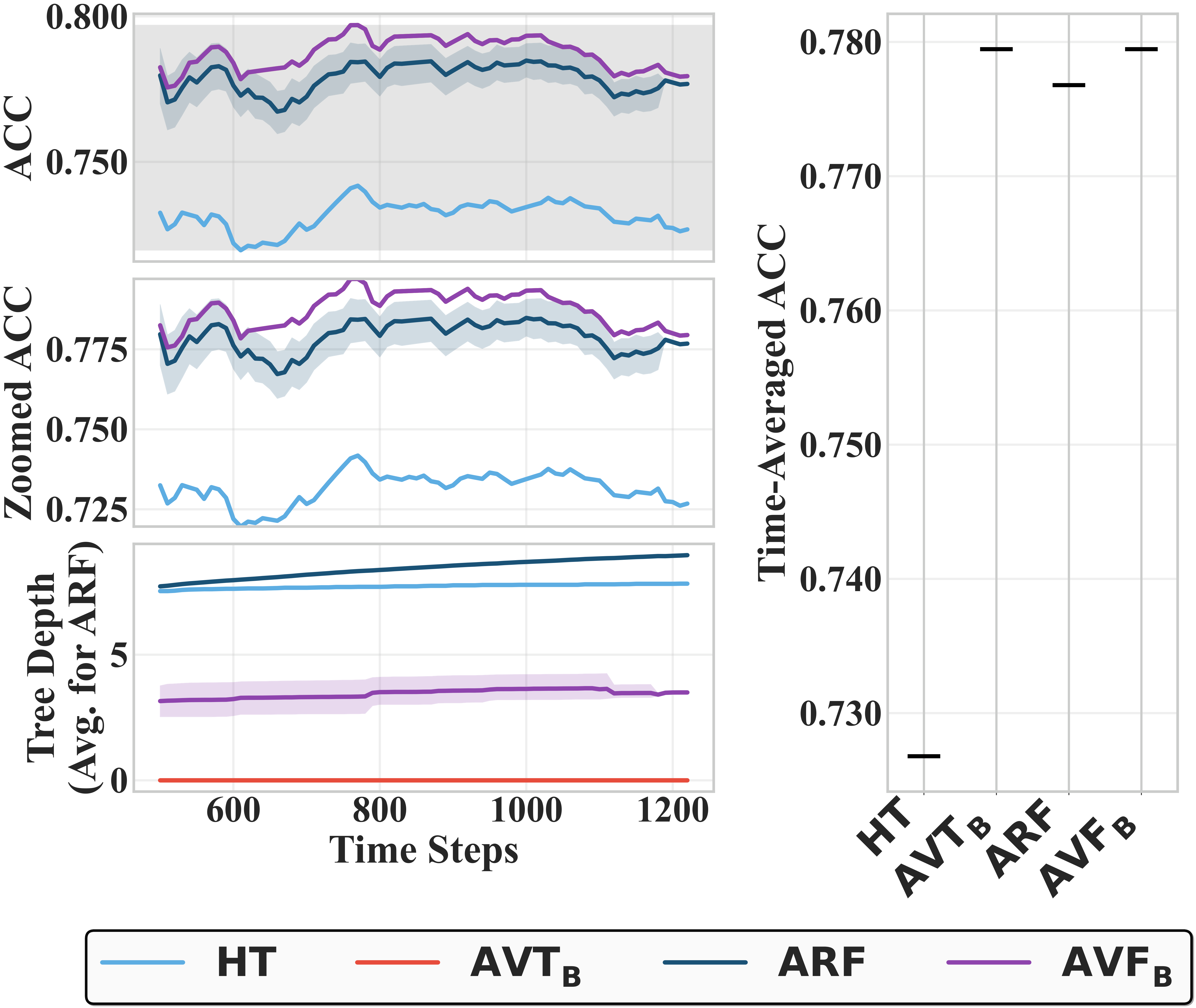}
    \caption{airlines}
    \label{fig:cls_airlines}
  \end{subfigure} &
  \begin{subfigure}[b]{0.31\textwidth}
    \centering
    \includegraphics[width=\textwidth]{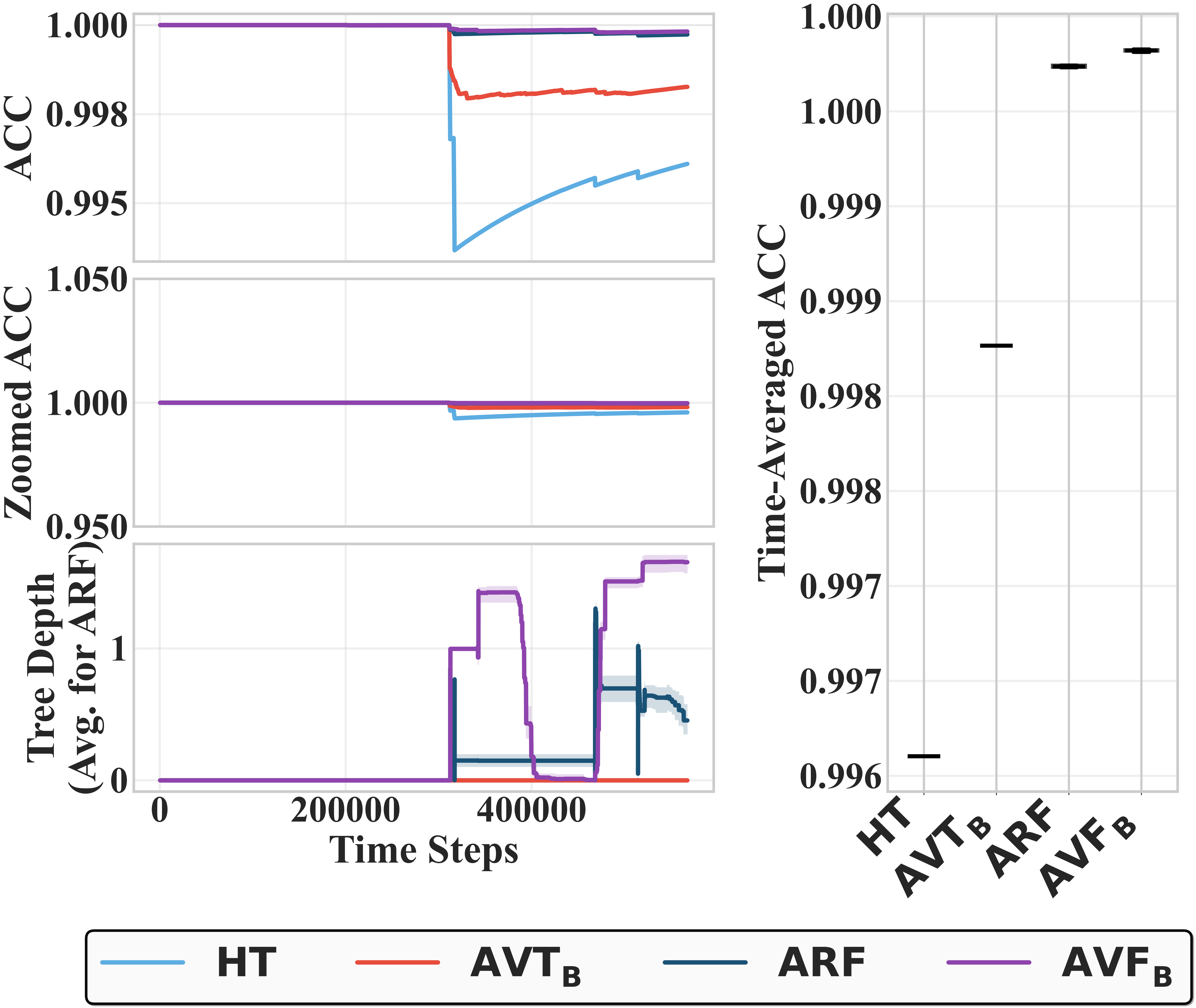}
    \caption{http-KDD99}
    \label{fig:cls_http}
  \end{subfigure} \\[4pt]

  \begin{subfigure}[b]{0.31\textwidth}
    \centering
    \includegraphics[width=\textwidth]{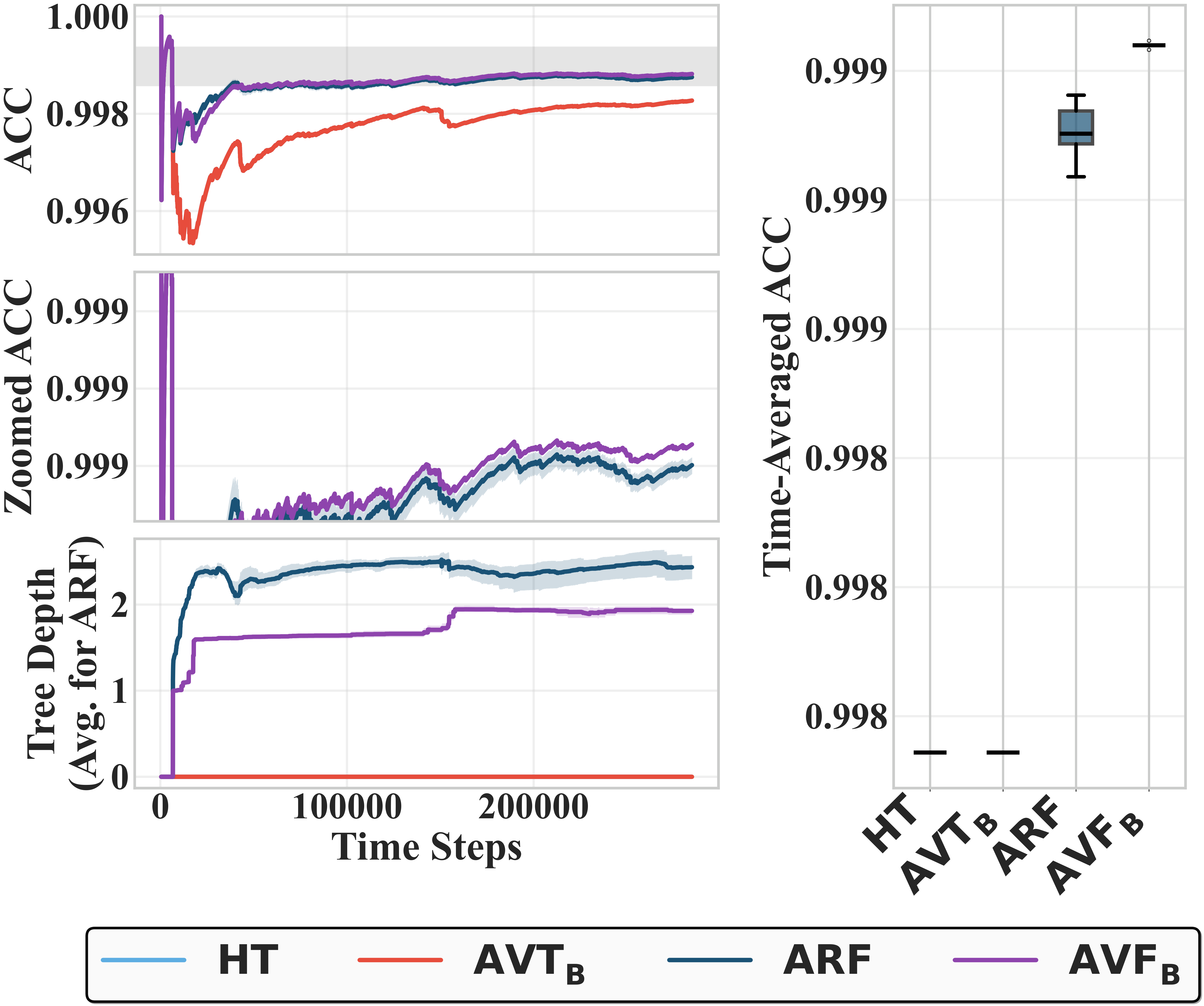}
    \caption{creditcard}
    \label{fig:cls_creditcard}
  \end{subfigure} &
  \begin{subfigure}[b]{0.31\textwidth}
    \centering
    \includegraphics[width=\textwidth]{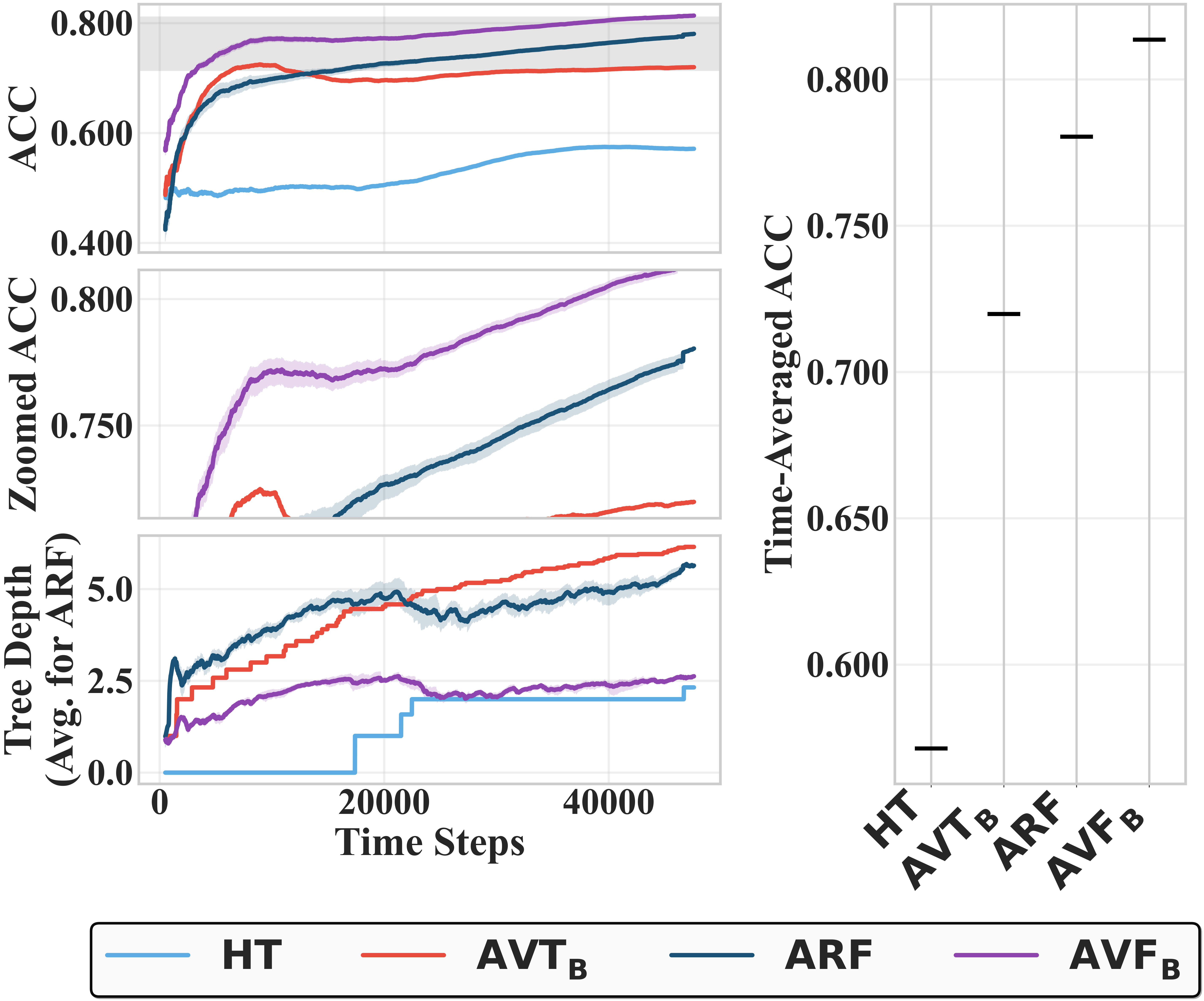}
    \caption{rbfm100k}
    \label{fig:cls_rbfm100k}
  \end{subfigure} &
  \begin{subfigure}[b]{0.31\textwidth}
    \centering
    \includegraphics[width=\textwidth]{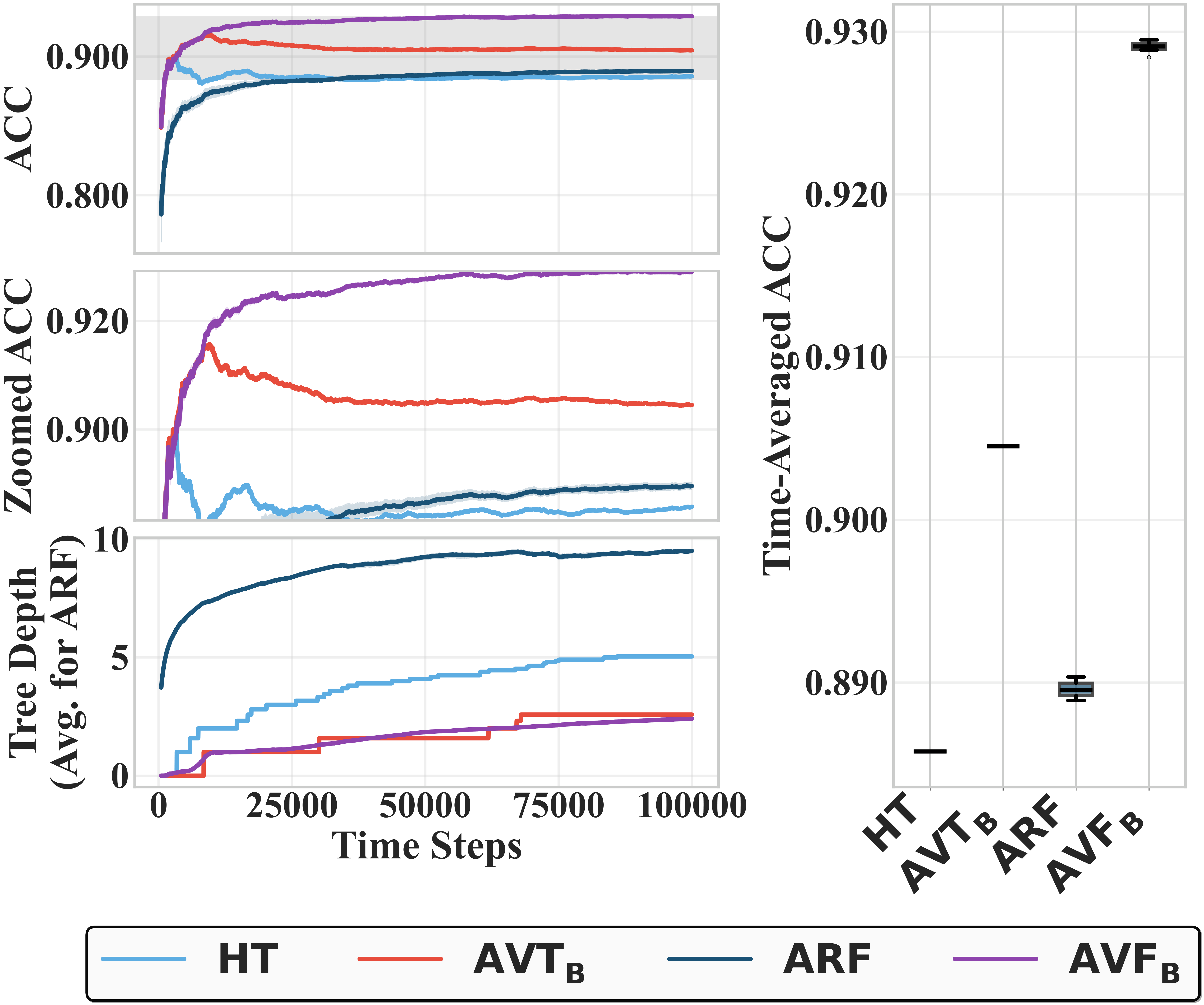}
    \caption{hyper100k}
    \label{fig:cls_hyper100k}
  \end{subfigure}

  \end{tabular}

  \caption{Regression datasets \textbf{(a)–(f)} and classification datasets \textbf{(g)–(l)}. Each panel shows prequential performance over $10$ runs with $95\%$ confidence intervals, zoomed ambiguous regions, average tree depth, and aggregated metrics.}

  \label{fig:main_all_datasets_merged}
\end{figure*}

\section*{Disclaimer}
This paper was prepared for informational
purposes by the Artificial Intelligence Research group of
JPMorgan Chase \& Co., and its affiliates (``JP Morgan'')
and is not a product of the Research Department of JP
Morgan. JP Morgan makes no representation, and warranty
whatsoever, and disclaims all liability, for the completeness,
accuracy or reliability of the information contained herein.
This document is not intended as investment research or
investment advice, or a recommendation, offer or solicitation for the purchase or sale of any security, financial
instrument, financial product or service, or to be used in
any way for evaluating the merits of participating in any transaction, and shall not constitute a solicitation under any jurisdiction or to any person, if such solicitation under such jurisdiction or to such person would be unlawful.

\textit{© 2026 JP Morgan Chase \& Co. All rights reserved.}








\section*{Impact Statement}

This paper presents work whose goal is to advance the field of Machine
Learning. There are many potential societal consequences of our work, none
which we feel must be specifically highlighted here.

\nocite{langley00}

\bibliography{example_paper}
\bibliographystyle{icml2026}

\newpage
\appendix
\onecolumn

\section{Proofs}
\label{app:proofs}

\subsection{Basic measurability and boundedness}
Fix a node--candidate pair $(v,c)$ and a test start time $s:=s^{v,c}$.
By construction, the incumbent and challenger predictions at time $t$
depend only on data observed up to time $t-1$ and any algorithmic randomness
revealed by then; thus they are $\F_{t-1}$-measurable. Since the loss $\ell$
is deterministic given a prediction and $Y_t$, the difference
$\Delta_t^{v,c}=\ell(m^v_{t-1},Y_t)-\ell(m^{v_c}_{t-1},Y_t)$ is $\F_t$-measurable.

Assume $\Delta_t^{v,c}\in[-1,1]$ a.s. and a predictable betting fraction
$\beta_t^{v,c}\in[0,1]$, i.e.\ $\beta_t^{v,c}$ is $\F_{t-1}$-measurable.
Then the multiplicative factor satisfies
$1+\beta_t^{v,c}\Delta_t^{v,c}\in[0,2]$ a.s., hence the wealth
\[
W_{s-1}^{v,c}=1,\qquad
W_t^{v,c}=W_{t-1}^{v,c}\bigl(1+\beta_t^{v,c}\Delta_t^{v,c}\bigr),\ t\ge s,
\]
is well-defined and nonnegative.

\subsection{Proof of Lemma~\ref{lem:anytime_valid_camready}}

\begin{lemma}[Anytime-valid single-split test]
\label{lem:anytime_valid_camready}
Under $H_0^{v,c}$, the process $(W_t^{v,c})_{t\ge s^{v,c}-1}$ is a nonnegative
$\F_t$-supermartingale. Consequently, for any $\alpha^{v,c}\in(0,1)$,
\[
\mathbb{P}_{H_0^{v,c}}\!\left(\sup_{t\ge s^{v,c}} W^{v,c}_t \ge \frac{1}{\alpha^{v,c}}\right)
\le \alpha^{v,c},
\qquad\text{and hence}\qquad
\mathbb{P}_{H_0^{v,c}}(\tau^{v,c}<\infty)\le \alpha^{v,c}.
\]
\end{lemma}

Under the null hypothesis $H_0^{v,c}$ we have
$\E[\Delta_t^{v,c}\mid\F_{t-1}]\le 0$ for all $t\ge s$.
Using predictability of $\beta_t^{v,c}$ and $\F_{t-1}$-measurability of $W_{t-1}^{v,c}$,
\begin{align*}
\E[W_t^{v,c}\mid \F_{t-1}]
&=
\E\!\left[W_{t-1}^{v,c}\bigl(1+\beta_t^{v,c}\Delta_t^{v,c}\bigr)\,\Big|\,\F_{t-1}\right]\\
&=
W_{t-1}^{v,c}\Bigl(1+\beta_t^{v,c}\E[\Delta_t^{v,c}\mid\F_{t-1}]\Bigr)
\;\le\; W_{t-1}^{v,c}.
\end{align*}
Thus $(W_t^{v,c})_{t\ge s-1}$ is a nonnegative $\F_t$-supermartingale.

We now invoke Ville's inequality: if $(W_t)_{t\ge 0}$ is a nonnegative supermartingale
with $W_0=1$, then for any $a>0$,
$\Pp(\sup_{t\ge 0} W_t \ge a)\le 1/a$.
Applying this with $a=1/\alpha^{v,c}$ yields
\[
\Pp_{H_0^{v,c}}\!\left(\sup_{t\ge s} W_t^{v,c}\ge \frac{1}{\alpha^{v,c}}\right)
\le \alpha^{v,c}.
\]
Finally, $\{\tau^{v,c}<\infty\}\subseteq \{\sup_{t\ge s}W_t^{v,c}\ge 1/\alpha^{v,c}\}$,
so $\Pp_{H_0^{v,c}}(\tau^{v,c}<\infty)\le \alpha^{v,c}$ as claimed. 

We can prove similarly the same result for the weak hypothesis $H_{\mathrm w,0}^{v,c}$.
\qed

\subsection{Countability under adaptive tree growth}
We record a simple sufficient condition for countability, used by
Theorem~\ref{thm:global_control_camready}. Suppose (i) at each time $t$ the
algorithm maintains finitely many current nodes, and (ii) each node $v$ has a
finite candidate class $C_v$. Then the set of all tests ever instantiated up to
any finite time $t$ is finite; hence the set of tests over all times
$t\in\mathbb{N}$ is a countable union of finite sets and therefore countable.

\begin{theorem}[Global false-split control]
\label{thm:global_control_camready}
Assume that the set of all tests ever instantiated by the (possibly adaptive)
tree-growing algorithm is almost surely countable, and that each level
$\alpha^{v,c}$ is fixed at (or before) time $s^{v,c}$ (i.e.\ $\alpha^{v,c}$ is
$\F_{s^{v,c}-1}$-measurable). If the allocation satisfies
\[
\sum_{(v,c)} \alpha^{v,c} \le \alpha,
\]
then the probability of \emph{any} false split over the entire lifetime of the tree
is at most $\alpha$:
\[
\mathbb{P}\!\left(\exists\,(v,c): H_0^{v,c}\ \text{holds and}\ \tau^{v,c}<\infty\right)
\le \alpha.
\]
\end{theorem}

\subsection{Proof of Theorem~\ref{thm:global_control_camready}}
Let $\mathcal{T}$ denote the (a.s.\ countable) set of all tests instantiated by the
algorithm, indexed by pairs $(v,c)$. For each test define the ``false rejection''
event
\[
E^{v,c}:=\{H_0^{v,c}\ \text{holds and}\ \tau^{v,c}<\infty\}.
\]
By Lemma~\ref{lem:anytime_valid_camready}, $\Pp(E^{v,c})\le \alpha^{v,c}$ for each
$(v,c)\in\mathcal{T}$, regardless of dependence across tests. Hence by the union bound
over the countable family,
\[
\Pp\!\left(\bigcup_{(v,c)\in\mathcal{T}} E^{v,c}\right)
\;\le\;
\sum_{(v,c)\in\mathcal{T}} \Pp(E^{v,c})
\;\le\;
\sum_{(v,c)\in\mathcal{T}} \alpha^{v,c}
\;\le\; \alpha.
\]
This proves the stated global control.
\qed

\newpage 
\subsection{Proof of Theorem~\ref{thm:monotone_model}}
\label{app:proofs_monotone_commits}

We prove monotone expected performance of the \emph{deployed} predictor
$(\hat m_t)_{t\ge 0}$ produced by the procedure under i.i.d.\ data and convex
losses, on the event $\mathcal E$ that no false split is ever committed. For clarity, we focus on the classification setting; however, the result
extends directly to any bounded-output prediction problem. Recall the theorem says: 

\begin{theorem}[Strict monotonicity between and at commit times]
\label{app:thm:monotone_model}
Assume $(X,Y)\sim P$ are i.i.d.\ and $\ell(\hat y,y)$ is convex in its first argument.
Run Algorithm~\ref{alg:avht} and let $0=\tau_0<\tau_1<\tau_2<\cdots$ denote the (random)
commit times, and write $\hat m_t$ for the deployed predictor.

\emph{(Between commits.)}
For all $k\ge 0$ and $t\in(\tau_k,\tau_{k+1})$,
\[
\mathbb E\!\left[\ell(\hat m_t(X),Y)\right]
\;\leq\;
\mathbb E\!\left[\ell(\hat m_{t-1}(X),Y)\right].
\]

\emph{(At commits.)}
There exists $\varepsilon>0$ such that if splits are committed
only when the weak test certifies an advantage exceeding $\varepsilon$, then with
probability at least $1-\alpha$, for every commit time $\tau_k$,
\[
\mathbb E\!\left[\ell(\hat m_{\tau_k}(X),Y)\right]
\;\le\;
\mathbb E\!\left[\ell(\hat m_{\tau_k^-}(X),Y)\right]
\]
where $\hat m_{\tau_k^-}$ denotes the pre-commit model.
\end{theorem}

\paragraph{Standing notation.}
Let $(\mathcal F_t)_{t\ge 0}$ be the natural filtration generated by the stream
and the internal randomness of the algorithm.
At each time $t\ge 0$, $\hat m_t$ denotes the deployed tree predictor after
processing $t$ samples; hence $\hat m_t$ is $\mathcal F_t$-measurable.
Let $(X,Y)\sim P$ be an independent fresh draw, independent of $\mathcal F_t$.
Define the (conditional) one-step-ahead risk
\[
R_t
\;:=\;
\mathbb E\!\left[\ell(\hat m_t(X),Y)\,\middle|\,\mathcal F_t\right],
\qquad
\bar R_t
\;:=\;
\mathbb E\!\left[\ell(\hat m_t(X),Y)\,\middle|\,\mathcal E\right].
\]
Note that $R_t$ is a random variable (measurable w.r.t.\ $\mathcal F_t$), while
$\bar R_t$ is a deterministic scalar.

We decompose each update $t\mapsto t+1$ into two parts:
(i) a \emph{parameter update} of the leaf statistics (empirical probabilities),
and (ii) possibly a \emph{structural update} (a committed split).
We show that each part is non-increasing in expected risk on $\mathcal E$,
which implies the claimed monotonicity.

\subsubsection{Step 1: Marshall--Proschan implies monotone plug-in risk}

The next lemma is the key tool controlling the fact that for convex losses, the expected risk of an empirical-mean predictor is nonincreasing with sample size.

\begin{lemma}[Marshall--Proschan \cite{marshall1964inequality, mattei2025ensembles} monotonicity for empirical means]
\label{lem:MP_mean_app}
Let $Z_1,\dots,Z_n$ be exchangeable random variables in $[0,1]$ and let
$L:[0,1]\to\mathbb R$ be convex.
Define $\bar Z_n:=\frac1n\sum_{i=1}^n Z_i$.
Then for all $n\ge 2$,
\[
\mathbb E\!\left[L(\bar Z_n)\right]
\;\le\;
\mathbb E\!\left[L(\bar Z_{n-1})\right].
\]
\end{lemma}

\begin{proof}
Let $\mathbf 1\in\mathbb R^n$ denote the all-ones vector and define
\[
U
=
\begin{pmatrix}
0 & 1 & 1 & \cdots & 1 \\
1 & 0 & 1 & \cdots & 1 \\
1 & 1 & 0 & \cdots & 1 \\
\vdots & \vdots & \vdots & \ddots & \vdots \\
1 & 1 & 1 & \cdots & 0
\end{pmatrix}
\in \mathbb R^{n\times n}.
\]
Then $U\mathbf 1=(n-1)\mathbf 1$, hence $\mathbf 1=\frac{1}{n-1}U\mathbf 1$.

Let $Z=(Z_1,\dots,Z_n)^\top$. Note that
\[
\bar Z_n \;=\;\frac1n \mathbf 1^\top Z.
\]
Moreover, the $j$-th component of $U^\top Z$ equals $\sum_{i\neq j} Z_i$, hence
the vector of leave-one-out empirical means is $\frac{1}{n-1}U^\top Z$, and
\[
\frac1n \mathbf 1^\top\Big(\frac{1}{n-1}U^\top Z\Big)
=
\frac{1}{n(n-1)}(U\mathbf 1)^\top Z
=
\frac{1}{n}\mathbf 1^\top Z
=
\bar Z_n.
\]
Therefore
\[
\bar Z_n
=
\frac{1}{n}\sum_{j=1}^n
\bar Z_{n-1}^{(-j)},
\qquad
\bar Z_{n-1}^{(-j)}
:=
\frac{1}{n-1}\sum_{i\neq j} Z_i,
\]
i.e., $\bar Z_n$ is a uniform convex combination of the leave-one-out means.

Since $L$ is convex, Jensen’s inequality gives
\[
L(\bar Z_n)
\le
\frac{1}{n}\sum_{j=1}^n
L\!\left(\bar Z_{n-1}^{(-j)}\right).
\]
Taking expectations and using exchangeability (the law of $\bar Z_{n-1}^{(-j)}$
does not depend on $j$) yields
\[
\mathbb E[L(\bar Z_n)]
\le
\frac{1}{n}\sum_{j=1}^n \mathbb E\!\left[L(\bar Z_{n-1}^{(-j)})\right]
=
\mathbb E\!\left[L(\bar Z_{n-1}^{(-1)})\right]
=
\mathbb E\!\left[L(\bar Z_{n-1})\right],
\]
where the last equality uses $(Z_1,\dots,Z_n)\stackrel{d}{=}(Z_2,\dots,Z_n)$ under exchangeability.
\end{proof}

\begin{lemma}[Plug-in risk is nonincreasing in sample size]
\label{lem:plugin_risk_monotone_app}
Assume $Y_1,Y_2,\dots$ are i.i.d.\ Bernoulli$(p)$.
Let $\hat p_n := \frac1n\sum_{i=1}^n Y_i$.
Assume that for each $y\in\{0,1\}$, the map $\hat y\mapsto \ell(\hat y,y)$ is convex.
Let $Y\sim\mathrm{Bernoulli}(p)$ be an independent fresh draw, independent of
$(Y_i)_{i\ge 1}$.
Then for all $n\ge 1$,
\[
\mathbb E\!\left[\ell(\hat p_{n+1},Y)\right]
\;\le\;
\mathbb E\!\left[\ell(\hat p_{n},Y)\right].
\]
\end{lemma}

\begin{proof}
Fix $y\in\{0,1\}$ and define the convex function $L_y(\hat y):=\ell(\hat y,y)$.
Apply Lemma~\ref{lem:MP_mean_app} to the exchangeable variables
$Y_1,\dots,Y_{n+1}$ with $L=L_y$:
\[
\mathbb E\!\left[L_y\!\left(\frac{1}{n+1}\sum_{i=1}^{n+1}Y_i\right)\right]
\le
\mathbb E\!\left[L_y\!\left(\frac{1}{n}\sum_{i=1}^{n}Y_i\right)\right].
\]
Since $\hat p_{n+1}=\frac{1}{n+1}\sum_{i=1}^{n+1}Y_i$ and
$\hat p_n=\frac{1}{n}\sum_{i=1}^{n}Y_i$, this becomes
$\mathbb E[\ell(\hat p_{n+1},y)]\le \mathbb E[\ell(\hat p_n,y)]$.
Now average over the independent $Y\sim\mathrm{Bernoulli}(p)$ to conclude
\[
\mathbb E\!\left[\ell(\hat p_{n+1},Y)\right]
\le
\mathbb E\!\left[\ell(\hat p_{n},Y)\right].
\]
\end{proof}

\subsubsection{Step 2: Parameter updates (no split) cannot increase expected risk}

We now lift Lemma~\ref{lem:plugin_risk_monotone_app} from a single Bernoulli mean
to the full tree. When no split is committed, each new observation updates the
statistics of exactly one leaf, and convexity ensures that this update cannot
increase the expected risk.

\begin{lemma}[Leaf-statistics update is risk-nonincreasing]
\label{lem:leaf_update_monotone_app}
Let $(X,Y)\sim P$ be a fresh draw, independent of the data stream.
Fix a time $t\ge 1$ such that no structural change (split commit) occurs between
times $t-1$ and $t$; only leaf statistics are updated using the observation
$(X_t,Y_t)$.
Then
\[
\mathbb E\!\left[\ell(\hat m_t(X),Y)\right]
\;\le\;
\mathbb E\!\left[\ell(\hat m_{t-1}(X),Y)\right].
\]
\end{lemma}

\begin{proof}
Fix such a time $t$ and let $\mathcal L:=\mathcal L_{t-1}$ be the (random) set of
leaves of the deployed tree at time $t-1$, with associated regions
$\{A(u):u\in\mathcal L\}$. Since there is no structural update between $t-1$ and
$t$, the deployed partition $\{A(u)\}_{u\in\mathcal L}$ is the same at both
times.

\paragraph{A sigma-field that fixes the deployed partition.}
For each leaf $u\in\mathcal L$, let $\tau(u)\le t-1$ denote its (random) creation
time, and define the post-creation index set and corresponding sample count
\[
I_{t-1}(u)
:=
\{r:\ \tau(u)< r\le t-1,\ X_r\in A(u)\},
\qquad
N_{t-1}(u):=|I_{t-1}(u)|.
\]
We define the sigma-field $\mathcal H$ generated by the tree structure and split thresholds up to the creation of the leaf $u$. By construction, conditional on $\mathcal H$ the deployed partition
$\{A(u):u\in\mathcal L\}$ and all associated routing information are fixed and
deterministic. Moreover, for each leaf $u$, the empirical statistic
$\hat p_{t-1}(u)$ depends only on labels $(Y_r)_{r\in I_{t-1}(u)}$ observed
\emph{after} the creation of $u$. In particular, these labels are independent of the samples used to determine the
partition and to perform the split-selection tests.

\paragraph{Leafwise decomposition of the fresh risk.}
Because $\hat m_{t-1}$ is piecewise-constant on the deployed partition,
conditioning on $\mathcal H$ yields
\begin{align*}
\mathbb E\!\left[\ell(\hat m_{t-1}(X),Y)\,\middle|\,\mathcal H\right]
&=
\sum_{u\in\mathcal L}
\mathbb P\!\left(X\in A(u)\,\middle|\,\mathcal H\right)\;
\mathbb E\!\left[\ell(\hat p_{t-1}(u),Y)\,\middle|\,\mathcal H,\,X\in A(u)\right].
\end{align*}
The analogous decomposition holds at time $t$ with $\hat p_t(u)$ in place of
$\hat p_{t-1}(u)$.

\paragraph{Within each leaf, post-creation labels are i.i.d.\ given $\mathcal H$.}
By the honesty/no-reuse property mentioned above, which stated that the empirical statistics depend only on labels observed after the creation of the nodes, for every $s\ge \tau(u)+1$,
\[
\hat p_s(u)
=
\frac{1}{N_s(u)}\sum_{r\in I_s(u)} Y_r,
\qquad
I_s(u):=\{r:\ \tau(u)<r\le s,\ X_r\in A(u)\}.
\]
Fix $u\in\mathcal L$. Conditional on $\mathcal H$ and on $\{X\in A(u)\}$, the set
$A(u)$ and the time $\tau(u)$ are fixed, and the stream is i.i.d. Hence the
labels $(Y_r)_{r\in I_s(u)}$ are i.i.d.\ under
$\mathbb P(\,\cdot\mid \mathcal H,\,X\in A(u))$ with parameter
\[
p(u):=\mathbb P(Y=1\mid X\in A(u),\,\mathcal H).
\]

\paragraph{Apply Marshall--Proschan via Lemma~\ref{lem:plugin_risk_monotone_app}.}
Fix $u\in\mathcal L$ and condition further on $\{N_{t-1}(u)=n\}$, which is
$\mathcal H$-measurable. If $X_t\notin A(u)$ then leaf $u$ is not updated and
$\hat p_t(u)=\hat p_{t-1}(u)$. If $X_t\in A(u)$ then $N_t(u)=n+1$ and $\hat p_t(u)$
is the empirical mean of the $n+1$ i.i.d.\ Bernoulli$(p(u))$ labels routed to $u$.
Therefore, by Lemma~\ref{lem:plugin_risk_monotone_app},
\[
\mathbb E\!\left[\ell(\hat p_t(u),Y)\,\middle|\,\mathcal H,\,X\in A(u),\,N_{t-1}(u)=n,\,X_t\in A(u)\right]
\le
\mathbb E\!\left[\ell(\hat p_{t-1}(u),Y)\,\middle|\,\mathcal H,\,X\in A(u),\,N_{t-1}(u)=n\right],
\]
while equality holds on $\{X_t\notin A(u)\}$. Averaging over the event
$\{X_t\in A(u)\}$ and then over $N_{t-1}(u)$ yields the leafwise inequality
\[
\mathbb E\!\left[\ell(\hat p_t(u),Y)\,\middle|\,\mathcal H,\,X\in A(u)\right]
\le
\mathbb E\!\left[\ell(\hat p_{t-1}(u),Y)\,\middle|\,\mathcal H,\,X\in A(u)\right].
\]

\paragraph{Conclude by mixing over leaves and removing conditioning.}
Plugging the leafwise inequality into the leaf-mixture decomposition gives
\[
\mathbb E\!\left[\ell(\hat m_t(X),Y)\,\middle|\,\mathcal H\right]
\le
\mathbb E\!\left[\ell(\hat m_{t-1}(X),Y)\,\middle|\,\mathcal H\right].
\]
Finally, take expectations over $\mathcal H$ (tower property) to conclude
\[
\mathbb E\!\left[\ell(\hat m_t(X),Y)\right]
\le
\mathbb E\!\left[\ell(\hat m_{t-1}(X),Y)\right],
\]
as claimed.
\end{proof}

\subsubsection{Step 3: Structural updates (committed splits) reduce expected risk on $\mathcal E$}

We now handle the remaining case: the algorithm commits a split at time $t$.
Although the theorem statement does not explicitly mention incumbents and challengers,
the commit rule is defined via an anytime-valid test comparing the prequential
performance of the currently deployed leaf predictor to that of its challenger split.
We use this comparison only inside the proof to show that the commit step is
risk-improving on $\mathcal E$.

\paragraph{Prequential comparison and instantaneous advantage.}
Fix a candidate split $(v,c)$ that is defined at time $t$ (i.e., its incumbent leaf
$v$ exists in the deployed tree at time $t-1$ and the challenger children
$v_c^L,v_c^R$ are well-defined).
Both incumbent and challenger predictors are prequential and $\mathcal F_{t-1}$-measurable:
\[
m^{v}_{t-1}(x)=p_{t-1}(v),
\qquad
m^{v_c}_{t-1}(x)=
\begin{cases}
p_{t-1}(v_c^L), & x\in R(v_c^L),\\[1mm]
p_{t-1}(v_c^R), & x\in R(v_c^R).
\end{cases}
\]
After observing $(X_t,Y_t)$, define the bounded prequential loss difference
\[
\Delta_t^{v,c}
:=
\ell\!\left(m^v_{t-1}(X_t),Y_t\right)
-
\ell\!\left(m^{v_c}_{t-1}(X_t),Y_t\right)
\in[-1,1],
\qquad
\delta_t^{v,c}:=\mathbb E\!\left[\Delta_t^{v,c}\mid \mathcal F_{t-1}\right].
\]

\paragraph{Prequential mean equals fresh-draw mean.}
Let $(X,Y)\sim P$ be an independent fresh draw, independent of $\mathcal F_{t-1}$.
The next lemma allows us to interpret $\delta_t^{v,c}$ as the conditional
\emph{fresh-draw} risk gap between the incumbent and challenger, which is the
quantity needed to argue that committing the split improves risk.

\begin{lemma}[Prequential conditional mean equals fresh-draw conditional mean]
\label{lem:prequential_equals_fresh_app}
Let $g:\mathcal X\times\mathcal Y\to\mathbb R$ be such that
$g(x,y)$ is $\mathcal F_{t-1}$-measurable through its dependence on predictors at
time $t-1$. Then
\[
\mathbb E\!\left[g(X_t,Y_t)\,\middle|\,\mathcal F_{t-1}\right]
=
\mathbb E\!\left[g(X,Y)\,\middle|\,\mathcal F_{t-1}\right].
\]
\end{lemma}

\begin{proof}
Conditional on $\mathcal F_{t-1}$, the pair $(X_t,Y_t)$ is independent of
$\mathcal F_{t-1}$ and distributed as $P$, hence has the same conditional law as
the independent fresh draw $(X,Y)\sim P$. Therefore the conditional expectations
coincide.
\end{proof}

Applying Lemma~\ref{lem:prequential_equals_fresh_app} with
$g(x,y)=\ell(m^v_{t-1}(x),y)-\ell(m^{v_c}_{t-1}(x),y)$ yields the identity
\[
\delta_t^{v,c}
=
\mathbb E\!\left[
\ell(m^v_{t-1}(X),Y)-\ell(m^{v_c}_{t-1}(X),Y)
\,\middle|\,\mathcal F_{t-1}
\right],
\]
i.e., $\delta_t^{v,c}$ is the conditional fresh-draw advantage of the challenger we use in the next sections, with $\Delta_t^{v,c} = \ell(m^v_{t-1}(X),Y)-\ell(m^{v_c}_{t-1}(X),Y)$

\paragraph{A bounded-drift correction for the weak (average-advantage) test.}
The weak test controls a running average of the advantage process
$\delta_t^{v,c}$, whereas committing a split requires a guarantee on the
\emph{instantaneous} advantage at the commit time. We bridge this gap by showing
that, between committed splits, the process $t\mapsto\delta_t^{v,c}$ admits a
uniformly bounded one-step drift. This allows us to introduce a correction
threshold $\epsilon$ such that a positive average advantage implies a positive
instantaneous advantage at the commit time.

We present two ways of selecting $\epsilon$: a conservative worst-case bound
and a fully adaptive, data-dependent alternative. We emphasize that this
threshold is introduced solely for theoretical control. In practice, the
effective drift is often much smaller, and empirical results suggest that the
required correction can be negligible and in some cases even set to zero while
still preserving monotonicity.

\begin{assumption}[Minimum leaf sample size (weak-test stability)]
\label{ass:nmin_drift_weak}
We work with Brier loss $\ell(\hat y,y)=(\hat y-y)^2$ and leaf predictions and target output in $[0,1]$.
A candidate $(v,c)$ is eligible for testing or commitment only once all leaf statistics
used by the incumbent and challenger predictors on the region affected by the split
have received at least $n_{\min}\ge 1$ post-creation samples.
\end{assumption}

\begin{lemma}[Bounded one-step drift of the prequential advantage (between commits)]
\label{lem:bounded_drift_delta_weak}
Assume Brier loss and Assumption~\ref{ass:nmin_drift_weak}.
Fix a candidate $(v,c)$ and a time $t$ such that \emph{no structural update (commit)}
occurs between times $t-1$ and $t$.
Then
\[
\bigl|\delta^{v,c}_{t+1}-\delta^{v,c}_{t}\bigr|
\;\le\;
\rho,
\qquad\text{where}\qquad
\rho:=\frac{4}{n_{\min}+1},
\]
almost surely.
\end{lemma}

\begin{proof}
For each realized history, define the fresh-draw loss-gap functions
\[
h_{t-1}(x,y)
:=
\ell(m^v_{t-1}(x),y)-\ell(m^{v_c}_{t-1}(x),y),
\qquad
h_t(x,y)
:=
\ell(m^v_{t}(x),y)-\ell(m^{v_c}_{t}(x),y).
\]
Under the i.i.d. assumption and with $(X,Y)\sim P$ independent of the stream,
\[
\delta_t^{v,c}=\E[h_{t-1}(X,Y)\mid\F_{t-1}],
\qquad
\delta_{t+1}^{v,c}=\E[h_t(X,Y)\mid\F_t].
\]
Since the distribution of the fresh draw is fixed, the same identities may be
viewed as integrals of the random functions $h_{t-1}$ and $h_t$ against $P$.
Hence
\[
|\delta_{t+1}^{v,c}-\delta_t^{v,c}|
\le
\E\!\left[|h_t(X,Y)-h_{t-1}(X,Y)|\mid\F_t\right]
\le
\|h_t-h_{t-1}\|_\infty .
\]
It remains to bound the uniform change in the loss-gap function.

Because no structural update affects the comparison, only empirical-mean leaf
statistics are updated. If a leaf mean $\hat p_n\in[0,1]$ with $n\ge n_{\min}$
receives one additional label $Y_t\in[0,1]$, then
\[
|\hat p_{n+1}-\hat p_n|
=
\frac{|Y_t-\hat p_n|}{n+1}
\le
\frac{1}{n_{\min}+1}.
\]
For Brier loss on $[0,1]$,
\[
|(\hat y-y)^2-(\hat y'-y)^2|
=
|\hat y-\hat y'|\,|\hat y+\hat y'-2y|
\le
2|\hat y-\hat y'|.
\]
Thus the loss of the incumbent changes by at most $2/(n_{\min}+1)$ uniformly in
$(x,y)$, and the loss of the challenger changes by at most the same amount.
By the triangle inequality,
\[
\|h_t-h_{t-1}\|_\infty
\le
\frac{4}{n_{\min}+1}.
\]
Combining the displays proves the claim.
\end{proof}

\paragraph{From average advantage to instantaneous advantage.}
Fix a candidate $(v,c)$ and suppose it is compared over a time interval during
which no structural update affecting the relevant routing occurs. Let
\[
\bar\Delta_t^{v,c}
:=\frac{1}{t-s^{v,c}}
\sum_{r=s^{v,c}+1}^{t}\Delta_r^{v,c},
\qquad
\bar\delta_t^{v,c}
:=\frac{1}{t-s^{v,c}}
\sum_{r=s^{v,c}+1}^{t}\delta_r^{v,c}.
\]
The next lemma converts a lower bound on the running average into a lower bound
on the final instantaneous advantage.

\begin{lemma}[Average-to-last transfer under bounded drift]
\label{lem:avg_to_last_bounded_drift_weak}
Fix a candidate $(v,c)$ and a time $t\ge s^{v,c}+1$. Assume that for all
$r\in\{s^{v,c}+1,\dots,t-1\}$,
$|\delta_{r+1}^{v,c}-\delta_r^{v,c}|\le \rho$. Then
\[
\delta_t^{v,c}
\ge
\bar\delta_t^{v,c}
-
\rho\,\frac{t-s^{v,c}-1}{2}.
\]
In particular, if
$\bar\delta_t^{v,c}\ge \rho(t-s^{v,c}-1)/2+\varepsilon$, then
$\delta_t^{v,c}\ge\varepsilon$.
\end{lemma}

\begin{proof}
Let $T=t-s^{v,c}$ and write $\delta_r=\delta_r^{v,c}$. For
$k=0,\ldots,T-1$, the drift condition gives
\[
\delta_{t-k}\le \delta_t+k\rho,
\qquad\text{or equivalently}\qquad
\delta_t\ge \delta_{t-k}-k\rho.
\]
Averaging over $k=0,\ldots,T-1$ yields
\[
\delta_t
\ge
\frac1T\sum_{k=0}^{T-1}\delta_{t-k}
-
\frac{\rho}{T}\sum_{k=0}^{T-1}k
=
\bar\delta_t^{v,c}-\rho\frac{T-1}{2}.
\]
\end{proof}

\paragraph{Corrected weak-test stopping rule.}
Let $(L_t,U_t)$ be a confidence sequence for the running average
$\bar\delta_t^{v,c}$, so that on its coverage event $L_t\le\bar\delta_t^{v,c}$
for all $t\ge s^{v,c}+1$. Combining
Lemmas~\ref{lem:bounded_drift_delta_weak}--\ref{lem:avg_to_last_bounded_drift_weak},
on the coverage event we have
\[
\delta_t^{v,c}
\ge
L_t-\frac{\rho(t-s^{v,c}-1)}{2}.
\]
Therefore, to ensure an instantaneous margin $\varepsilon\ge0$ at the commit
time, one may use the corrected weak-test stopping time
\[
\tau_{\mathrm{w,bd}}^{v,c}
:=
\inf\left\{
 t\ge s^{v,c}+1:
 L_t>\varepsilon+\frac{\rho(t-s^{v,c}-1)}{2}
\right\},
\qquad
\rho:=\frac{4}{n_{\min}+1}.
\]

\subsubsection{Adaptive bounded-drift correction using leaf counts}
\label{adaptive_bound}
The minimum leaf sample size requirement in
Assumption~\ref{ass:nmin_drift_weak} can be removed by controlling the one-step
variation of the instantaneous prequential advantage using the actual leaf
counts involved in the shadow incumbent and shadow challenger predictions.

\begin{lemma}[Adaptive one-step drift via leaf counts]
\label{lem:adaptive_bounded_drift_delta_weak}
Assume i.i.d. data and Brier loss $\ell(\hat y,y)=(\hat y-y)^2$ with predictions
in $[0,1]$. Fix a candidate $(v,c)$ and a time $t\ge s^{v,c}+1$ such that no
structural update affects this comparison between $t-1$ and $t$. Let
$n^{\mathrm{inc}}_t$ and $n^{\mathrm{chal}}_t$ denote the numbers of samples in
the shadow incumbent leaf and the relevant shadow challenger leaf immediately
before their update at time $t$. Then
\[
\bigl|\delta_{t+1}^{v,c}-\delta_t^{v,c}\bigr|
\le
\rho_t
:=
\frac{2}{n^{\mathrm{inc}}_t+1}
+
\frac{2}{n^{\mathrm{chal}}_t+1}.
\]
\end{lemma}

\begin{proof}
Use the fresh-draw functions $h_{t-1}$ and $h_t$ from the proof of
Lemma~\ref{lem:bounded_drift_delta_weak}. The empirical mean of the affected
incumbent leaf changes by at most $1/(n^{\mathrm{inc}}_t+1)$, and the empirical
mean of the affected challenger leaf changes by at most
$1/(n^{\mathrm{chal}}_t+1)$. The Brier loss is $2$-Lipschitz in its prediction
argument on $[0,1]$. Therefore
\[
\|h_t-h_{t-1}\|_\infty
\le
\frac{2}{n^{\mathrm{inc}}_t+1}
+
\frac{2}{n^{\mathrm{chal}}_t+1}.
\]
As before,
$|\delta_{t+1}^{v,c}-\delta_t^{v,c}|\le\|h_t-h_{t-1}\|_\infty$, proving the
claim.
\end{proof}

\begin{lemma}[Average-to-last transfer under adaptive drift]
\label{lem:avg_to_last_adaptive_drift_weak}
Let $(a_i)_{i=1}^n$ be a real sequence such that
$|a_{i+1}-a_i|\le\rho_i$ for $i=1,\ldots,n-1$. Then
\[
a_n
\ge
\frac1n\sum_{i=1}^n a_i
-
\frac1n\sum_{j=1}^{n-1}j\rho_j.
\]
Consequently, with $a_i=\delta_{s^{v,c}+i}^{v,c}$ and
$n=t-s^{v,c}$,
\[
\delta_t^{v,c}
\ge
\bar\delta_t^{v,c}
-
\frac{1}{t-s^{v,c}}
\sum_{j=1}^{t-s^{v,c}-1}
 j\,\rho_{s^{v,c}+j}.
\]
\end{lemma}

\begin{proof}
For $i=1,\ldots,n$, repeated use of the drift bounds gives
\[
a_i
\le
 a_n+
\sum_{j=i}^{n-1}\rho_j.
\]
Averaging over $i$ yields
\[
\frac1n\sum_{i=1}^n a_i
\le
 a_n+
\frac1n\sum_{i=1}^n\sum_{j=i}^{n-1}\rho_j.
\]
For a fixed $j$, the term $\rho_j$ appears exactly for
$i=1,\ldots,j$, hence exactly $j$ times. Thus
\[
\frac1n\sum_{i=1}^n a_i
\le
 a_n+
\frac1n\sum_{j=1}^{n-1}j\rho_j,
\]
which rearranges to the desired inequality.
\end{proof}

\paragraph{Adaptive weak-test stopping rule.}
Combining Lemmas~\ref{lem:adaptive_bounded_drift_delta_weak} and
\ref{lem:avg_to_last_adaptive_drift_weak}, on the CS coverage event it is enough
to use
\[
\tau_{\mathrm{w,ad}}^{v,c}
:=
\inf\left\{
 t\ge s^{v,c}+1:
 L_t>
 \varepsilon+
 \frac{1}{t-s^{v,c}}
 \sum_{j=1}^{t-s^{v,c}-1}
 j\,\rho_{s^{v,c}+j}
\right\},
\]
where
\[
\rho_u
:=
\frac{2}{n^{\mathrm{inc}}_u+1}
+
\frac{2}{n^{\mathrm{chal}}_u+1}.
\]
This corrects the adaptive average-to-last penalty: the appropriate drift budget
is the weighted sum $(t-s^{v,c})^{-1}\sum_j j\rho_{s^{v,c}+j}$, not the
unweighted one-half sum.

\paragraph{Stationary risk certificate as a direct alternative.}
The weak-test corrections above are one route to certifying instantaneous
commit-time advantage. Under stationarity, a cleaner alternative is to certify
the current risk improvement directly.

\begin{lemma}[Stationary risk certificate]
\label{lem:stationary_risk_certificate_app}
Assume $(X_t,Y_t)$ are i.i.d. with law $P$ and $Y\in\{1,\ldots,K\}$. Fix a
candidate split $(v,c)$ at a possible commit time $t$, and let
$b_c(x)\in\{L,R\}$ denote the child selected by split $c$. Conditional on
$X\in R(v)$, define
\[
\theta^{v,c}_{u,k}
:=
\Pp\!\left(b_c(X)=u,\ Y=k\mid X\in R(v)\right),
\qquad u\in\{L,R\},\ k\in[K].
\]
Let $q_{v,t-1}$ be the current shadow-incumbent prediction and let
$q_{L,t-1},q_{R,t-1}$ be the current shadow-challenger child predictions. Suppose
$\Theta_t^{v,c}$ is a confidence set such that, on an event $\mathcal C_t^{v,c}$,
$\theta^{v,c}\in\Theta_t^{v,c}$. Define
\[
\underline G_t^{v,c}
:=
\inf_{\theta\in\Theta_t^{v,c}}
\sum_{u\in\{L,R\}}\sum_{k=1}^K
\theta_{u,k}
\Bigl(
\ell(q_{v,t-1},k)-\ell(q_{u,t-1},k)
\Bigr).
\]
On $\mathcal C_t^{v,c}$, if $\underline G_t^{v,c}\ge0$, then the split satisfies
the instantaneous commit certificate
\[
\E\!\left[
\ell(m^v_{t-1}(X),Y)-\ell(m^{v_c}_{t-1}(X),Y)
\mid \F_{t-1}
\right]
\ge0.
\]
If $\underline G_t^{v,c}\ge\varepsilon>0$, then the unconditional fresh-draw risk
improvement is at least
$\varepsilon\,\Pp(X\in R(v))$.
\end{lemma}

\begin{proof}
The incumbent and challenger agree outside $R(v)$. Hence the unconditional
fresh-draw risk gap equals
\[
\Pp(X\in R(v))
\sum_{u\in\{L,R\}}\sum_{k=1}^K
\theta^{v,c}_{u,k}
\Bigl(
\ell(q_{v,t-1},k)-\ell(q_{u,t-1},k)
\Bigr).
\]
On $\mathcal C_t^{v,c}$, the true vector $\theta^{v,c}$ belongs to
$\Theta_t^{v,c}$, so the conditional-on-$R(v)$ gap is at least
$\underline G_t^{v,c}$. If this lower bound is nonnegative, the unconditional gap
is nonnegative. If the lower bound is at least $\varepsilon$, the unconditional
gap is at least $\varepsilon\Pp(X\in R(v))$.
\end{proof}

\paragraph{Commit is risk-improving on $\mathcal E$.}
We now isolate the structural effect of committing a split. Let $\hat m_t^{sp}$ be
the intermediate deployed predictor obtained before the leaf-statistics update at
time $t$ but after committing the split, and let $\hat m_t$ be the deployed predictor after committing the split.

\begin{lemma}[Committing an advantageous split reduces conditional one-step risk on $\mathcal E$]
\label{lem:commit_improves_app}
Fix a commit time $t$ at which the algorithm commits a split $(v,c)$. Assume that
on the event $\mathcal E$, the committed split satisfies the instantaneous
fresh-draw advantage certificate
\[
\delta_t^{v,c}
:=
\mathbb E\!\left[
\ell\!\left(m^v_{t-1}(X),Y\right)
-
\ell\!\left(m^{v_c}_{t-1}(X),Y\right)
\,\middle|\,\mathcal F_{t-1}
\right]
\ge 0,
\tag{$\star$}
\]
where $(X,Y)\sim P$ is an independent fresh draw. Then, on $\mathcal E$,
\[
\mathbb E\!\left[
\ell(\hat m_t^{sp}(X),Y)
\,\middle|\,
\mathcal F_{t-1}
\right]
\le
\mathbb E\!\left[
\ell(\hat m_{t-1}(X),Y)
\,\middle|\,
\mathcal F_{t-1}
\right]
\quad\text{a.s.}
\]
Consequently,
\[
\mathbb E\!\left[\ell(\hat m_t^{sp}(X),Y)\right]
\le
\mathbb E\!\left[\ell(\hat m_{t-1}(X),Y)\right].
\]
If $(\star)$ holds with margin $\varepsilon>0$, then the unconditional structural
risk decrease is at least $\varepsilon$ when $\varepsilon$ is an unconditional
certificate, and at least $\varepsilon\Pp(X\in R(v))$ when $\varepsilon$ is a
certificate conditional on $X\in R(v)$.
\end{lemma}

\begin{proof}
The structural update replaces the incumbent prediction by the challenger
prediction only on the affected region $R(v)$; outside $R(v)$, the two deployed
predictors agree. Therefore,
\begin{align*}
&\mathbb E\!\left[
\ell(\hat m_{t-1}(X),Y)-\ell(\hat m_t^{sp}(X),Y)
\,\middle|\,
\mathcal F_{t-1}
\right]
\\
&\qquad=
\mathbb E\!\left[
\ell\!\left(m^v_{t-1}(X),Y\right)
-
\ell\!\left(m^{v_c}_{t-1}(X),Y\right)
\,\middle|\,
\mathcal F_{t-1}
\right]
=
\delta_t^{v,c}.
\end{align*}
The equality uses the fact that $m^v_{t-1}$ and $m^{v_c}_{t-1}$ agree outside
$R(v)$ by construction, and that both predictors are $\mathcal F_{t-1}$-measurable.
By $(\star)$, the right-hand side is nonnegative on $\mathcal E$, which proves the
conditional risk inequality. Taking expectations gives the unconditional claim.
The margin statements follow from the same display. If the certificate is stated
conditional on $X\in R(v)$, the unconditional gap is the conditional gap multiplied
by $\Pp(X\in R(v))$.
\end{proof}

\subsubsection{Step 4: Conclude the theorem}

\begin{proof}[Proof of Theorem~\ref{thm:monotone_model}]
Fix $t\ge 1$.
There are two cases.

\smallskip
\noindent
\textbf{Case 1: no split is committed at time $t$.}
Then $\hat m_t$ differs from $\hat m_{t-1}$ only through leaf-statistic updates.
Lemma~\ref{lem:leaf_update_monotone_app} yields
\[
\mathbb E\!\left[\ell(\hat m_t(X),Y)\right]
\le
\mathbb E\!\left[\ell(\hat m_{t-1}(X),Y)\right].
\]

\smallskip
\noindent
\textbf{Case 2: a split is committed at time $t$.}
Write $\hat m^{sp}_{t}$ for the intermediate model after structural commit but
before the leaf-statistic updates.
By Lemma~\ref{lem:commit_improves_app} on $\mathcal{E}$,
\[
\mathbb E\!\left[\ell(\hat m^{sp}_{t}(X),Y)\right]
\le
\mathbb E\!\left[\ell(\hat m_{t-1}(X),Y)\right],
\]
and by Lemma~\ref{lem:leaf_update_monotone_app},
\[
\mathbb E\!\left[\ell(\hat m_t(X),Y)\right]
\le
\mathbb E\!\left[\ell(\hat m^{sp}_{t}(X),Y)\right].
\]
Combining gives
\[
\mathbb E\!\left[\ell(\hat m_t(X),Y)\right]
\le
\mathbb E\!\left[\ell(\hat m_{t-1}(X),Y)\right].
\]

Since $t$ was arbitrary, monotonicity holds for all $t\ge 1$ and the
sequence $\bigl(\mathbb E[\ell(\hat m_t(X),Y)]\bigr)_{t\ge 0}$ is
nonincreasing.
\end{proof}

\newpage

\subsection{Proof of Theorem~\ref{thm:power}}
\label{app:proof_power_full_detailed}

We prove the two commitment guarantees in Theorem~\ref{thm:power}:
\begin{enumerate}
\item Under the \textbf{strong alternative}, the \textbf{betting/wealth} stopping time
$\tau^{v,c}$ is finite,
and admits a high-probability bound of order
$\tilde{\cO}(\log(1/\alpha^{v,c})/\Delta^2)$.
\item Under the \textbf{weak alternative}, the \textbf{CS-crossing} stopping time
$\tau_{\mathrm{w}}^{v,c}:=\inf\{t\ge s^{v,c}:L_t>0\}$ is finite with probability
at least $1-\alpha^{v,c}$, and admits the same rate when instantiated with the
(empirical Bernstein) confidence sequence \cite{choe2024comparing} used.
\end{enumerate}

\paragraph{Global notation.}
Fix a node--candidate pair $(v,c)$ and suppress superscripts $(v,c)$.
Let the test start time be $s:=s^{v,c}$.
For brevity write
\[
\Delta_t:=\Delta_t^{v,c}\in[-1,1],
\qquad
\delta_t:=\E[\Delta_t\mid\F_{t-1}],
\qquad t\ge s.
\]
We also set $n:=t-s+1$ for the \emph{effective sample size} at time $t\ge s$.

The wealth of the betting-based test defined in (\ref{eq:wealth_update_main}) using the betting fraction (\ref{eq:UP_beta}) can be expressed as a continuous mixture
over betting fractions,
\begin{equation}
\label{eq:UP_continuous_app}
W_t
=
\int_0^1
\prod_{i=s}^{t}
\bigl(1+\beta\,\Delta_i\bigr)
\,d\pi(\beta),
\end{equation}
where $\pi$ denotes the Jeffreys prior on $[0,1]$.

In practice, we have to resort to approximation. We use a finite discrete mixture that allows to maintain the theoretical guarantees.
\subsubsection{Part I: Wealth (betting) stopping time under the strong alternative}

We first prove the wealth-based guarantee for the actual discrete-mixture wealth
used by the algorithm.

\begin{assumption}[Persistent strong advantage for the original discrete mixture]
\label{ass:split_advantage_discrete_restate}
Fix $(v,c)$ and let $s=s^{v,c}$ be the first time at which this candidate's loss
difference is used as test evidence. There exist $\Delta\in(0,1]$ and
$n_0\in\mathbb N$ such that, with $t_0:=s+n_0$,
\[
\delta_t
=
\E[\Delta_t\mid\F_{t-1}]
\ge
\Delta
\qquad\forall t\ge t_0.
\]
Thus $n_0$ is the number of tested increments before the persistent advantage
begins. The betting grid contains a point
\[
\beta_\star\in\left[\frac{\Delta}{8},\frac{\Delta}{4}\right]
\]
with mixture weight $w_\star>0$.
\end{assumption}

\paragraph{Original discrete-mixture wealth and stopping rule.}
Let $\mathcal B=\{\beta_1,\ldots,\beta_K\}\subset[0,1)$ be the fixed grid of
betting fractions and let $w_k>0$ with $\sum_{k=1}^K w_k=1$. The wealth is
\begin{equation}
\label{eq:discrete_wealth_power}
W_t
=
\sum_{k=1}^K
w_k
\prod_{u=s}^{t}
\bigl(1+\beta_k\Delta_u\bigr),
\qquad
W_{s-1}=1.
\end{equation}
The betting stopping time is
\[
\tau:=\tau^{v,c}
:=
\inf\left\{t\ge s:
W_t\ge \frac1{\alpha^{v,c}}
\right\}.
\]

\paragraph{Goal.}
We show that for every $\eta\in(0,1)$, with probability at least $1-\eta$,
\[
\tau
\le
 t_0+N_\eta-1,
\]
where
\begin{equation}
\label{eq:N_eta_discrete_power}
N_\eta
:=
\left\lceil
\frac{32}{\Delta^2}
\left(
\log\frac1{\alpha^{v,c}}
+
\log\frac1{w_\star}
+
B_0
+
\log\frac1\eta
\right)
\right\rceil,
\qquad
B_0:=n_0\log\frac1{1-\beta_\star}.
\end{equation}

\subsubsection*{Step 0: Lower-bound the original mixture by one favorable grid point}

For every $t\ge s$, the mixture contains the component indexed by $\beta_\star$,
so deterministically
\begin{equation}
\label{eq:mixture_lower_one_component}
W_t
\ge
w_\star
\prod_{u=s}^{t}
\bigl(1+\beta_\star\Delta_u\bigr).
\end{equation}
Because $\Delta_u\in[-1,1]$ and $\beta_\star<1$, each factor is strictly positive:
\[
1+\beta_\star\Delta_u
\ge
1-\beta_\star
>
0.
\]
Thus we may take logarithms. At time $t=t_0+N-1$, split the log-product into the
pre-advantage block $u=s,\ldots,t_0-1$ and the post-advantage block
$u=t_0,\ldots,t_0+N-1$:
\begin{align}
\log W_{t_0+N-1}
&\ge
\log w_\star
+
\sum_{u=s}^{t_0-1}\log(1+\beta_\star\Delta_u)
+
\sum_{u=t_0}^{t_0+N-1}\log(1+\beta_\star\Delta_u).
\label{eq:logW_split_pre_post}
\end{align}

\subsubsection*{Step 1: Control the stale pre-advantage history}

No sign or mean condition is imposed before $t_0$. Nevertheless, boundedness
alone gives a deterministic lower bound. Since $\Delta_u\ge -1$,
\[
1+\beta_\star\Delta_u
\ge
1-\beta_\star,
\]
and hence
\[
\log(1+\beta_\star\Delta_u)
\ge
\log(1-\beta_\star)
=
-\log\frac1{1-\beta_\star}.
\]
There are $n_0=t_0-s$ pre-advantage tested increments, so
\begin{equation}
\label{eq:stale_history_bound}
\sum_{u=s}^{t_0-1}\log(1+\beta_\star\Delta_u)
\ge
-n_0\log\frac1{1-\beta_\star}
= -B_0.
\end{equation}
This is the only cost of using the original single-start wealth rather than a
restart mixture. We could, in fact, circumvent this cost entirely by utilizing the restart mixture defined below:

For late-emerging advantages, one may allocate capital across restart times.
For deterministic weights \(\pi_r\ge0\), \(\sum_{r=s}^{\infty}\pi_r\le1\), define
\[
\widetilde W_t
=
\sum_{r=s}^{t}\pi_r
\sum_{k=1}^K w_k\prod_{u=r}^{t}(1+\beta_k\Delta_u)
+
\sum_{r>t}\pi_r.
\]
Each started component is an e-process under the strong null, and unstarted
capital is held in cash, so \((\widetilde W_t)\) is a nonnegative
supermartingale. If the advantage starts at \(t_0\), the component \(r=t_0\)
pays a penalty \(\log(1/\pi_{t_0})\) instead of the stale-history term. This is a
valid optional variant but is a different algorithmic wealth process.

\subsubsection*{Step 2: Concentration of the post-advantage cumulative loss difference}

Fix a block length $N\ge1$ and define, for $u\ge t_0$,
\[
X_u:=\Delta_u-\delta_u.
\]
Then $(X_u,\F_u)$ is a martingale difference sequence:
\[
\E[X_u\mid\F_{u-1}]
=
\E[\Delta_u-\delta_u\mid\F_{u-1}]
=
\delta_u-\delta_u
=0.
\]
Moreover, since $\Delta_u\in[-1,1]$, its conditional mean $\delta_u$ also lies
in $[-1,1]$, and $X_u$ lies conditionally in an interval of length at most $2$.
By the conditional Hoeffding lemma, for every $\lambda>0$,
\[
\E\!\left[\exp(-\lambda X_u)\mid\F_{u-1}\right]
\le
\exp\left(\frac{\lambda^2}{2}\right).
\]
Iterating this inequality over $u=t_0,\ldots,t_0+N-1$ gives
\[
\E\!\left[
\exp\left(-\lambda\sum_{u=t_0}^{t_0+N-1}X_u\right)
\right]
\le
\exp\left(\frac{N\lambda^2}{2}\right).
\]
Therefore, by Chernoff's method, for any $a>0$,
\[
\Pp\!\left(
\sum_{u=t_0}^{t_0+N-1}X_u\le -a
\right)
\le
\exp\left(-\frac{a^2}{2N}\right).
\]
Taking $a=\sqrt{2N\log(1/\eta)}$, we obtain an event $\mathcal A_N$ satisfying
$\Pp(\mathcal A_N)\ge1-\eta$ on which
\begin{equation}
\label{eq:mart_diff_lower_discrete}
\sum_{u=t_0}^{t_0+N-1}X_u
\ge
-\sqrt{2N\log\frac1\eta}.
\end{equation}
On $\mathcal A_N$, using $\delta_u\ge\Delta$ for $u\ge t_0$,
\begin{align}
\sum_{u=t_0}^{t_0+N-1}\Delta_u
&=
\sum_{u=t_0}^{t_0+N-1}\delta_u
+
\sum_{u=t_0}^{t_0+N-1}X_u
\nonumber\\
&\ge
N\Delta
-
\sqrt{2N\log\frac1\eta}.
\label{eq:cum_delta_lower_discrete}
\end{align}
If
\begin{equation}
\label{eq:N_concentration_requirement_discrete}
N\ge \frac{8}{\Delta^2}\log\frac1\eta,
\end{equation}
then
\[
\sqrt{2N\log\frac1\eta}\le \frac{N\Delta}{2},
\]
and therefore on $\mathcal A_N$,
\begin{equation}
\label{eq:cum_delta_half_discrete}
\sum_{u=t_0}^{t_0+N-1}\Delta_u
\ge
\frac{N\Delta}{2}.
\end{equation}

\subsubsection*{Step 3: Convert cumulative advantage into log-wealth growth}

We use the elementary inequality
\[
\log(1+x)\ge x-x^2,
\qquad x\in[-1/2,1/2].
\]
Indeed, for $g(x)=\log(1+x)-x+x^2$, one has $g(0)=0$ and
$g'(x)=x(1+2x)/(1+x)$, so $g$ is minimized at zero on $[-1/2,1/2]$.
Since $\beta_\star\le\Delta/4\le1/4$ and $\Delta_u\in[-1,1]$, the product
$x=\beta_\star\Delta_u$ lies in $[-1/4,1/4]$. Hence
\[
\log(1+\beta_\star\Delta_u)
\ge
\beta_\star\Delta_u-
\beta_\star^2\Delta_u^2
\ge
\beta_\star\Delta_u-
\beta_\star^2.
\]
Summing over the post-advantage block and using
\eqref{eq:cum_delta_half_discrete}, on $\mathcal A_N$ we get
\begin{align}
\sum_{u=t_0}^{t_0+N-1}\log(1+\beta_\star\Delta_u)
&\ge
\beta_\star
\sum_{u=t_0}^{t_0+N-1}\Delta_u
-
N\beta_\star^2
\nonumber\\
&\ge
N\left(\frac{\beta_\star\Delta}{2}-\beta_\star^2\right).
\label{eq:post_log_growth_first}
\end{align}
Because $\beta_\star\le\Delta/4$,
\[
\beta_\star^2\le \frac{\beta_\star\Delta}{4},
\]
so
\[
\frac{\beta_\star\Delta}{2}-\beta_\star^2
\ge
\frac{\beta_\star\Delta}{4}.
\]
Because $\beta_\star\ge\Delta/8$,
\[
\frac{\beta_\star\Delta}{4}
\ge
\frac{\Delta^2}{32}.
\]
Therefore, on $\mathcal A_N$ and under
\eqref{eq:N_concentration_requirement_discrete},
\begin{equation}
\label{eq:post_log_growth_discrete}
\sum_{u=t_0}^{t_0+N-1}\log(1+\beta_\star\Delta_u)
\ge
\frac{N\Delta^2}{32}.
\end{equation}

\subsubsection*{Step 4: Solve for threshold crossing}

Combining \eqref{eq:logW_split_pre_post}, \eqref{eq:stale_history_bound}, and
\eqref{eq:post_log_growth_discrete}, on $\mathcal A_N$ we have
\begin{equation}
\label{eq:logW_final_discrete_lower}
\log W_{t_0+N-1}
\ge
\log w_\star
-
B_0
+
\frac{N\Delta^2}{32}.
\end{equation}
Thus $W_{t_0+N-1}\ge1/\alpha^{v,c}$ is guaranteed whenever
\[
\log w_\star
-
B_0
+
\frac{N\Delta^2}{32}
\ge
\log\frac1{\alpha^{v,c}},
\]
or equivalently whenever
\begin{equation}
\label{eq:N_threshold_discrete}
\frac{N\Delta^2}{32}
\ge
\log\frac1{\alpha^{v,c}}
+
\log\frac1{w_\star}
+
B_0.
\end{equation}
The choice $N=N_\eta$ in \eqref{eq:N_eta_discrete_power} implies
\eqref{eq:N_threshold_discrete}. It also implies the concentration requirement
\eqref{eq:N_concentration_requirement_discrete}, because the bracket in
\eqref{eq:N_eta_discrete_power} contains $\log(1/\eta)$ and $32\ge8$.
Therefore,
\[
\Pp\!\left(
W_{t_0+N_\eta-1}\ge\frac1{\alpha^{v,c}}
\right)
\ge
1-\eta.
\]
By definition of $\tau$,
\[
\Pp\!\left(
\tau\le t_0+N_\eta-1
\right)
\ge
1-\eta.
\]

\paragraph{Almost sure finiteness.}
For every $\eta\in(0,1)$, there is a deterministic finite $N_\eta$ such that
\[
\Pp\!\left(\tau>t_0+N_\eta-1\right)\le\eta.
\]
Since $\{\tau=\infty\}\subseteq\{\tau>t_0+N_\eta-1\}$, we have
$\Pp(\tau=\infty)\le\eta$ for every $\eta\in(0,1)$. Letting $\eta\downarrow0$
gives $\Pp(\tau=\infty)=0$, so $\tau<\infty$ almost surely.

\paragraph{Stale-history penalty.}
Since $\beta_\star\le\Delta/4\le1/4$,
\[
B_0=n_0\log\frac1{1-\beta_\star}
\le
\frac{n_0\beta_\star}{1-\beta_\star}
\le
\frac{4}{3}n_0\beta_\star
\le
\frac{n_0\Delta}{3}.
\]
Thus the original single-start wealth pays an explicit stale-history penalty. Validity holds for any finite grid;
the rate above additionally requires that the grid contain a betting fraction of
order $\Delta$, namely $\beta_\star\in[\Delta/8,\Delta/4]$.

\subsubsection{Part II: CS-crossing stopping time under the weak alternative}

We now prove the CS part under the weak alternative.

\paragraph{Running averages.}
For $t\ge s$ define
\[
\bar\Delta_t
:=
\frac{1}{t-s+1}\sum_{u=s}^t \Delta_u,
\qquad
\bar\delta_t
:=
\frac{1}{t-s+1}\sum_{u=s}^t \delta_u.
\]
The weak null and weak alternative are:
\[
H_{\mathrm{w},0}:\ \forall t\ge s,\ \bar\delta_t\le 0,
\qquad
H_{\mathrm{w},1}(\Delta):\ \exists t_0\ge s\ \text{s.t.}\ \forall t\ge t_0,\ \bar\delta_t\ge \Delta.
\]

\paragraph{Confidence sequence and stopping rule.}
Let $(L_t,U_t)_{t\ge s}$ be a $(1-\alpha^{v,c})$ CS for $\bar\delta_t$:
\begin{equation}
\label{eq:CS_cover_full}
\Pp\Bigl(\forall t\ge s:\ L_t\le \bar\delta_t\le U_t\Bigr)\ge 1-\alpha^{v,c}.
\end{equation}
Define the CS-crossing stopping time
\[
\tau_{\mathrm{w}}:=\tau_{\mathrm{w}}^{v,c}:=\inf\{t\ge s:\ L_t>0\}.
\]

\subsubsection*{Step 1: Generic crossing lemma}

\begin{lemma}[Crossing once width is below the advantage]
\label{lem:cs_cross_full}
Let $\mathcal G:=\{\forall t\ge s:\ L_t\le \bar\delta_t\le U_t\}$.
Assume the weak alternative $H_{\mathrm{w},1}(\Delta)$ holds for some $\Delta>0$ and $t_0\ge s$.
Define the CS width $w_t:=U_t-L_t$.
Then on $\mathcal G$, for all $t\ge t_0$,
\[
L_t \ge \Delta - w_t.
\]
In particular, on $\mathcal G$,
\[
\tau^{\mathrm{s}}
\le
\inf\{t\ge t_0:\ w_t<\Delta\}.
\]
Consequently, if $w_t\to 0$ on $\mathcal G$, then $\tau^{\mathrm{s}}<\infty$ on $\mathcal G$.
\end{lemma}

\begin{proof}
Work on $\mathcal G$.
For $t\ge t_0$, the weak alternative gives $\bar\delta_t\ge \Delta$.
Since $\bar\delta_t\le U_t$ and $L_t\le \bar\delta_t$ on $\mathcal G$,
\[
L_t
=
U_t-(U_t-L_t)
\ge
\bar\delta_t-w_t
\ge
\Delta-w_t.
\]
If $w_t<\Delta$, then $L_t>0$ and therefore $\tau^{\mathrm{s}}\le t$.
The remaining statements follow immediately.
\end{proof}

Using \eqref{eq:CS_cover_full}, we immediately obtain:
\[
\Pp(\tau^{\mathrm{s}}<\infty)\ge \Pp(\mathcal G)\ge 1-\alpha^{v,c}.
\]

\subsubsection*{Step 2: Our explicit empirical Bernstein CS}

We now \emph{directly define} the empirical Bernstein CS used in our paper from \cite{choe2024comparing},
in the exact form needed for the proof.

\paragraph{Empirical Bernstein radius (time-uniform).}
Assume $\Delta_t\in[-1,1]$ a.s.
Define the predictable empirical variance proxy
\[
\hat V_t
:=
\frac{1}{t-s+1}\sum_{u=s}^{t}
\bigl(\Delta_u-\bar\Delta_{u-1}\bigr)^2,
\qquad
\bar\Delta_{s-1}:=0.
\]
Let
\[
\ell_n(\alpha)
:=
\log\frac{1}{\alpha}+\log\log\bigl(e n\bigr),
\qquad n\ge 1,
\]
and define the (two-sided) empirical Bernstein half-width
\begin{equation}
\label{eq:empbern_radius_full}
r_n(\alpha)
:=
\sqrt{\frac{2\hat V_t\,\ell_n(\alpha)}{n}}
+
\frac{3\,\ell_n(\alpha)}{n},
\qquad n=t-s+1.
\end{equation}
(Any equivalent time-uniform empirical Bernstein CS from our construction can be
used here; constants are universal.)

Define
\[
L_t:=\bar\Delta_t-r_{t-s+1}(\alpha^{v,c}),
\qquad
U_t:=\bar\Delta_t+r_{t-s+1}(\alpha^{v,c}),
\qquad t\ge s.
\]

\begin{theorem}[Time-uniform empirical Bernstein CS (our instantiation)]
\label{thm:empbern_cs_full}
For any $\alpha\in(0,1)$,
\[
\Pp\Bigl(\forall t\ge s:\ \bar\delta_t\in[L_t,U_t]\Bigr)\ge 1-\alpha.
\]
\end{theorem}

\paragraph{Width bound.}
By construction,
\[
w_t
=
U_t-L_t
=
2r_{t-s+1}(\alpha^{v,c}).
\]

\subsubsection*{Step 3: Solving explicitly for the crossing time}

By Lemma~\ref{lem:cs_cross_full}, on $\mathcal G$ it suffices to have
\[
w_t<\Delta
\quad\Longleftrightarrow\quad
2r_{n}(\alpha^{v,c})<\Delta,
\qquad n=t-s+1.
\]
Using \eqref{eq:empbern_radius_full}, a sufficient condition is
\[
2\sqrt{\frac{2\hat V_t\,\ell_n(\alpha^{v,c})}{n}}
+
\frac{6\,\ell_n(\alpha^{v,c})}{n}
<\Delta.
\]
Since $\Delta_t\in[-1,1]$, we have the crude bound $\hat V_t\le 4$ a.s.,
so a deterministic sufficient condition is
\begin{equation}
\label{eq:det_suff_cond_cs}
2\sqrt{\frac{8\,\ell_n(\alpha^{v,c})}{n}}
+
\frac{6\,\ell_n(\alpha^{v,c})}{n}
<\Delta.
\end{equation}
We now solve \eqref{eq:det_suff_cond_cs}.

\paragraph{Controlling the square-root term.}
Require
\[
2\sqrt{\frac{8\,\ell_n}{n}}
\le
\frac{\Delta}{2}
\quad\Longleftrightarrow\quad
\sqrt{\frac{8\,\ell_n}{n}}
\le
\frac{\Delta}{4}
\quad\Longleftrightarrow\quad
\frac{8\,\ell_n}{n}
\le
\frac{\Delta^2}{16}
\quad\Longleftrightarrow\quad
n
\ge
\frac{128}{\Delta^2}\,\ell_n.
\]

\paragraph{Controlling the linear term.}
Also require
\[
\frac{6\,\ell_n}{n}
\le
\frac{\Delta}{2}
\quad\Longleftrightarrow\quad
n\ge \frac{12}{\Delta}\,\ell_n.
\]
Since $\Delta\le 1$, the $\Delta^{-2}$ requirement dominates asymptotically, so it is enough to enforce
\[
n\ \ge\ \frac{C}{\Delta^2}\,\ell_n(\alpha^{v,c})
\]
for a universal constant $C$.

\paragraph{Self-consistency and the $\log\log$ term.}
Recall $\ell_n(\alpha)=\log(1/\alpha)+\log\log(en)$.
If we pick
\[
n
\ge
\frac{C}{\Delta^2}\left(\log\frac{1}{\alpha^{v,c}}+\log\log\frac{e}{\alpha^{v,c}}\right),
\]
then $\log\log(en)$ is at most a constant times $\log\log(e/\alpha^{v,c})$ (up to additive constants),
so the inequality is self-consistent. Therefore, there exists a universal constant $C'>0$ such that
on $\mathcal G$,
\[
\tau^{\mathrm{s}}
\le
t_0-1
+
\frac{C'}{\Delta^2}\left(\log\frac{1}{\alpha^{v,c}}+\log\log\frac{e}{\alpha^{v,c}}\right),
\]
i.e.,
\[
\tau^{\mathrm{s}}-t_0
=
\tilde{\cO}\!\left(\frac{\log(1/\alpha^{v,c})}{\Delta^2}\right)
\quad\text{on }\mathcal G.
\]

\subsubsection{Final conclusion}

\begin{proof}[Proof of Theorem~\ref{thm:power}]
\emph{(Wealth/betting, strong alternative).}
Assume Assumption~\ref{ass:split_advantage_discrete_restate}. Part~I shows that
for any $\eta\in(0,1)$, with probability at least $1-\eta$,
\[
\tau^{v,c}
\le
s+n_0+N_\eta-1,
\]
where
\[
N_\eta
=
\left\lceil
\frac{32}{\Delta^2}
\left(
\log\frac1{\alpha^{v,c}}
+
\log\frac1{w_\star}
+
n_0\log\frac1{1-\beta_\star}
+
\log\frac1\eta
\right)
\right\rceil .
\]
In particular $\tau^{v,c}<\infty$ almost surely.

\emph{(CS-crossing, weak alternative).}
Assume the weak alternative: there exist $\Delta>0$ and $t_0\ge s$ such that
$\bar\delta_t\ge\Delta$ for all $t\ge t_0$.
Let $(L_t,U_t)$ be the empirical Bernstein CS defined above.
By Theorem~\ref{thm:empbern_cs_full}, the coverage event $\mathcal G$ holds with
probability at least $1-\alpha^{v,c}$.
On $\mathcal G$, Lemma~\ref{lem:cs_cross_full} implies that the stopping time
$\tau_{\mathrm{w}}^{v,c}=\inf\{t\ge s:L_t>0\}$ is finite, and solving the width
condition using \eqref{eq:empbern_radius_full} yields
\[
\tau_{\mathrm{w}}^{v,c}-t_0
=
\tilde{\cO}\!\left(\frac{\log(1/\alpha^{v,c})}{\Delta^2}\right)
\quad\text{with probability at least }1-\alpha^{v,c}.
\]
This completes the proof of Theorem~\ref{thm:power}.
\end{proof}

\section{Implementation Details and Parameters}
\label{app:parameters}

This appendix provides implementation-specific details underlying
Algorithm~\ref{alg:avht}.

\subsection{Discrete-Mixture Betting Strategy}
\label{app:discrete_mixture}

In theory, the betting-based test can be expressed as a continuous mixture
over betting fractions,
\begin{equation}
\label{eq:UP_continuous_app}
W_t
=
\int_0^1
\prod_{i=s}^{t}
\bigl(1+\beta\,\Delta_i\bigr)
\,d\pi(\beta),
\end{equation}
where $\pi$ denotes the Jeffreys prior on $[0,1]$.
In practice, we have to resort to approximation. We use a finite discrete mixture,
which eliminates approximation error entirely, and allow to maintain the theoretical guarantees

\paragraph{Discrete-mixture wealth.}
Fix a finite grid of betting fractions
\[
\mathcal B := \{\beta_1,\dots,\beta_K\} \subset [0,1],
\]
for example a geometric or uniform grid such as
$\mathcal B=\{0,0.01,0.02,\dots,0.99\}$.
Let $\{w_k\}_{k=1}^K$ be nonnegative prior weights satisfying
$\sum_{k=1}^K w_k=1$.

The betting wealth is defined as
\begin{equation}
\label{eq:UP_discrete_app}
W_t
:=
\sum_{k=1}^K
w_k
\prod_{i=s}^{t}
\bigl(1+\beta_k\,\Delta_i\bigr).
\end{equation}
For each $k$, the process
$W_t^{(k)}:=\prod_{i=s}^{t}(1+\beta_k\Delta_i)$
is a nonnegative supermartingale under the null hypothesis
$\mathbb E[\Delta_t\mid\mathcal F_{t-1}]\le 0$.
Therefore, $W_t$ is itself a nonnegative supermartingale.

\paragraph{Validity and efficiency.}
This construction has three key consequences:
(i) \emph{Exact anytime-validity}: a finite sum of supermartingales remains a
supermartingale, so Ville's inequality applies directly;
(ii) \emph{Approximation of the continuous Universal Portfolio}: standard
online portfolio results imply that a sufficiently fine grid approximates the
continuous mixture up to an additive $O(\log K)$ term in log-wealth;
and (iii) \emph{computational efficiency}: updating $K$grid points
incurs only $O(K)$ cost per step and is typically faster than Monte Carlo
integration. We use $K=100$.

Unless stated otherwise, all theoretical guarantees in the main text apply
exactly to the discrete-mixture wealth~\eqref{eq:UP_discrete_app}.

\subsection{Significance-Level Allocation}
\label{app:alpha-allocation}

We allocate significance levels only to tests that are actually
instantiated. Let \(\mathcal T\) denote the random countable set of
node--candidate tests ever created by the algorithm. It is enough for
global validity that, pathwise,
\[
    \sum_{(v,c)\in\mathcal T}\alpha^{v,c}\le \alpha .
\]

One can use the following online alpha-spending rule. For \(d\ge0\), \(m\ge1\),
and \(q\ge1\), define
\[
    a_d=\frac{6}{\pi^2(d+1)^2},\qquad
    b_m=\frac{6}{\pi^2m^2},\qquad
    g_q=\frac{6}{\pi^2q^2}.
\]
Thus \(\sum_{d\ge0}a_d=\sum_{m\ge1}b_m=\sum_{q\ge1}g_q=1\).

For a node \(v\), let \(d(v)\) be its depth and let \(r(v)\) be its
creation rank among nodes of the same depth. Candidate splits may be
generated adaptively. Let \(q\) index the candidate-generation calls made
at node \(v\). On the \(q\)-th call, after removing duplicate or previously
tested candidates, suppose \(B_{v,q}\) new candidates remain. If
\(B_{v,q}=0\), no level is spent. Otherwise, every new candidate \(c\) in
that batch receives
\[
    \boxed{
    \alpha^{v,c}
    =
    \alpha_{\rm tree}\,
    a_{d(v)}\,b_{r(v)}\,g_q\,
    \frac{1}{B_{v,q}}
    } .
\]
For a single tree, take \(\alpha_{\rm tree}=\alpha\). For a forest with
\(M\) trees, take \(\alpha_{\rm tree}=\alpha/M\).

The level \(\alpha^{v,c}\) is fixed before the first observation used to
update the corresponding test statistic. Hence it is predictable at the
test start time.

To verify summability, fix any realized run of the algorithm. Let
\(\mathcal V\) be the set of nodes ever created and let
\(\mathcal C_{v,q}\) be the set of new candidates instantiated at node
\(v\) on proposal call \(q\). Then
\[
\begin{aligned}
\sum_{(v,c)\in\mathcal T}\alpha^{v,c}
&=
\alpha_{\rm tree}
\sum_{v\in\mathcal V}
a_{d(v)}b_{r(v)}
\sum_{q\ge1}g_q
\sum_{c\in\mathcal C_{v,q}}\frac{1}{B_{v,q}}
\\
&\le
\alpha_{\rm tree}
\sum_{d\ge0}a_d
\sum_{m\ge1}b_m
\sum_{q\ge1}g_q
=
\alpha_{\rm tree}.
\end{aligned}
\]
Therefore
\[
    \sum_{(v,c)\in\mathcal T}\alpha^{v,c}\le \alpha_{\rm tree}
\]
pathwise.
\subsection{Candidate Split Generation}
\label{app:candidate_generation}

To avoid biasing the comparison, candidate splits are generated using exactly
the same mechanisms as in the baseline Hoeffding Tree.
Depending on the configuration, each node employs either a Gaussian splitter
or a histogram-based splitter, maintaining feature-wise sufficient statistics. By default, HT uses the gaussien splitter.
When queried, the node proposes a finite set of promising candidates based on
empirical impurity reduction. In our experiments, we use the default setting of proposing up to 10 candidates per time step; however, we automatically discard duplicates and ensure that previously proposed splits are not re-evaluated.

Overall, we use exactly the same configuration as the default Hoeffding Tree. Our method only differs from the baseline only in the statistical test used to commit splits.

\subsection{Empirical Bernstein Confidence Sequence}
\label{app:empbern}

For the confidence-sequence--based test, we use a time-uniform empirical
Bernstein confidence sequence as in \cite{choe2024comparing}.
Assume $\Delta_t\in[-1,1]$ almost surely and let $s$ denote the test start time.

\paragraph{Empirical Bernstein radius (time-uniform).}
Define the running average
\[
\bar\Delta_t
:=
\frac{1}{t-s+1}\sum_{u=s}^{t}\Delta_u,
\qquad t\ge s,
\]
and the predictable empirical variance proxy
\[
\hat V_t
:=
\frac{1}{t-s+1}
\sum_{u=s}^{t}
\bigl(\Delta_u-\bar\Delta_{u-1}\bigr)^2,
\qquad
\bar\Delta_{s-1}:=0.
\]

Let
\[
\ell_n(\alpha)
:=
\log\frac{1}{\alpha}
+
\log\log\!\bigl(e n\bigr),
\qquad n\ge 1,
\]
and define the empirical Bernstein half-width
\begin{equation}
\label{eq:empbern_radius_app}
r_n(\alpha)
:=
\sqrt{\frac{2\,\hat V_t\,\ell_n(\alpha)}{n}}
+
\frac{3\,\ell_n(\alpha)}{n},
\qquad n=t-s+1.
\end{equation}

The resulting confidence sequence is
\[
L_t := \bar\Delta_t - r_{t-s+1}(\alpha^{v,c}),
\qquad
U_t := \bar\Delta_t + r_{t-s+1}(\alpha^{v,c}),
\qquad t\ge s.
\]

\newpage
\section{Additional experiments} \label{app:add_exp}

This section presents supplementary experiments that further characterize the empirical behavior of our anytime-valid decision tree framework. We first isolate and compare the proposed betting-based and confidence-sequence–based split-certification mechanisms in Appendix \ref{app:cs_comp}, and then report additional comparisons against established streaming ensemble baselines in Appendix \ref{app:baselines_2}.

\subsection{Comparison between \AVDTbet{} and \AVDTcs{}} \label{app:cs_comp}

We compare \AVDTbet{} and \AVDTcs{} across a diverse collection of real-world
regression and classification data streams. The two methods differ only in the underlying anytime-valid testing mechanism used to certify splits betting-based
tests for \AVDTbet{} versus confidence-sequence–based tests for \AVDTcs{}, allowing
for a clean empirical comparison of their practical behavior. Figure~\ref{app:fig:all_datasets_merged} reports full prequential and distribution performance
trajectories on twelve benchmark datasets.

Overall, \AVDTcs{} underperforms \AVDTbet{}, though it remains
competitive with and typically outperforms—the ARF baseline.

\begin{figure*}[!ht]
  \centering
  \setlength{\tabcolsep}{2pt}
  \renewcommand{\arraystretch}{0.9}

  \begin{tabular}{ccc}

  \begin{subfigure}[b]{0.31\textwidth}
    \centering
    \includegraphics[width=1\textwidth]{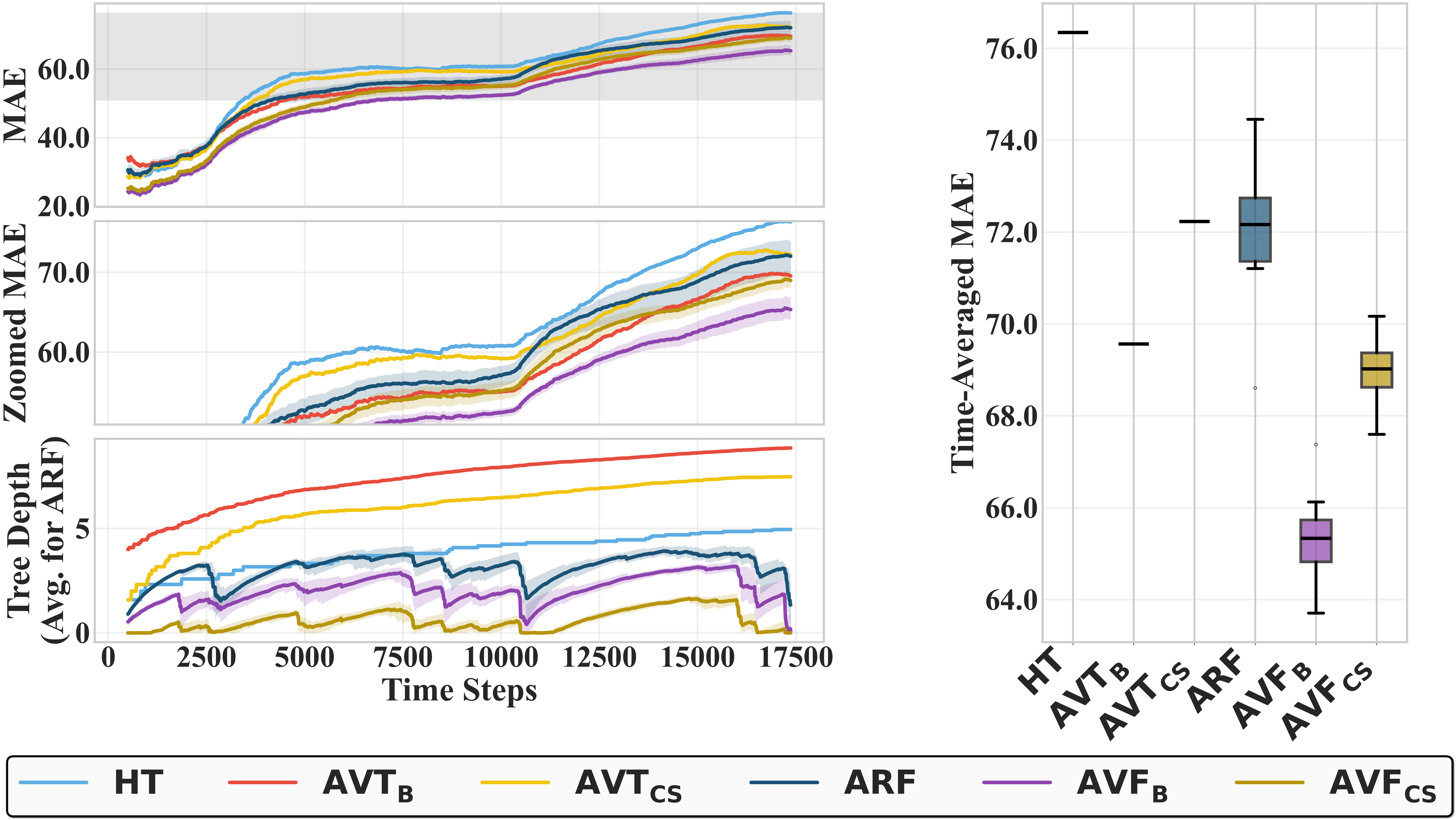}
    \caption{bike}
    \label{fig:reg_bike}
  \end{subfigure} &
  \begin{subfigure}[b]{0.31\textwidth}
    \centering
    \includegraphics[width=\textwidth]{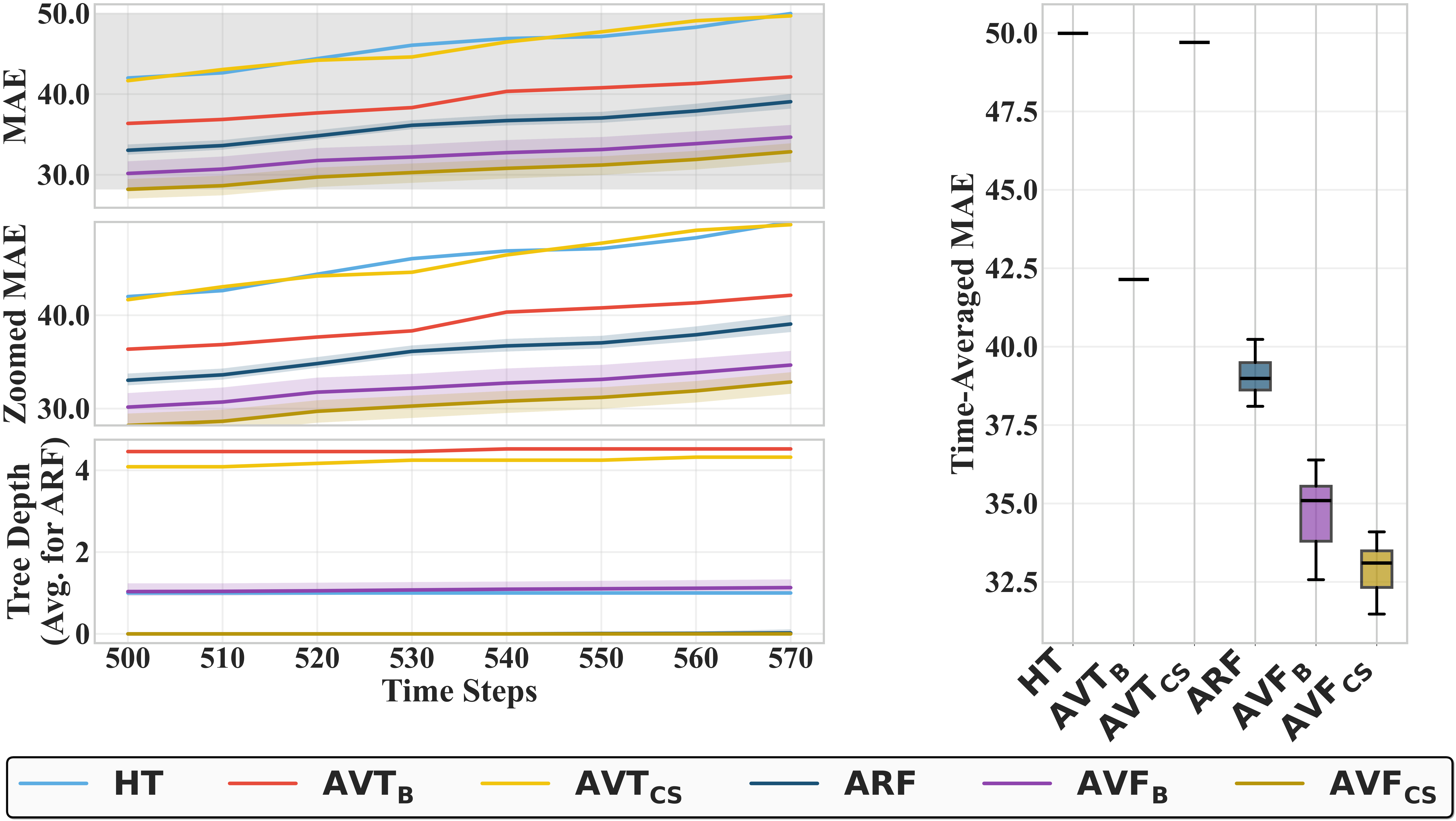}
    \caption{chick}
    \label{fig:reg_chick}
  \end{subfigure} &
  \begin{subfigure}[b]{0.31\textwidth}
    \centering
    \includegraphics[width=\textwidth]{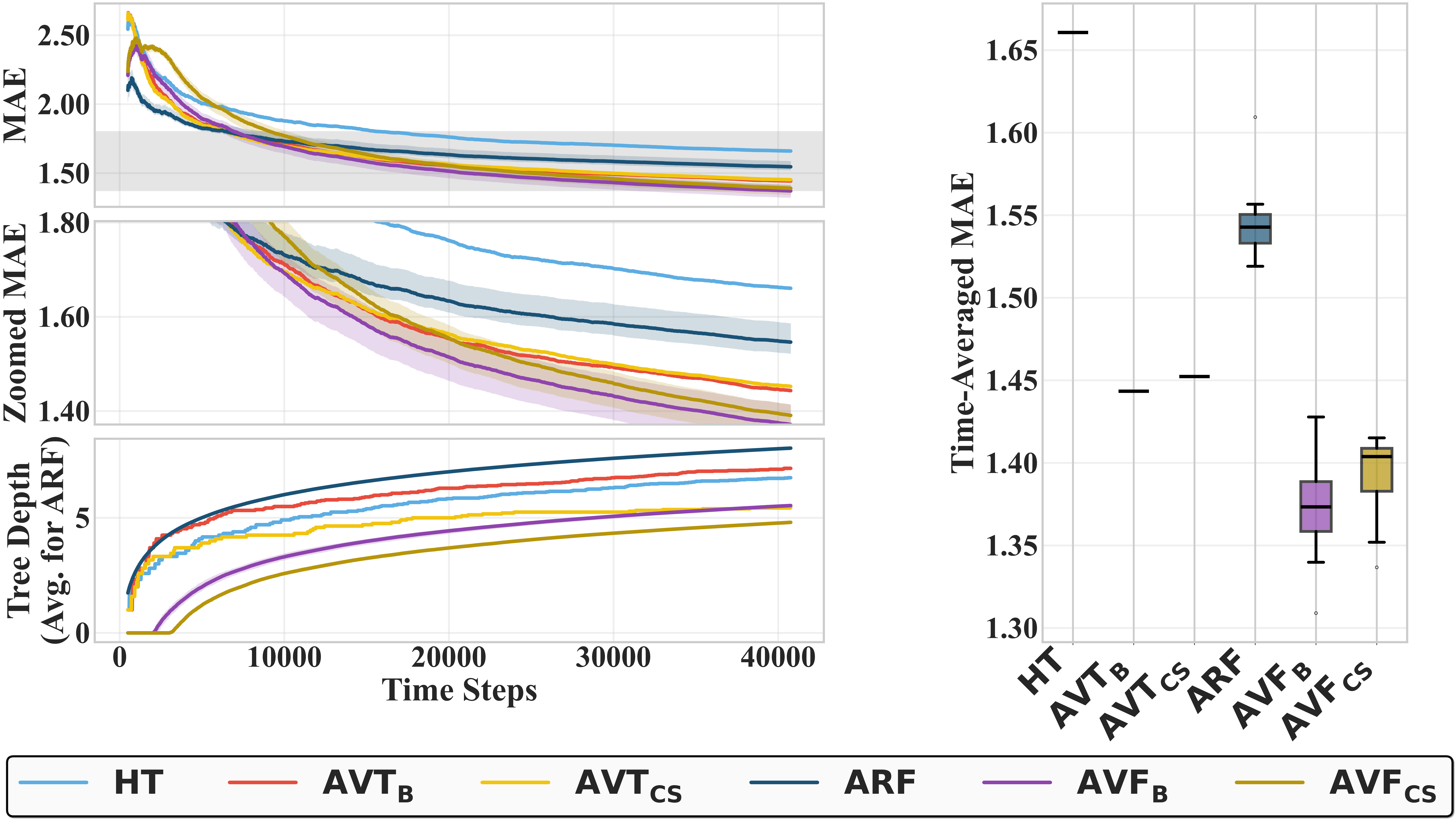}
    \caption{fried}
    \label{fig:reg_fried}
  \end{subfigure} \\[4pt]

  \begin{subfigure}[b]{0.31\textwidth}
    \centering
    \includegraphics[width=\textwidth]{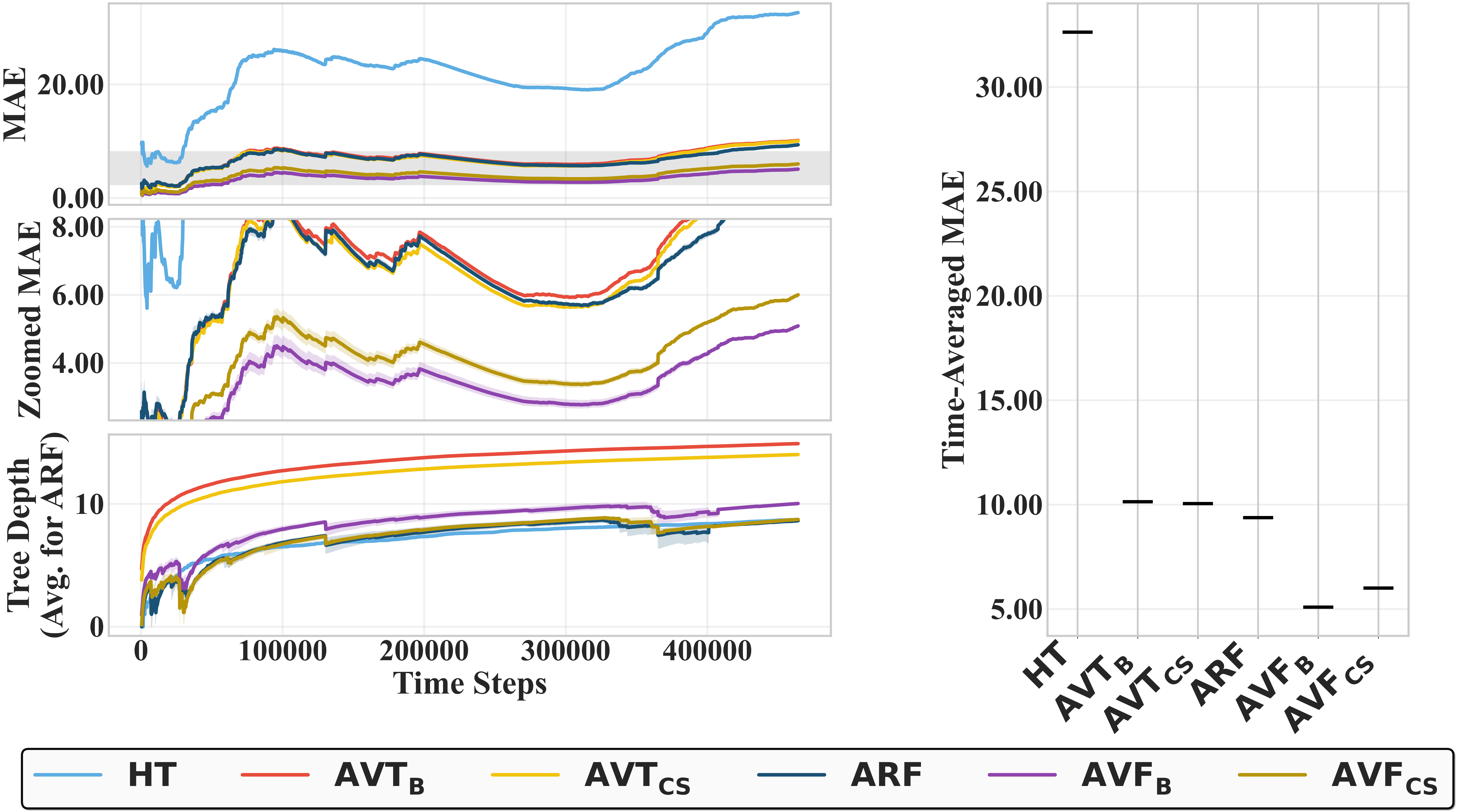}
    \caption{nzenergy}
    \label{fig:reg_nzenergy}
  \end{subfigure} &
  \begin{subfigure}[b]{0.31\textwidth}
    \centering
    \includegraphics[width=\textwidth]{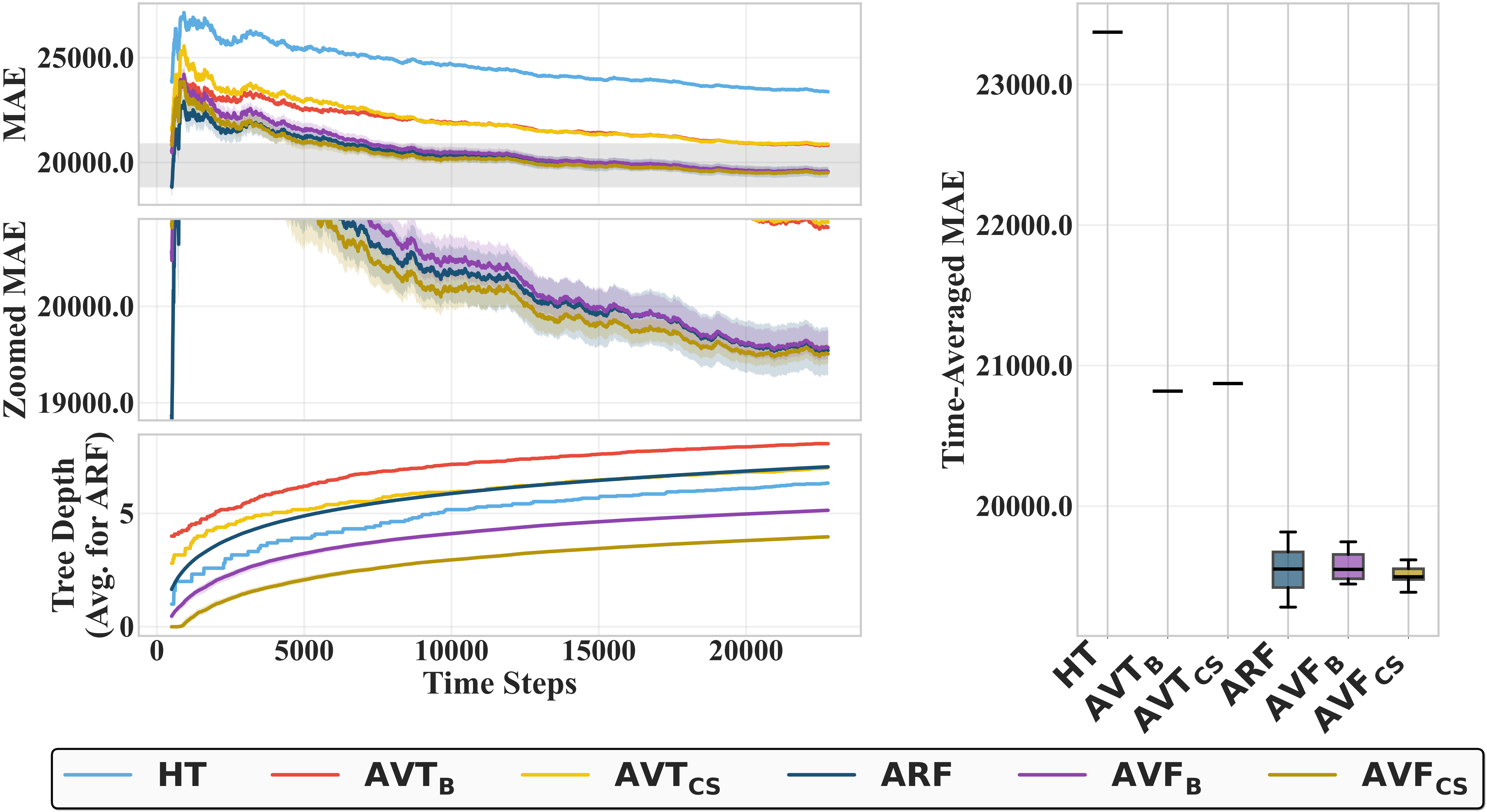}
    \caption{house}
    \label{fig:reg_house}
  \end{subfigure} &
  \begin{subfigure}[b]{0.31\textwidth}
    \centering
    \includegraphics[width=\textwidth]{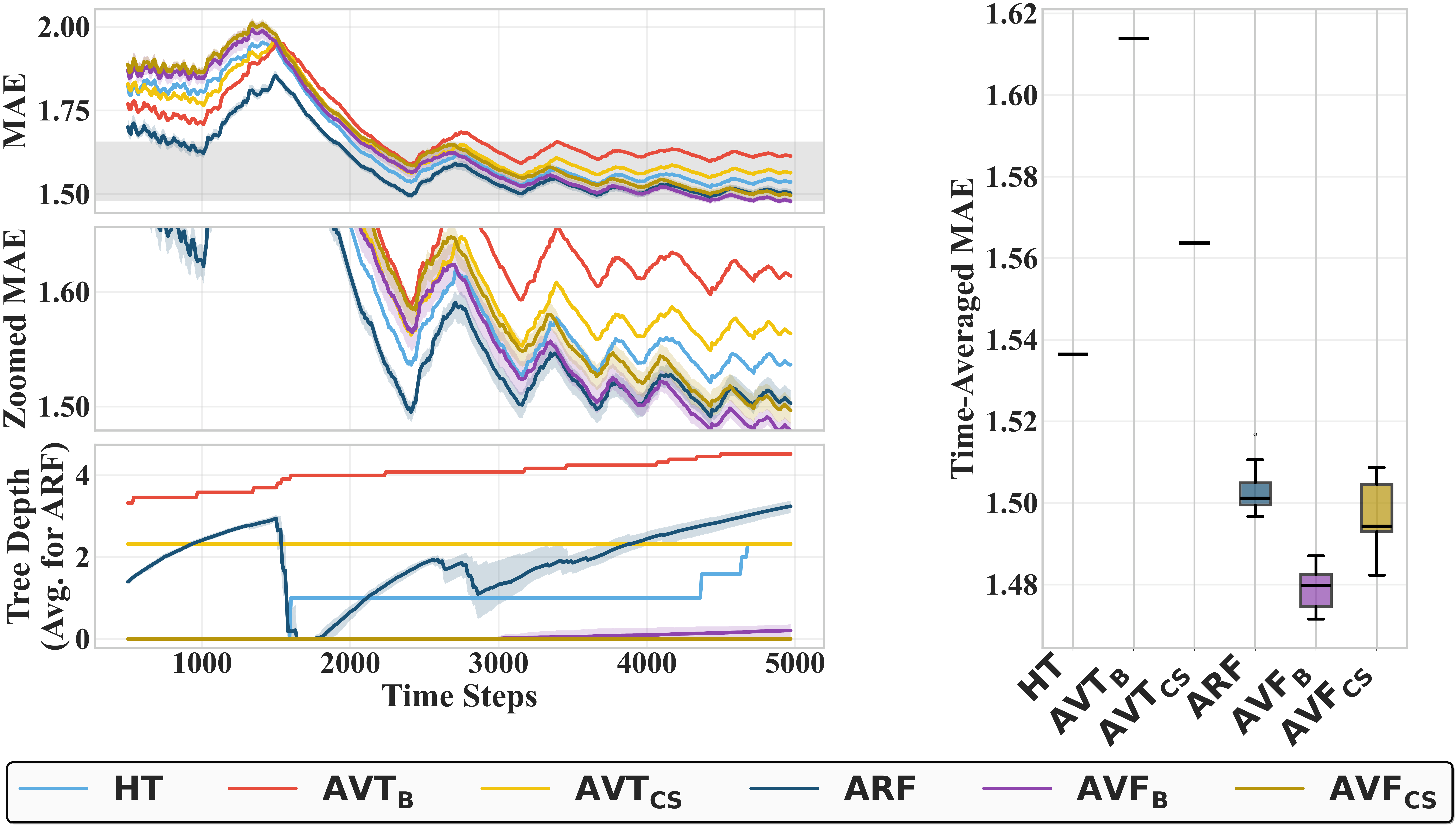}
    \caption{abalone}
    \label{fig:reg_abalone}
  \end{subfigure} \\[6pt]

  \begin{subfigure}[b]{0.31\textwidth}
    \centering
    \includegraphics[width=\textwidth]{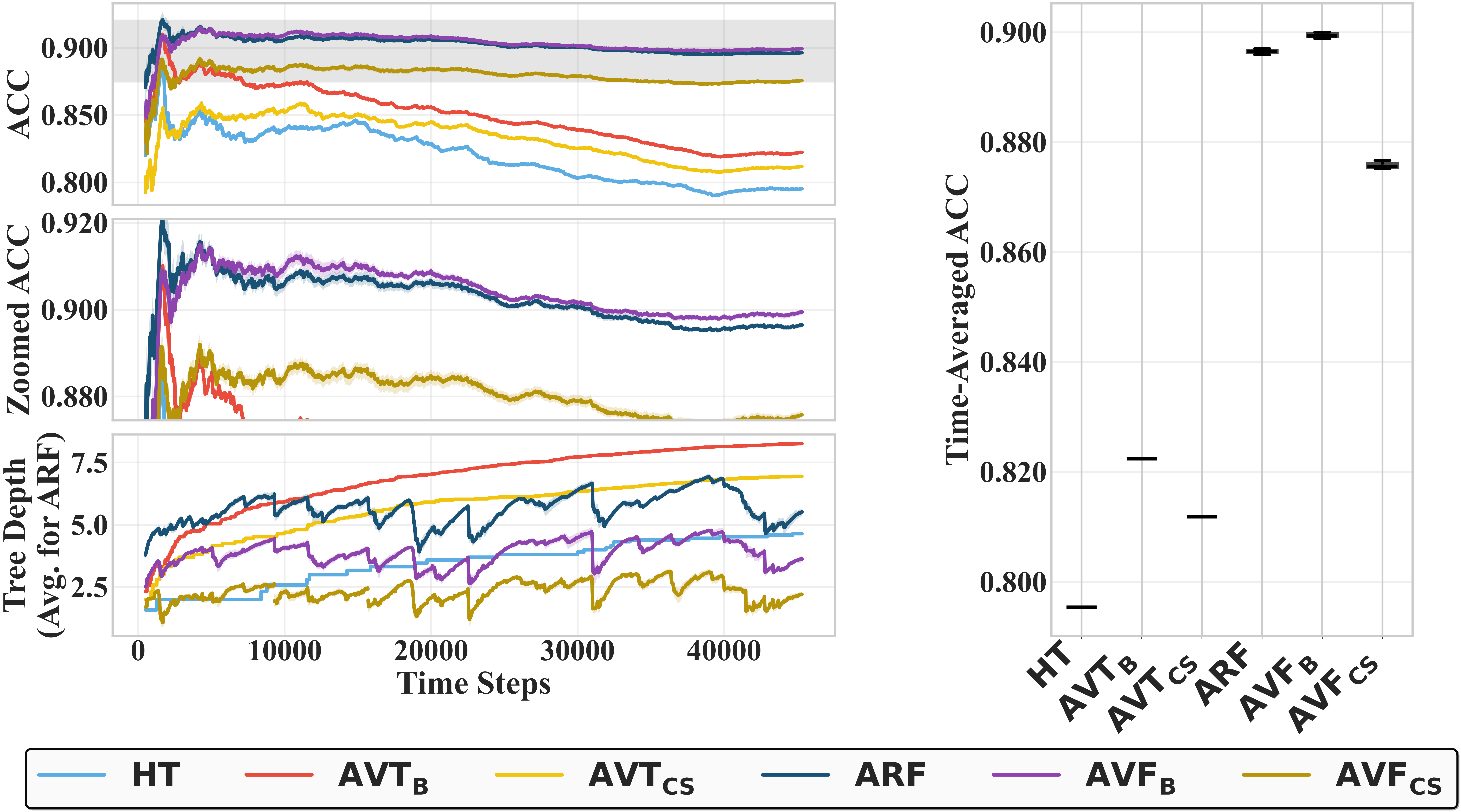}
    \caption{elec2}
    \label{fig:cls_elec2}
  \end{subfigure} &
  \begin{subfigure}[b]{0.31\textwidth}
    \centering
    \includegraphics[width=\textwidth]{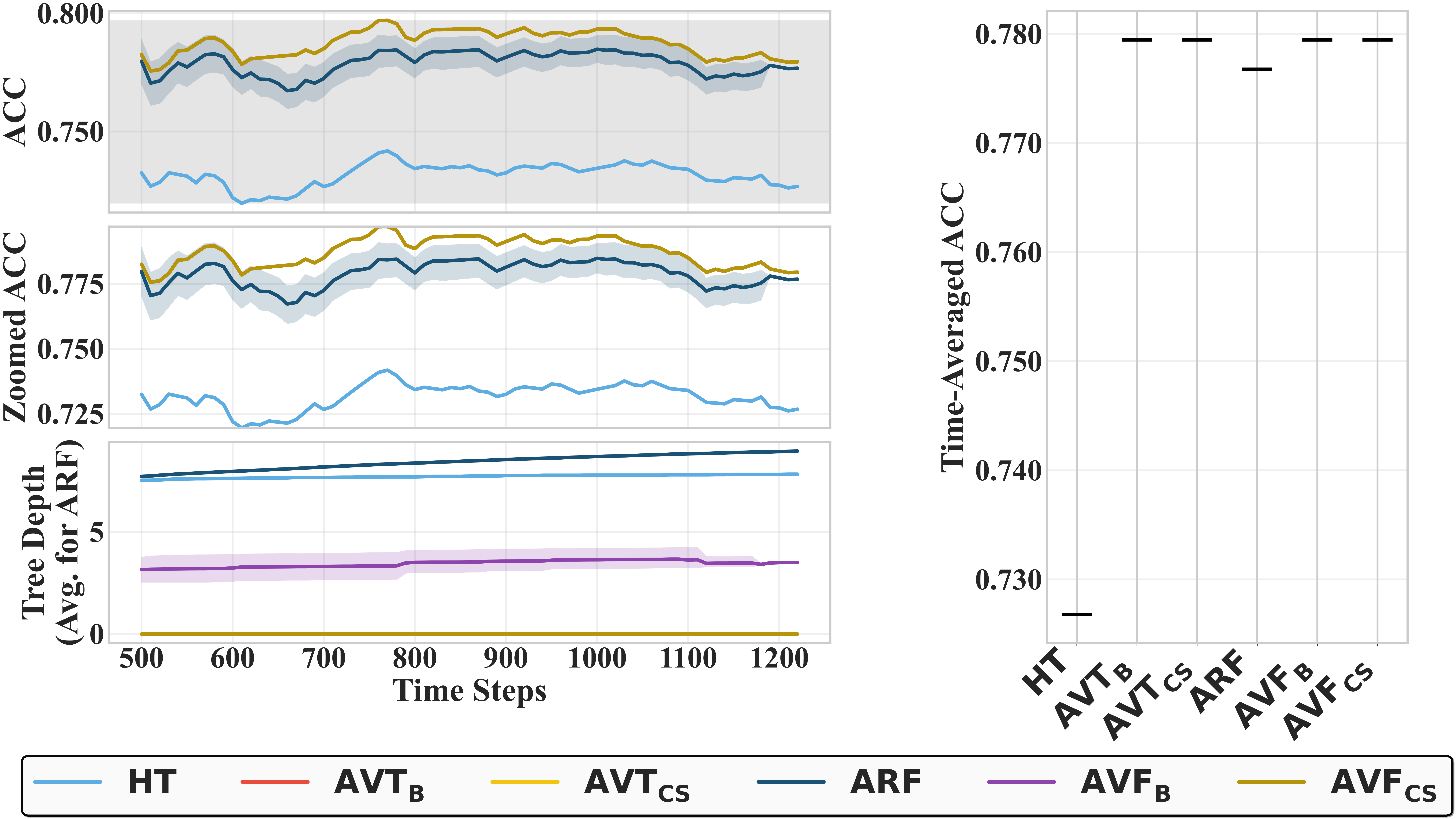}
    \caption{airlines}
    \label{fig:cls_airlines}
  \end{subfigure} &
  \begin{subfigure}[b]{0.31\textwidth}
    \centering
    \includegraphics[width=\textwidth]{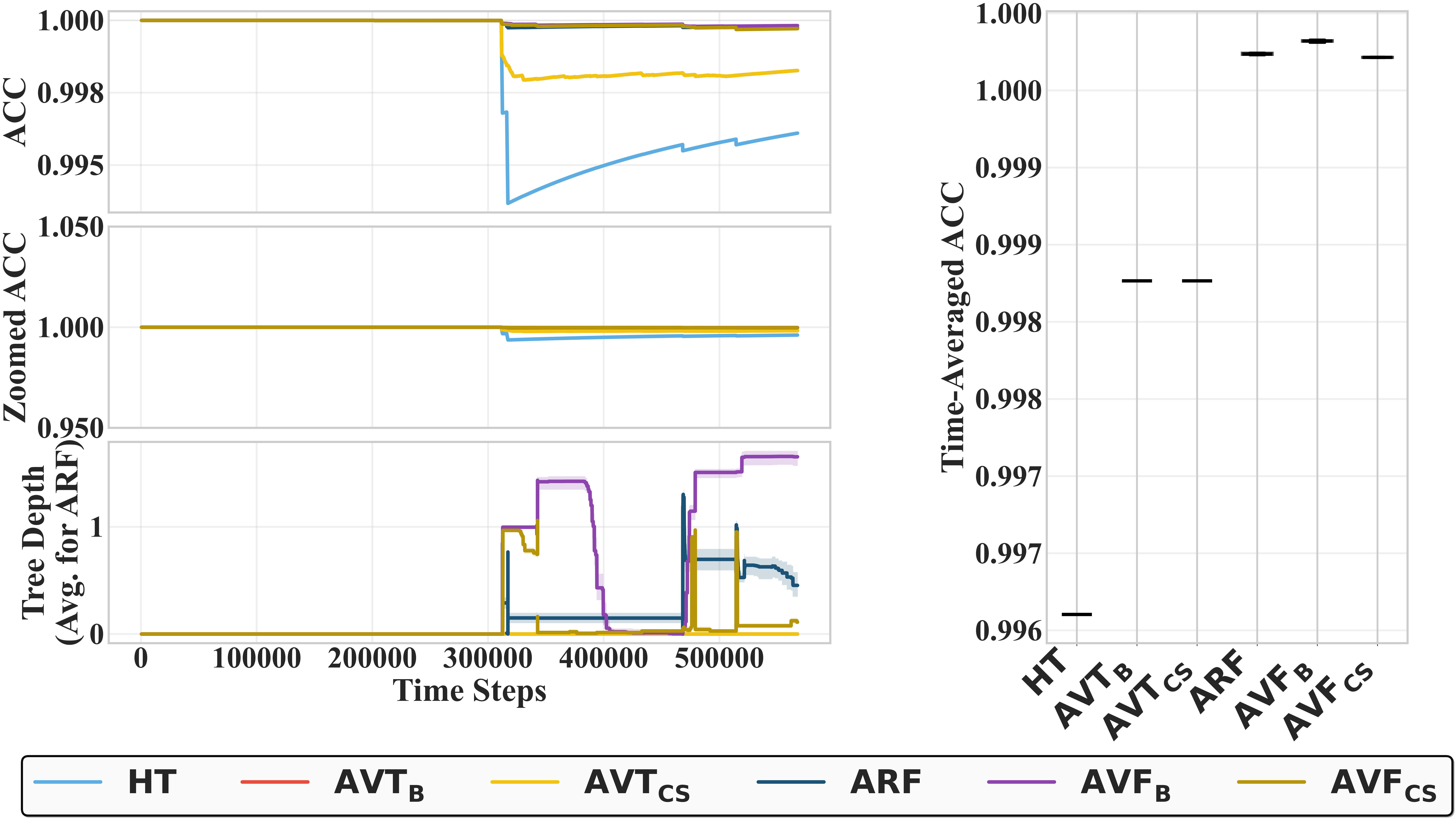}
    \caption{http-KDD99}
    \label{fig:cls_http}
  \end{subfigure} \\[4pt]

  \begin{subfigure}[b]{0.31\textwidth}
    \centering
    \includegraphics[width=\textwidth]{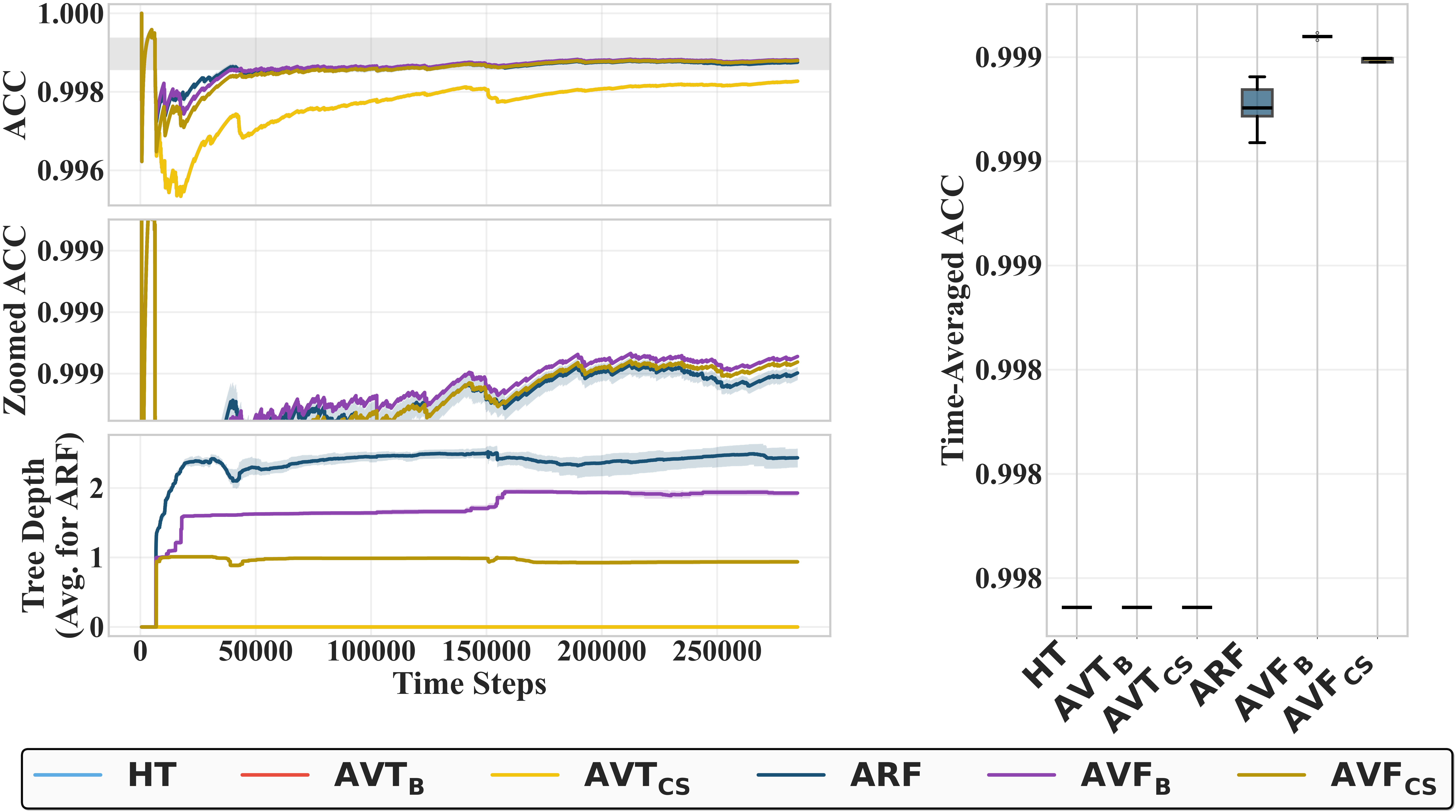}
    \caption{creditcard}
    \label{fig:cls_creditcard}
  \end{subfigure} &
  \begin{subfigure}[b]{0.31\textwidth}
    \centering
    \includegraphics[width=\textwidth]{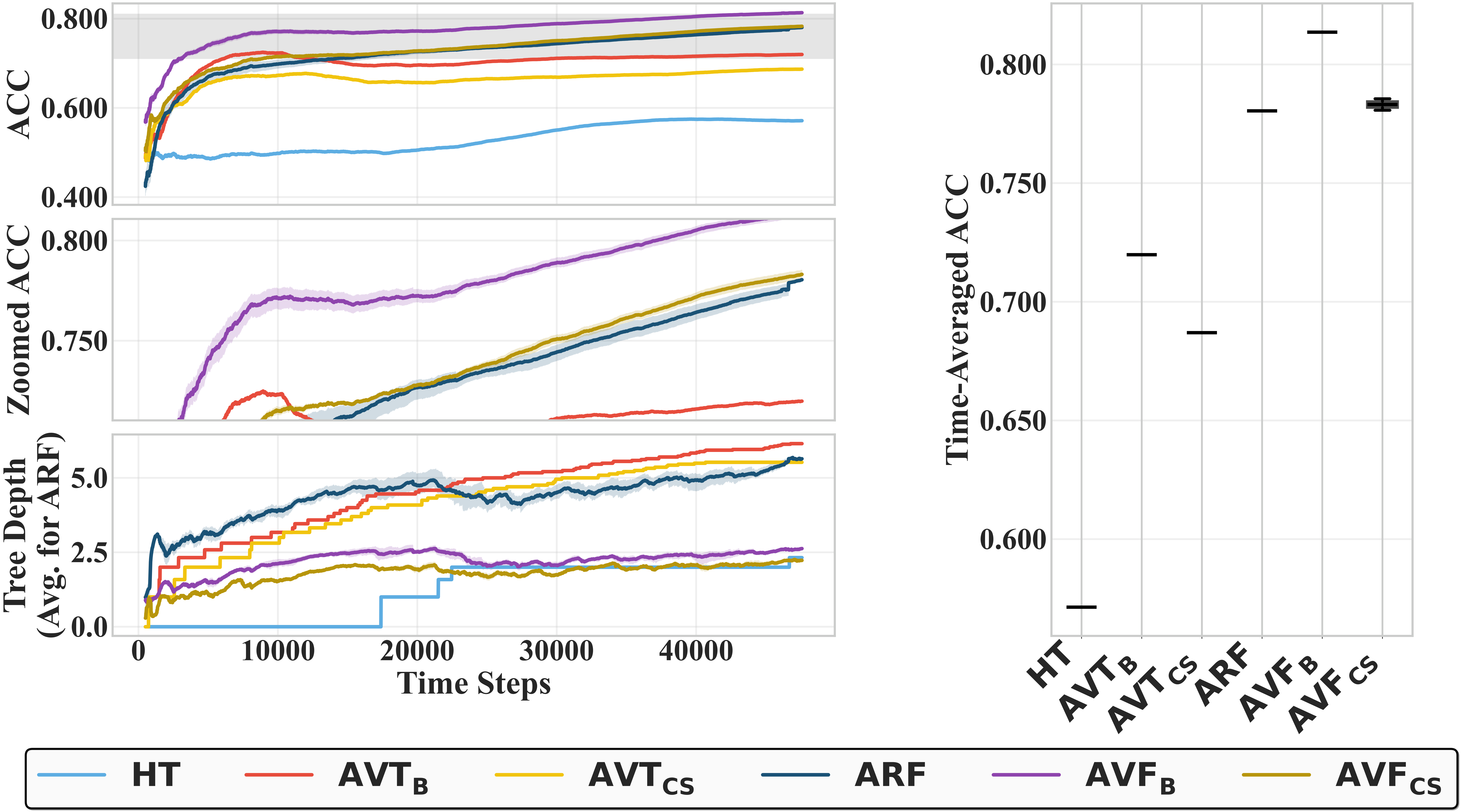}
    \caption{rbfm100k}
    \label{fig:cls_rbfm100k}
  \end{subfigure} &
  \begin{subfigure}[b]{0.31\textwidth}
    \centering
    \includegraphics[width=\textwidth]{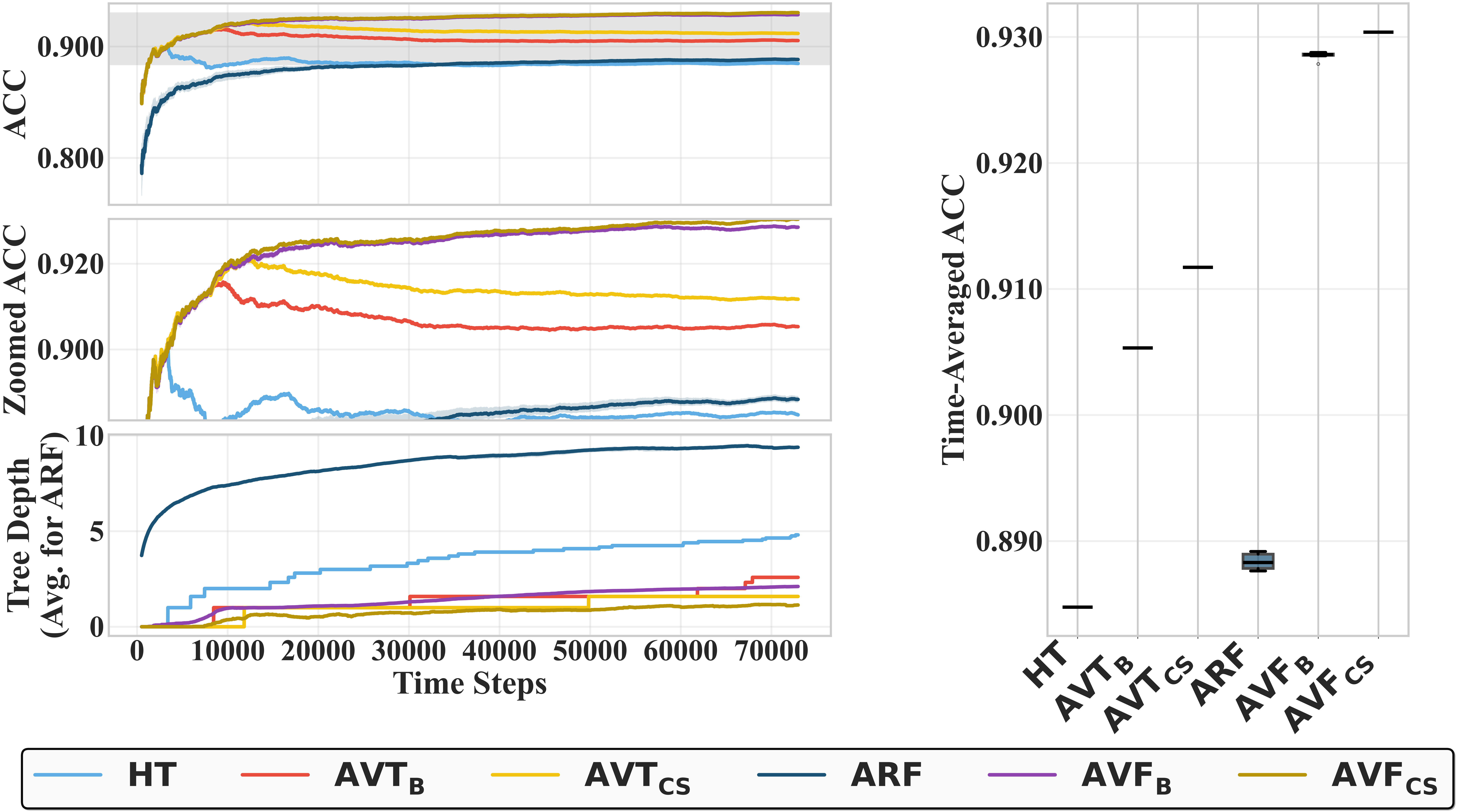}
    \caption{hyper100k}
    \label{fig:cls_hyper100k}
  \end{subfigure}

  \end{tabular}

  \caption{Performance comparison between \AVDTbet{} and \AVDTcs{} across
regression datasets \textbf{(a)--(f)} and classification datasets
\textbf{(g)--(l)}.}
  \label{app:fig:all_datasets_merged}
\end{figure*}

\newpage

\subsection{Other baselines} \label{app:baselines_2}
In this section, we report additional comparisons against existing streaming
ensemble baselines. We use the packages 
\texttt{river} \cite{montiel2021river} and \texttt{CapyMOA} \cite{gomes2025capymoaefficientmachinelearning} for the baselines. For classification, we use \texttt{SOKNL} \cite{sun2022soknl}, \texttt{OnlineSmoothBoost} \cite{chen2012online},
\texttt{StreamingRandomPatches} \cite{gomes2019streaming}, \texttt{OzaBoost} \cite{oza2001online}, and \texttt{ARF}.  
For regression, we compare only against \texttt{SOKNL} and \texttt{ARF}, as the previous
methods are not currently available for regression tasks.

As noted by \cite{read2025supervised}, the streaming literature has predominantly focused on classification, with regression receiving comparatively limited attention until recently. We also exclude \texttt{SGBT} \cite{gunasekara2025gradient} from
our evaluation due to unresolved implementation issues that prevent reliable
use; details of the issues are provided in Appendix~\ref{app:sgbt_issue}.

\begin{figure*}[!ht]
  \centering
  \setlength{\tabcolsep}{2pt}
  \renewcommand{\arraystretch}{0.9}

  \begin{tabular}{ccc}

  \begin{subfigure}[b]{0.31\textwidth}
    \centering
    \includegraphics[width=1\textwidth]{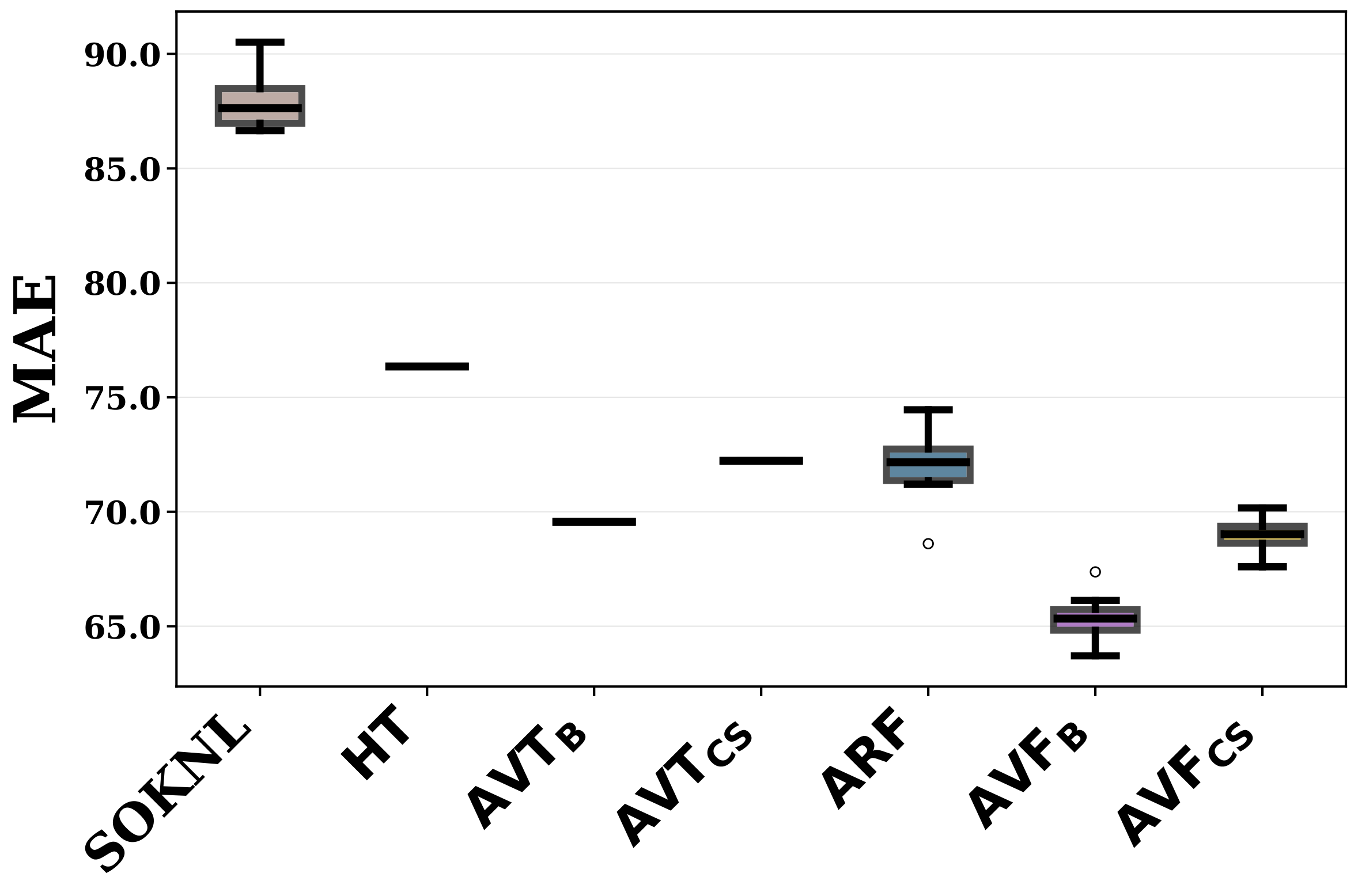}
    \caption{bike}
    \label{fig:reg_bike}
  \end{subfigure} &
  \begin{subfigure}[b]{0.31\textwidth}
    \centering
    \includegraphics[width=\textwidth]{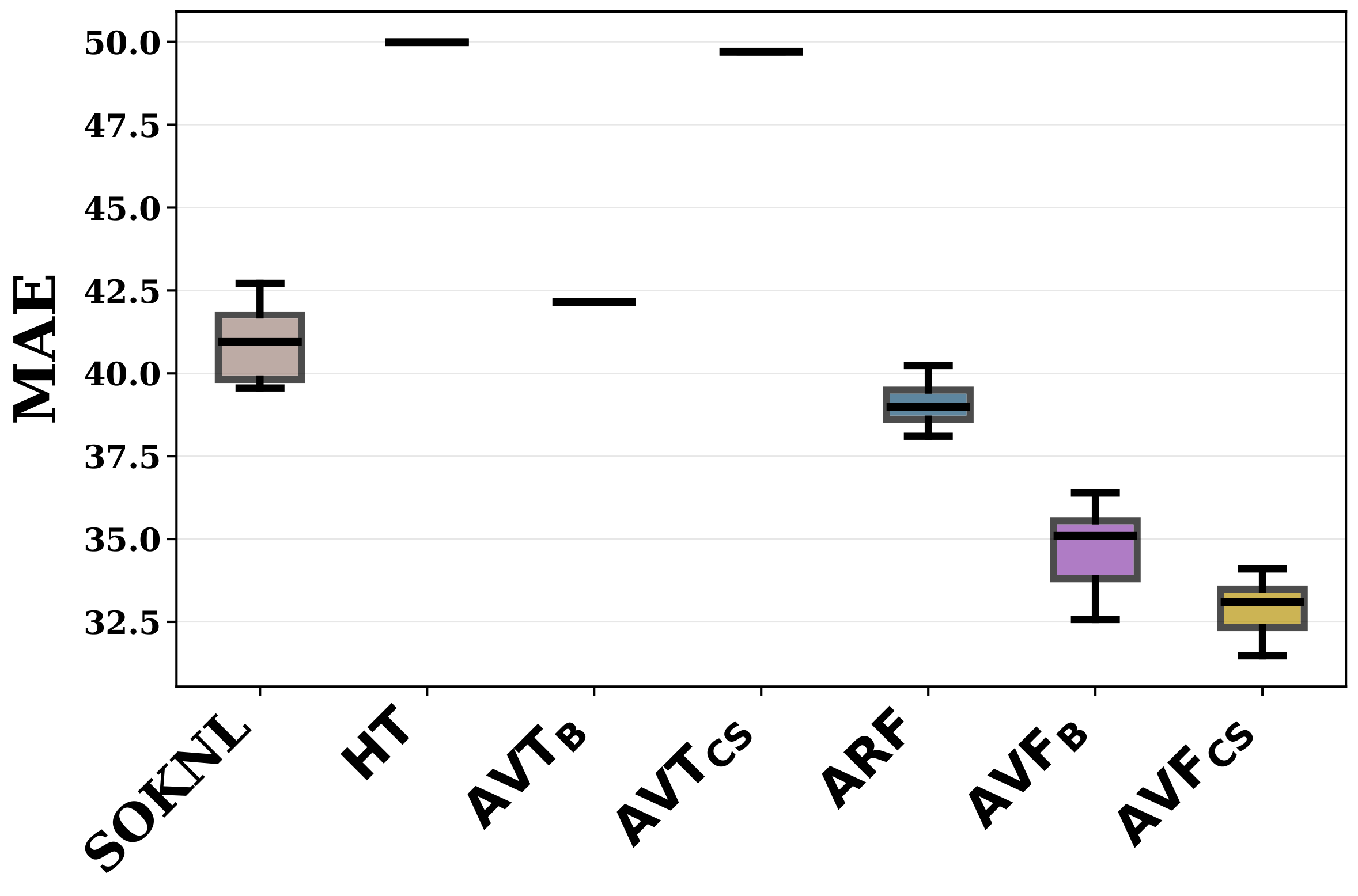}
    \caption{chick}
    \label{fig:reg_chick}
  \end{subfigure} &
  \begin{subfigure}[b]{0.31\textwidth}
    \centering
    \includegraphics[width=\textwidth]{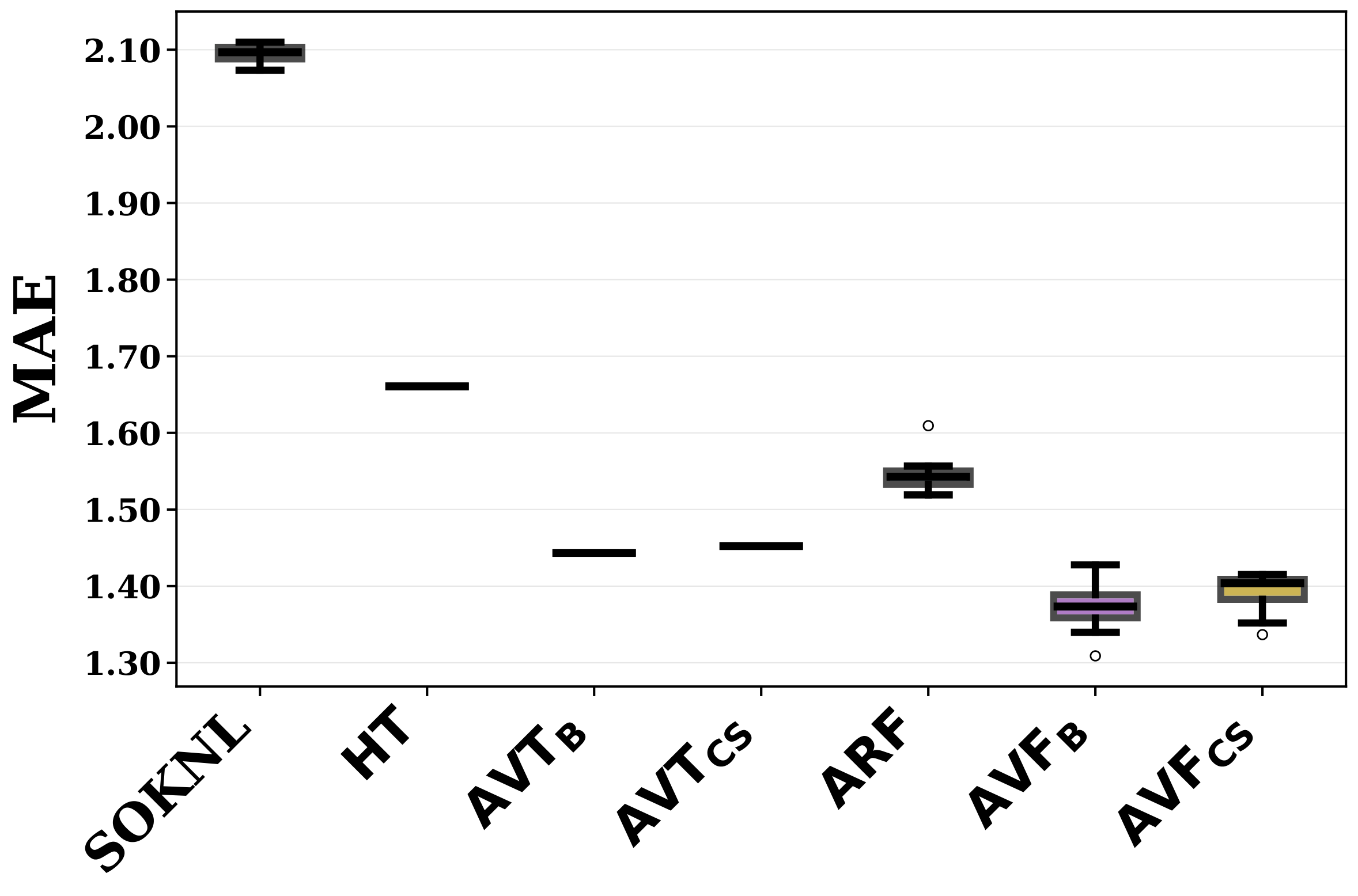}
    \caption{fried}
    \label{fig:reg_fried}
  \end{subfigure} \\[4pt]

  \begin{subfigure}[b]{0.31\textwidth}
    \centering
    \includegraphics[width=\textwidth]{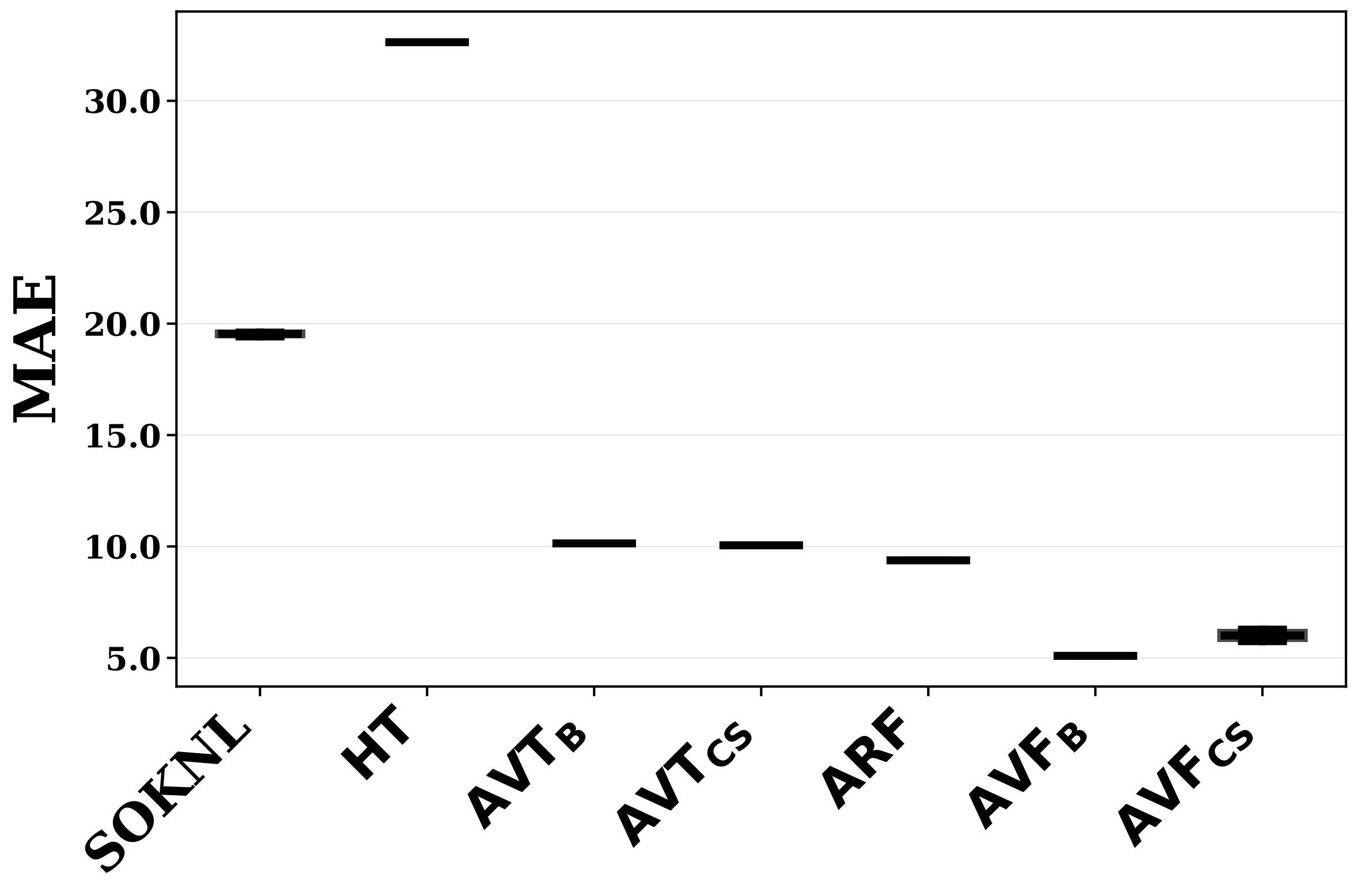}
    \caption{nzenergy}
    \label{fig:reg_nzenergy}
  \end{subfigure} &
  \begin{subfigure}[b]{0.31\textwidth}
    \centering
    \includegraphics[width=\textwidth]{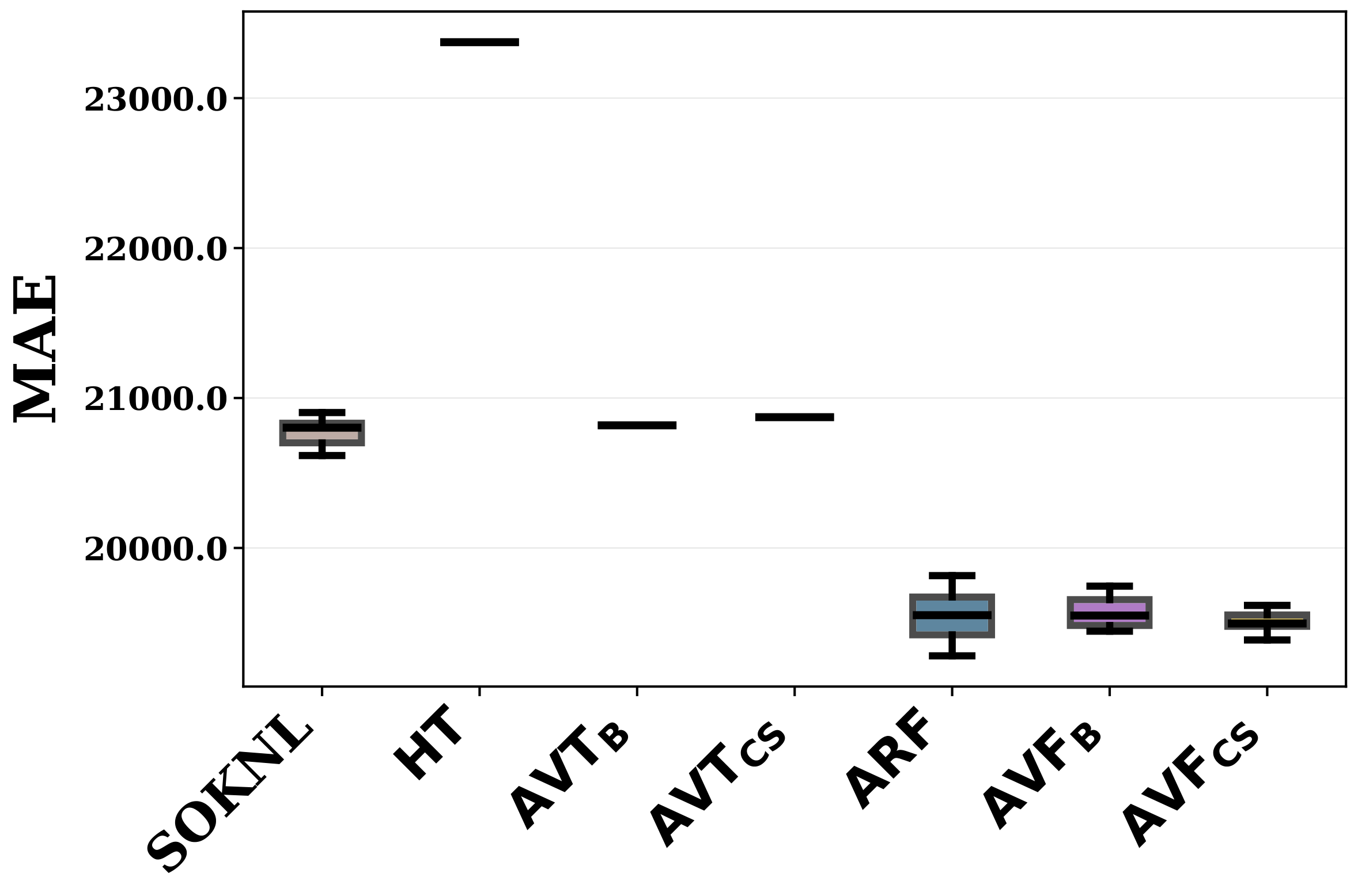}
    \caption{house}
    \label{fig:reg_house}
  \end{subfigure} &
  \begin{subfigure}[b]{0.31\textwidth}
    \centering
    \includegraphics[width=\textwidth]{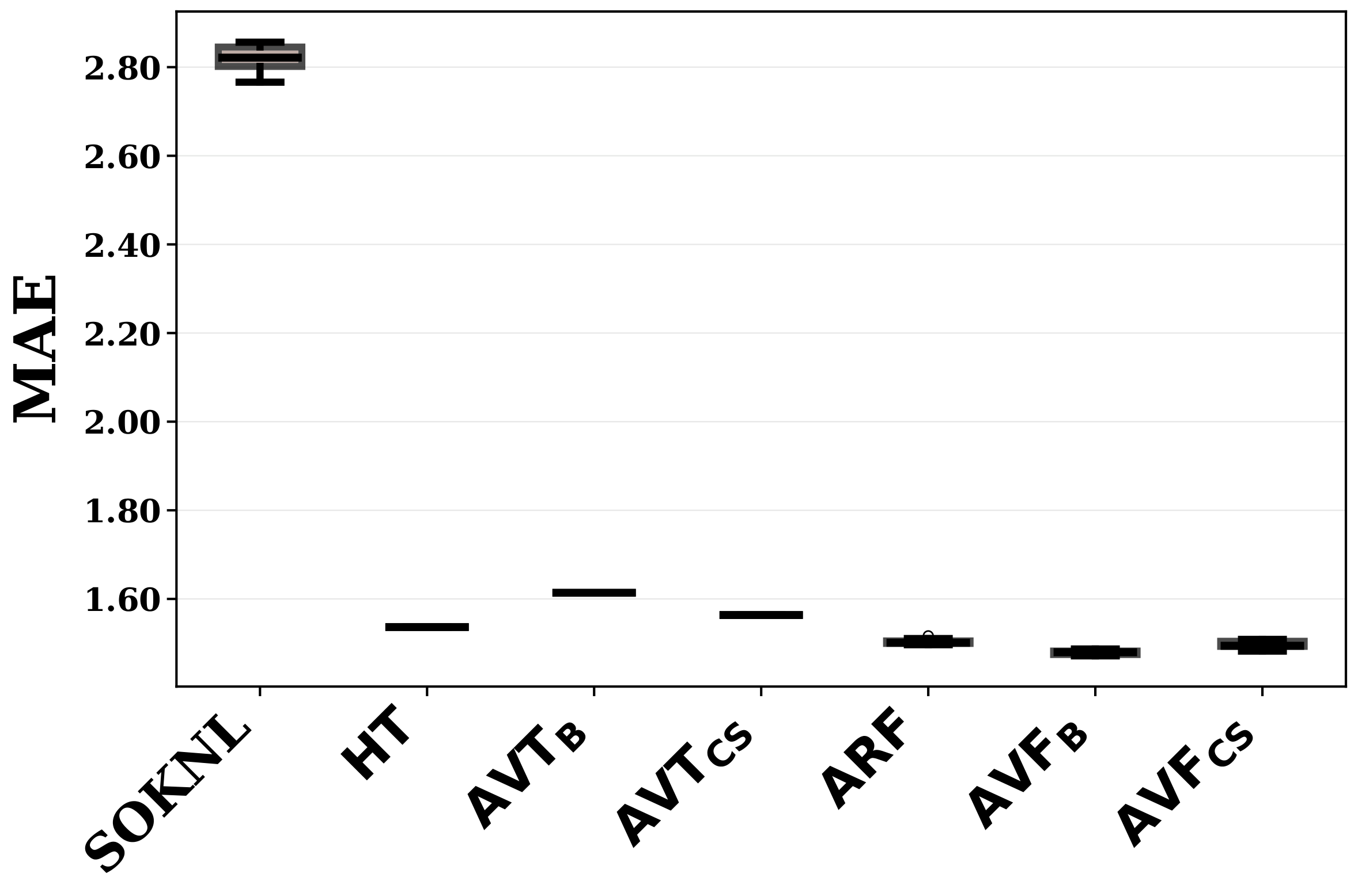}
    \caption{abalone}
    \label{fig:reg_abalone}
  \end{subfigure} \\[6pt]

  \begin{subfigure}[b]{0.31\textwidth}
    \centering
    \includegraphics[width=\textwidth]{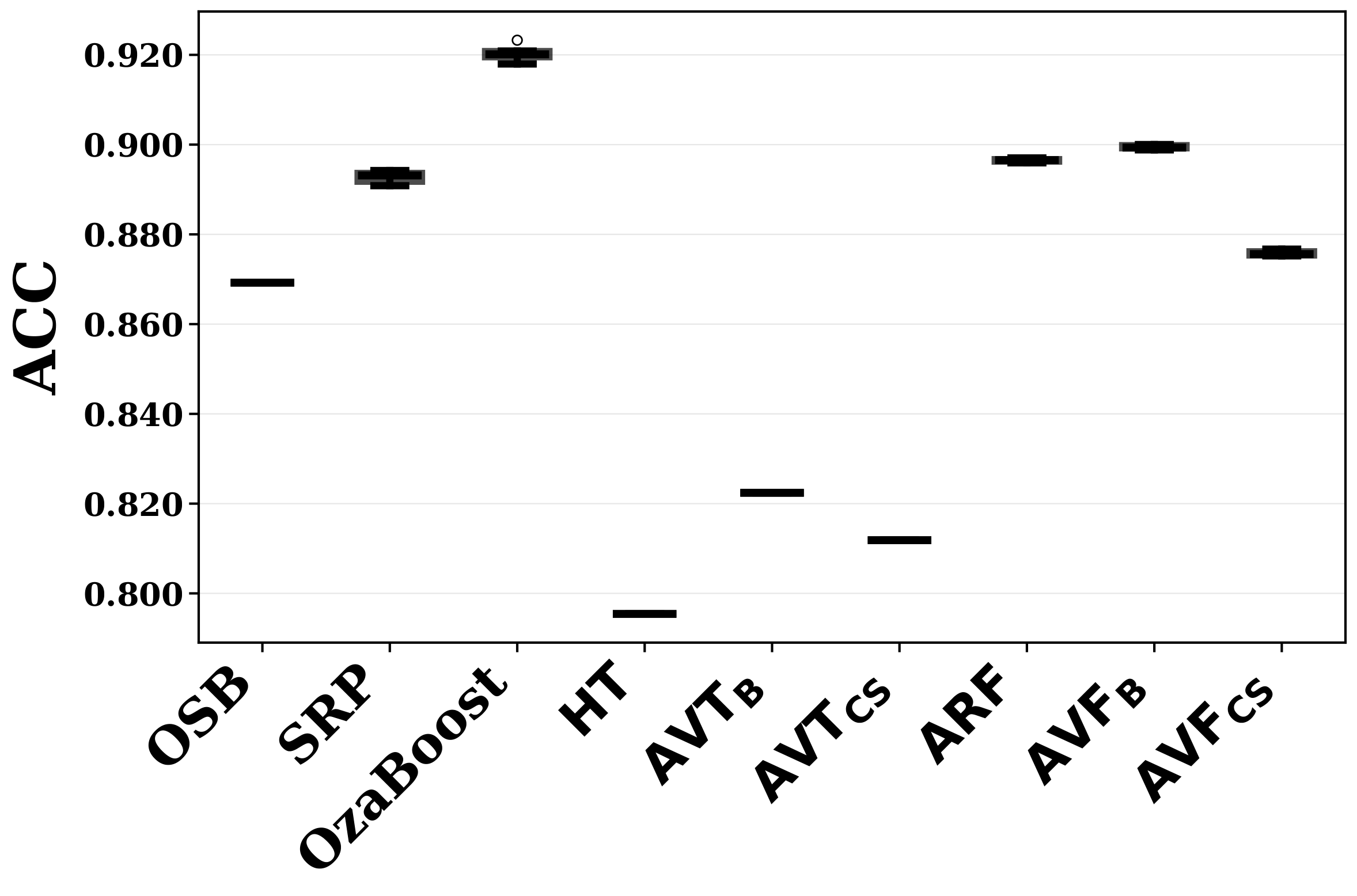}
    \caption{elec2}
    \label{fig:cls_elec2}
  \end{subfigure} &
  \begin{subfigure}[b]{0.31\textwidth}
    \centering
    \includegraphics[width=\textwidth]{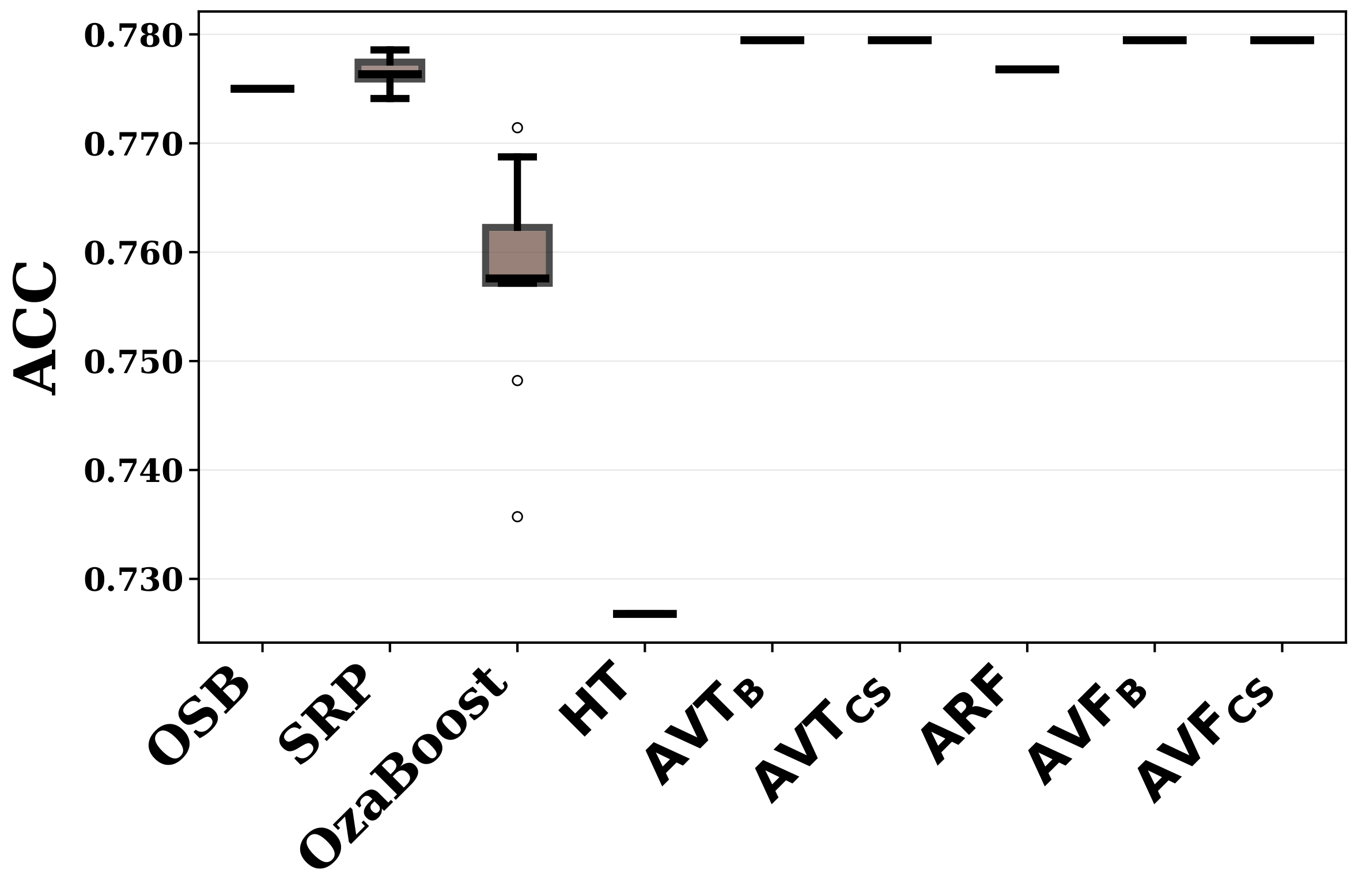}
    \caption{airlines}
    \label{fig:cls_airlines}
  \end{subfigure} &
  \begin{subfigure}[b]{0.31\textwidth}
    \centering
    \includegraphics[width=\textwidth]{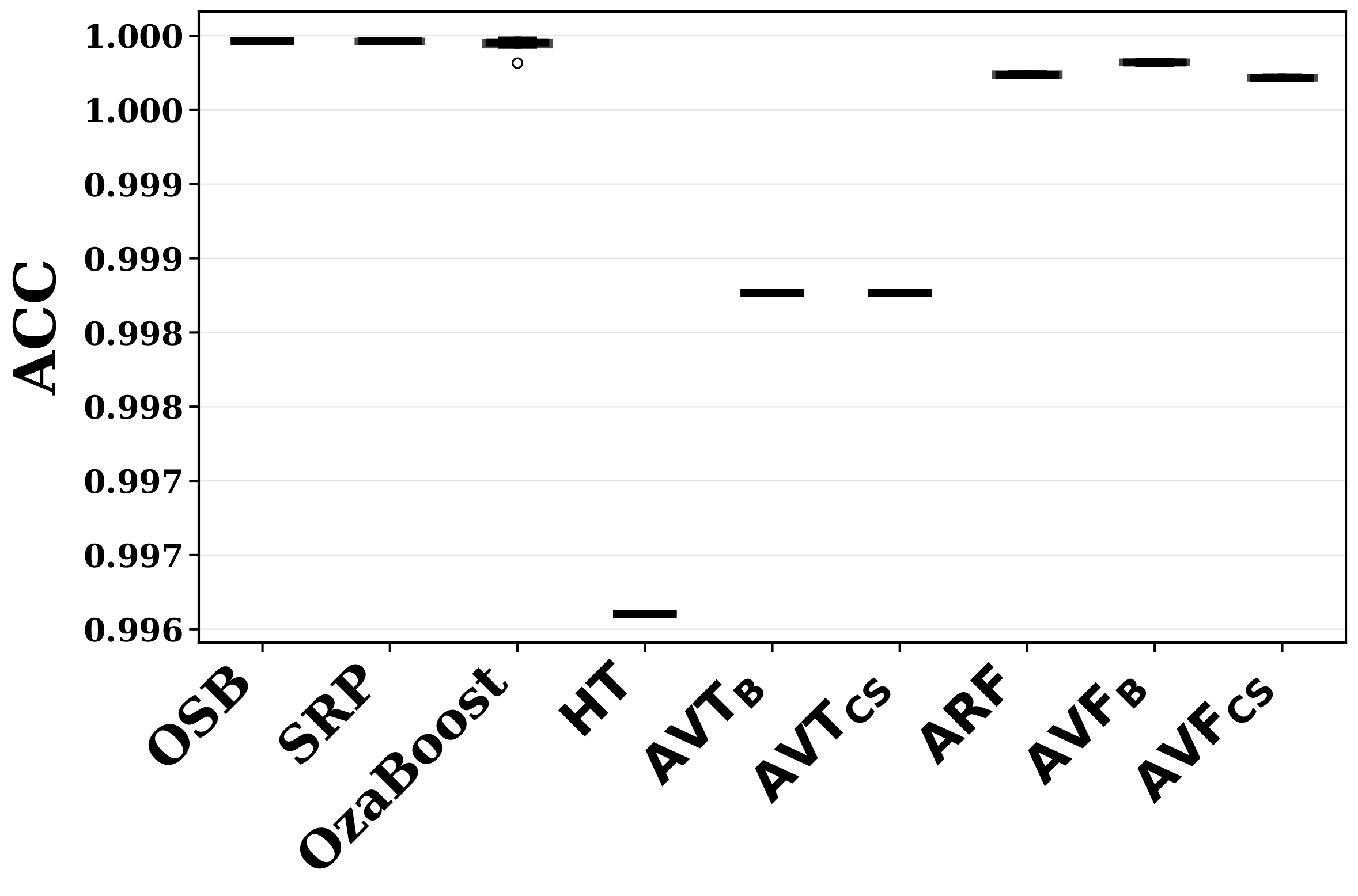}
    \caption{http-KDD99}
    \label{fig:cls_http}
  \end{subfigure} \\[4pt]

  \begin{subfigure}[b]{0.31\textwidth}
    \centering
    \includegraphics[width=\textwidth]{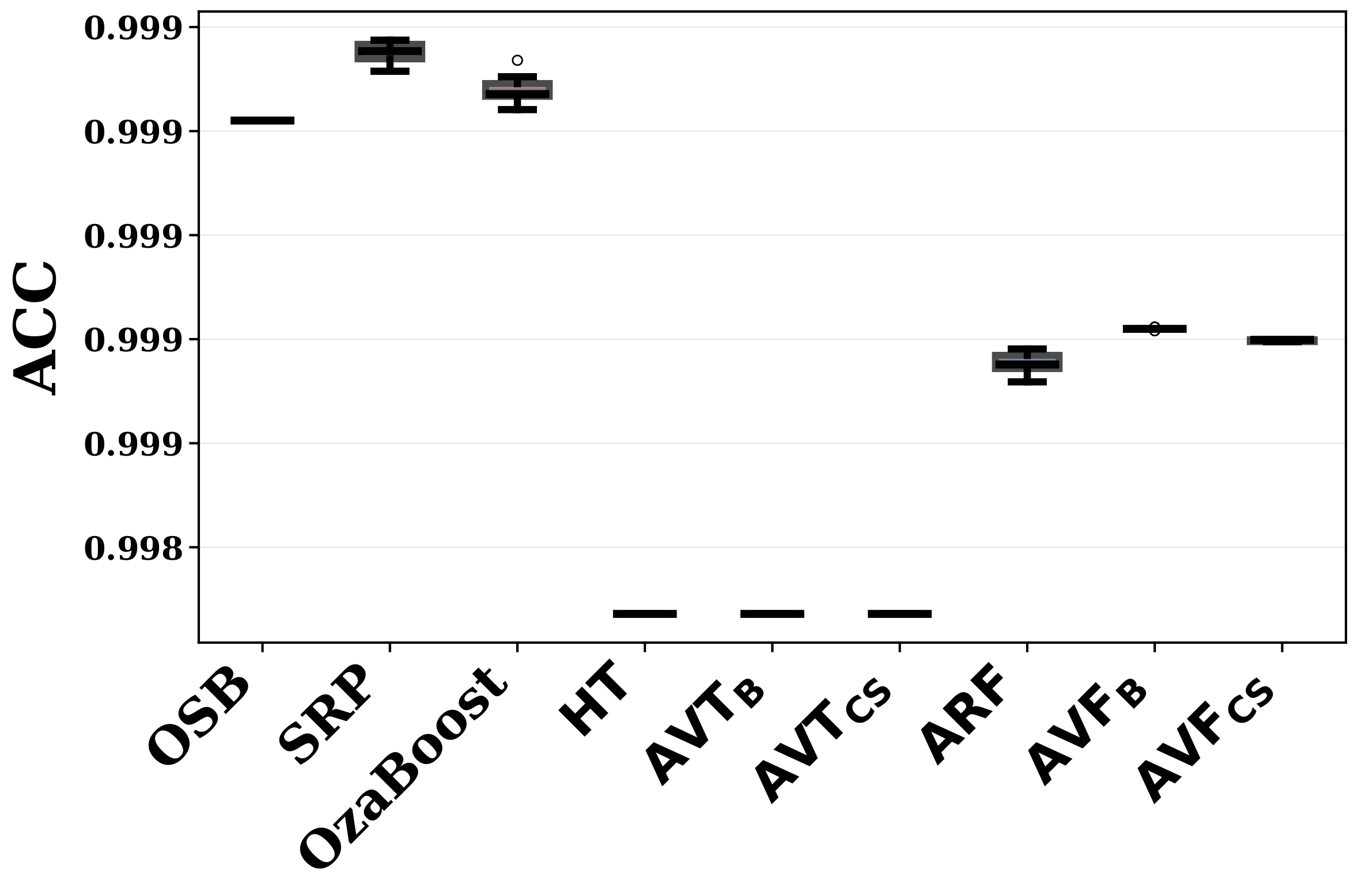}
    \caption{creditcard}
    \label{fig:cls_creditcard}
  \end{subfigure} &
  \begin{subfigure}[b]{0.31\textwidth}
    \centering
    \includegraphics[width=\textwidth]{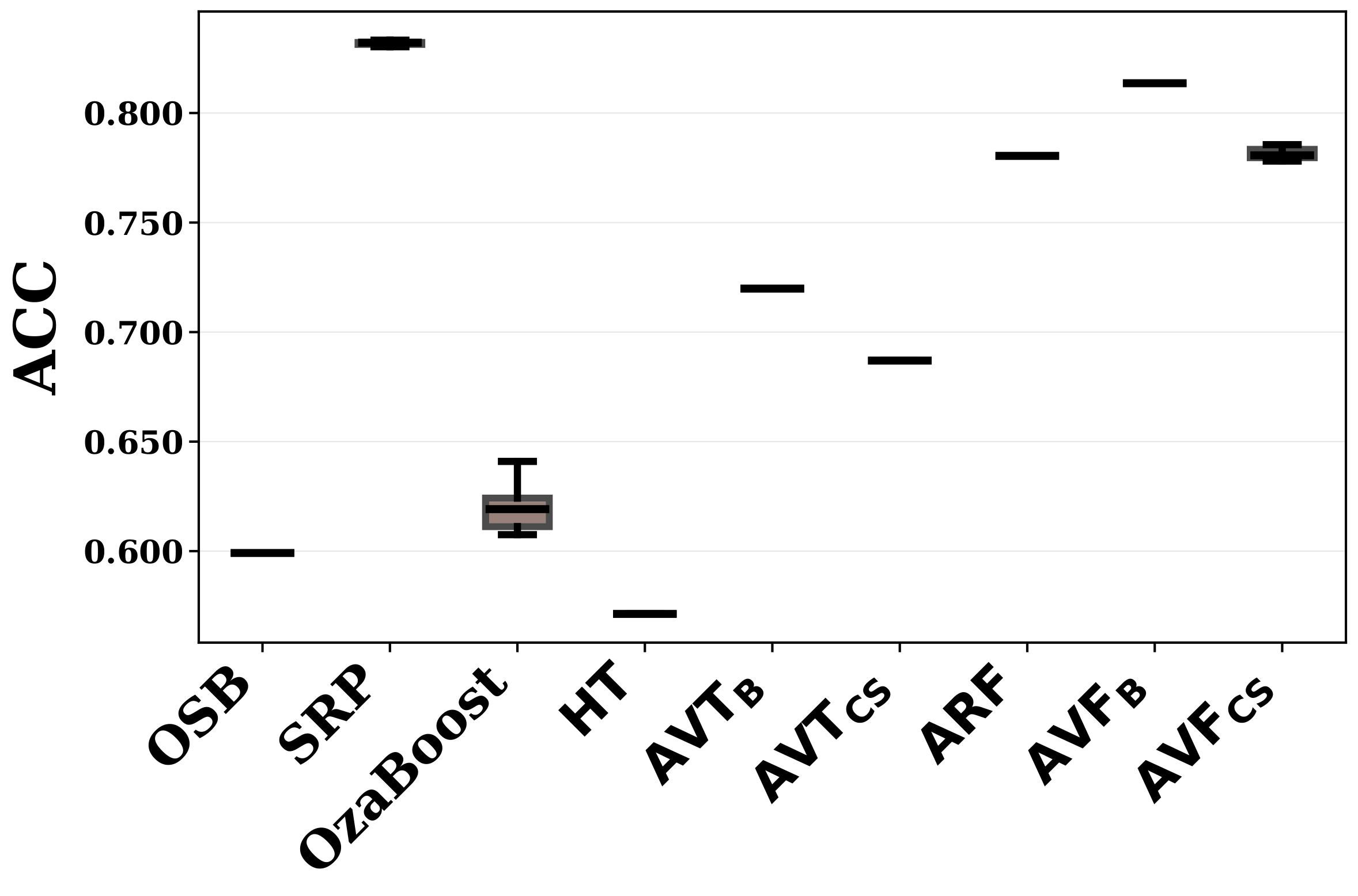}
    \caption{rbfm100k}
    \label{fig:cls_rbfm100k}
  \end{subfigure} &
  \begin{subfigure}[b]{0.31\textwidth}
    \centering
    \includegraphics[width=\textwidth]{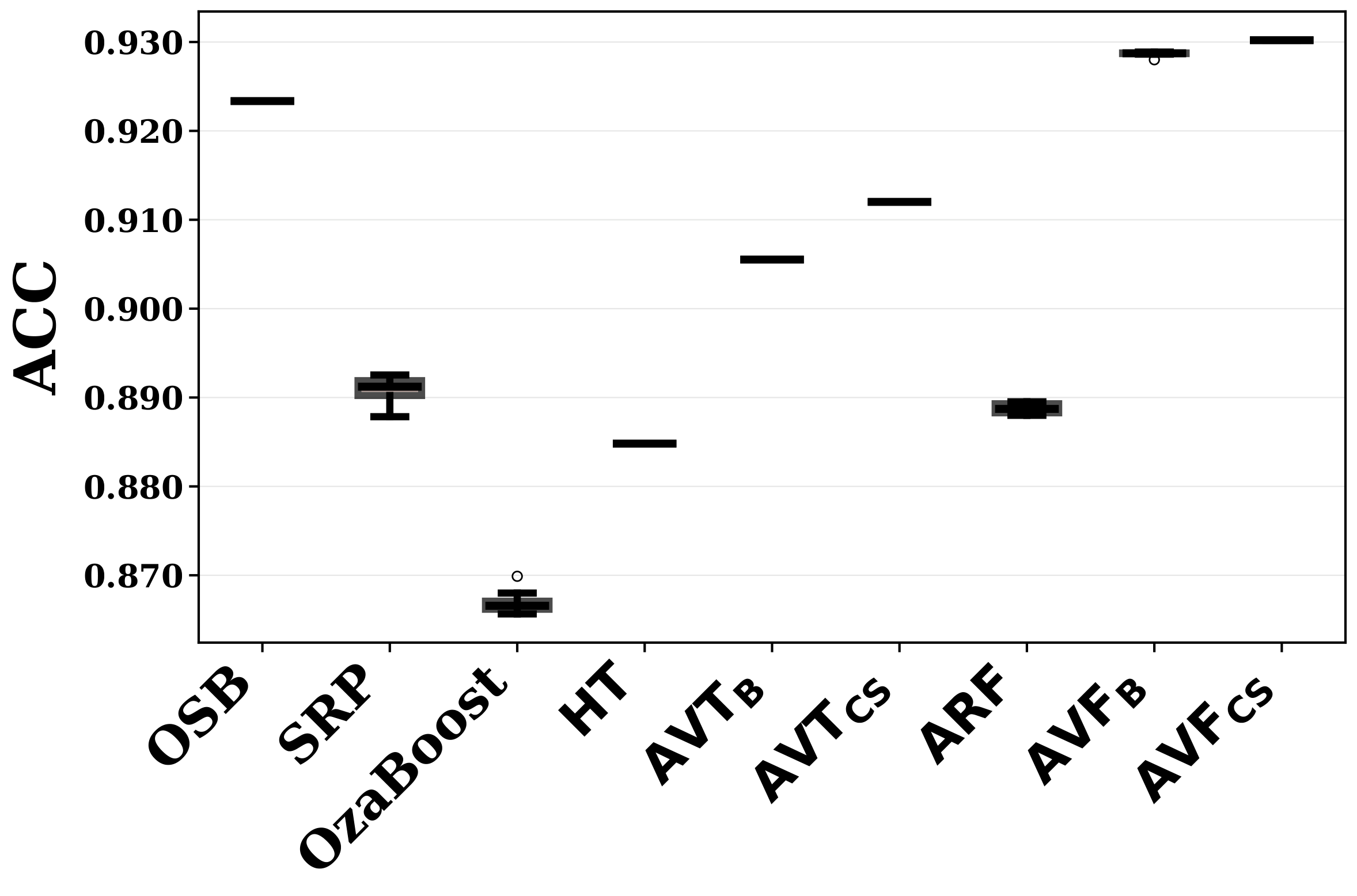}
    \caption{hyper100k}
    \label{fig:cls_hyper100k}
  \end{subfigure}

  \end{tabular}

  \caption{Regression datasets \textbf{(a)–(f)} and classification datasets \textbf{(g)–(l)}.}
  \label{fig:all_datasets_merged}
\end{figure*}

\section{Implementation issues with baselines}
\label{app:sgbt_issue}

During our evaluation, we encountered several implementation issues with existing
baselines, most notably with the Streaming Gradient Boosting Tree (SGBT) \cite{gunasekara2025gradient}.

\paragraph{Degenerate regression behavior.}
When used for regression, SGBT consistently outputs a zero prediction throughout
the entire stream, independently of the dataset or training progress. This
behavior is illustrated by the diagnostic experiment shown in
Fig.~\ref{fig:sgbt_issue}.

\begin{figure}[!ht]
    \centering
    \includegraphics[width=0.5\linewidth]{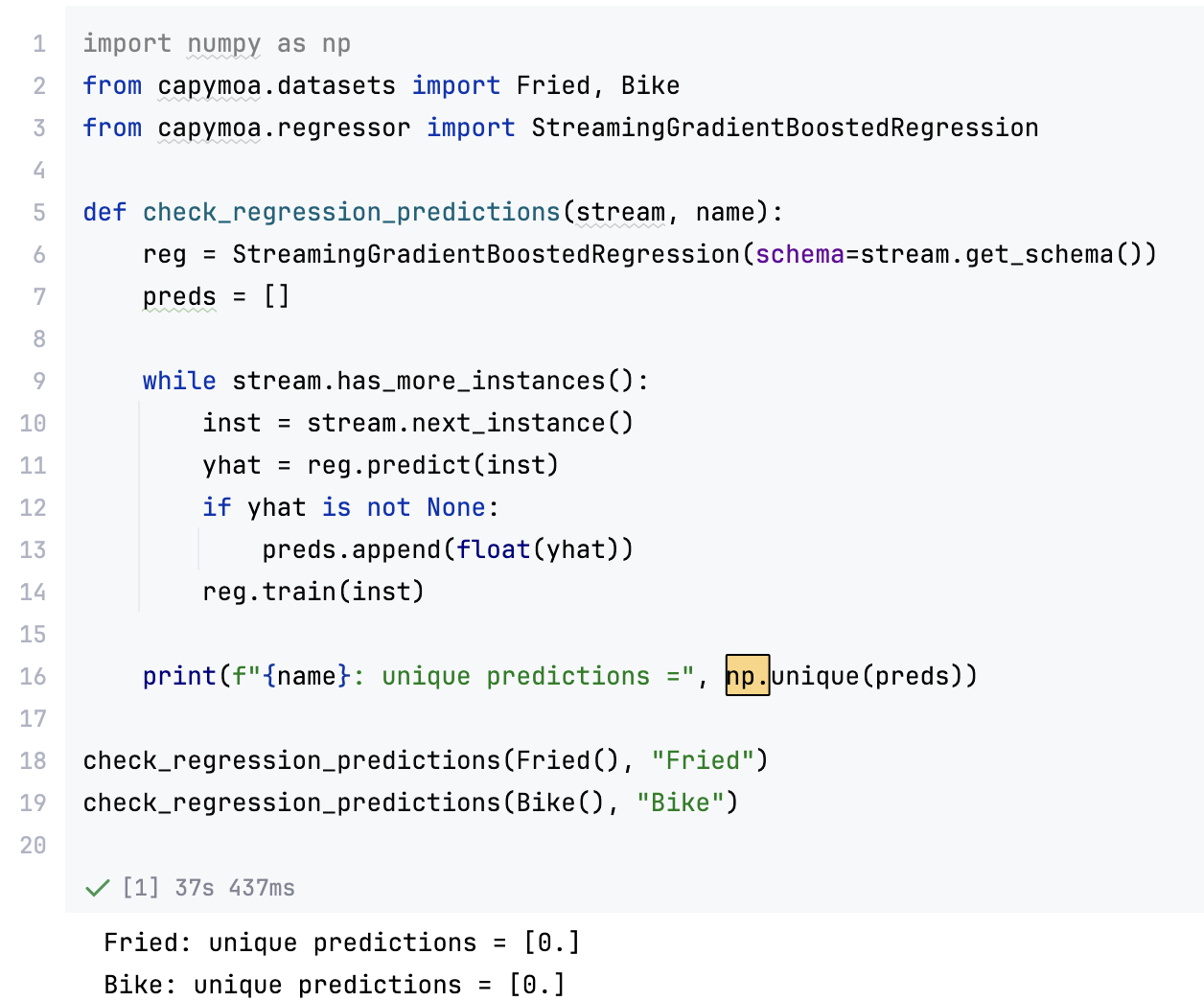}
    \caption{SGBT regression behavior: the model outputs a constant zero prediction
    over the full stream.}
    \label{fig:sgbt_issue}
\end{figure}

\paragraph{Source inspection.}
Inspection of the source code reveals that the SGBT regressor inherits from a
classifier-oriented base class (Fig.~\ref{fig:sgbt_class}), which may implicitly
enforce classification-specific behavior. We attempted several fixes but were
unable to obtain a functioning regression variant.
\begin{figure}[!ht]
    \centering
    \includegraphics[width=0.7\linewidth]{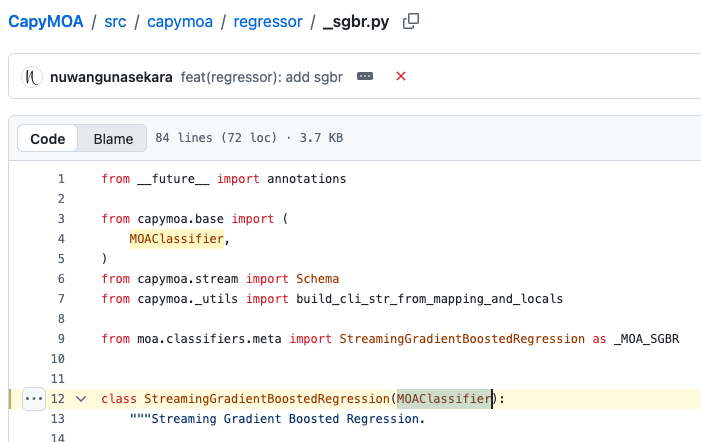}
    \caption{SGBT class hierarchy highlighting inheritance from a classifier base.}
    \label{fig:sgbt_class}
\end{figure}

\paragraph{Stability issues.}
More broadly, several CapyMOA baselines exhibit recurrent runtime failures. For
example, experiments on several datasets terminate with Java errors, which we were unable to resolve. All experiments were run on AWS EC2 instances. The main evaluation used a c7i.16xlarge instance (64 vCPUs, 128 GiB RAM), and we also attempted a memory-optimized r8a.16xlarge configuration (64 vCPUs, 512 GiB RAM), but the stability issues persisted.

\begin{figure}[!ht]
    \centering
    \includegraphics[width=\linewidth]{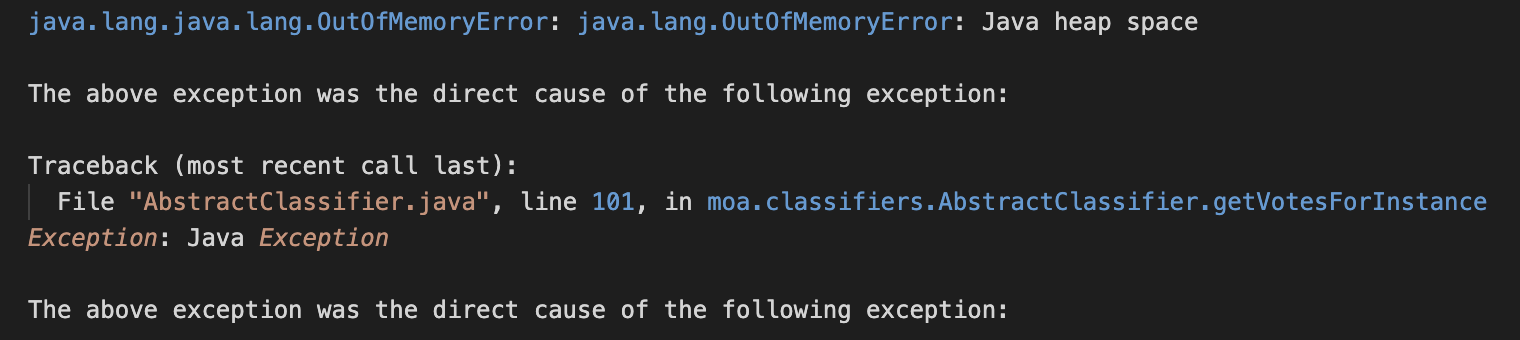}
    \caption{Example Java runtime error encountered when running CapyMOA baselines.}
    \label{fig:java}
\end{figure}

\section{Experiments details}
All experiments were run on AWS EC2 instances. The main evaluation used a c7i.16xlarge instance (64 vCPUs, 128 GiB RAM).

\end{document}